\documentclass{article}

\PassOptionsToPackage{table}{xcolor}
\usepackage[table]{xcolor}
\usepackage[final]{corl_2026} 

\title{
\textsc{StressDream}: Steering Video World Models for Robust Policy Evaluation and Improvement
}

\author{
Junwon Seo$^{1}$, 
Sushant Veer$^{2}$, 
Ran Tian$^{2}$, 
Wenhao Ding$^{2}$, 
Apoorva Sharma$^{2}$,\\[0.2em]
\textbf{Karen Leung$^{2,3}$,
Edward Schmerling$^{2}$,
Marco Pavone$^{2,4}$,
Andrea Bajcsy$^{1}$}\\[0.7em]
$^{1}$Carnegie Mellon University \qquad
$^{2}$NVIDIA Research\\[0.2em]
$^{3}$University of Washington \qquad
$^{4}$Stanford University \\[0.5em]
\href{https://stressdream.github.io/}{\textcolor{vibrant_orange}{\texttt{https://stressdream.github.io/}}}
}
  
\usepackage{moreverb,url}
\usepackage{xspace}
\usepackage{siunitx}
\usepackage{wrapfig}

\usepackage{amssymb}
\usepackage{amsmath}
\usepackage{amsthm}
\usepackage{amsbsy}
\usepackage{bbm}
\usepackage{dsfont}

\usepackage{amsthm}

\usepackage{graphicx}
\usepackage{wrapfig}
\usepackage{booktabs} 

\usepackage{subfiles}
\usepackage{csquotes}
\usepackage{bm}

\usepackage{tocloft}
\usepackage{array}

\usepackage[ruled,vlined,linesnumbered]{algorithm2e}

\usepackage{graphicx}
\usepackage{subfig} 
\usepackage{multirow}

\newcommand\BibTeX{{\rmfamily B\kern-.05em \textsc{i\kern-.025em b}\kern-.08em
T\kern-.1667em\lower.7ex\hbox{E}\kern-.125emX}}

\definecolor{wine}{RGB}{204, 0, 102}
\definecolor{magenta_wine}{RGB}{158, 44, 143}
\definecolor{dusty_wine}{RGB}{143, 59, 101}
\definecolor{ocean}{RGB}{13, 121, 202}
\definecolor{light_ocean}{RGB}{18, 178, 235}
\definecolor{dark_ocean}{RGB}{10, 89, 148}
\definecolor{grey}{RGB}{170, 170, 170}
\definecolor{light-grey}{RGB}{220, 220, 220}
\definecolor{dark-gray}{rgb}{0.2, 0.2, 0.2} 
\definecolor{med-grey}{rgb}{0.3, 0.3, 0.3} 
\definecolor{grape}{RGB}{112,48,160}
\definecolor{aqua}{RGB}{52,172,139}
\definecolor{dark_aqua}{RGB}{35,115,93}
\definecolor{dark_orange}{RGB}{216,92,0}
\definecolor{vibrant_orange}{RGB}{250, 160, 26}
\definecolor{vibrant_blue}{RGB}{14, 120, 255}
\definecolor{vibrant_pink}{RGB}{255, 0, 104}
\definecolor{dark_red}{RGB}{122, 0, 0}
\definecolor{dark_green}{RGB}{0, 92, 34}
\definecolor{dusty_blue}{RGB}{77, 91, 128}
\definecolor{dark_brown}{RGB}{125, 54, 36}
\definecolor{violet}{RGB}{116, 12, 173}

\definecolor{green-reg}{RGB}{13, 148, 143}

\hypersetup{
    colorlinks=true,
    linkcolor=vibrant_orange, 
    citecolor=vibrant_orange, 
    urlcolor=vibrant_orange   
}

\newtheorem{lemma}{Lemma}
\newtheorem{corollary}{Corollary}

\newcommand{\para}[1]{\smallskip\noindent\textbf{#1. }} 
 
\newcounter{qnum}
\usepackage{pifont}
 
\DeclareMathOperator*{\argmax}{argmax}
\DeclareMathOperator*{\argmin}{argmin}

\newcommand{\state}{s}
\newcommand{\ours}{\textsc{StressDream}\xspace}
\newcommand{\nominal}{\textit{Nominal}\xspace}
\newcommand{\bestofN}{\textit{Best-of-N}\xspace}
\newcommand{\classGuidance}{\textit{CG}\xspace}
\newcommand{\vista}{\textit{Vista}\xspace}
\newcommand{\ctrlworld}{\textit{Ctrl-World}\xspace}

\usepackage{annotate-equations}

\begin{document}

\let\oldaddcontentsline\addcontentsline
\renewcommand{\addcontentsline}[3]{}

\makeatletter
\let\@oldmaketitle\@maketitle
\renewcommand{\@maketitle}{%
\@oldmaketitle
\setcounter{figure}{0}
\centering
\phantomsection
\vspace{-0.2in}
\includegraphics[width=\textwidth]{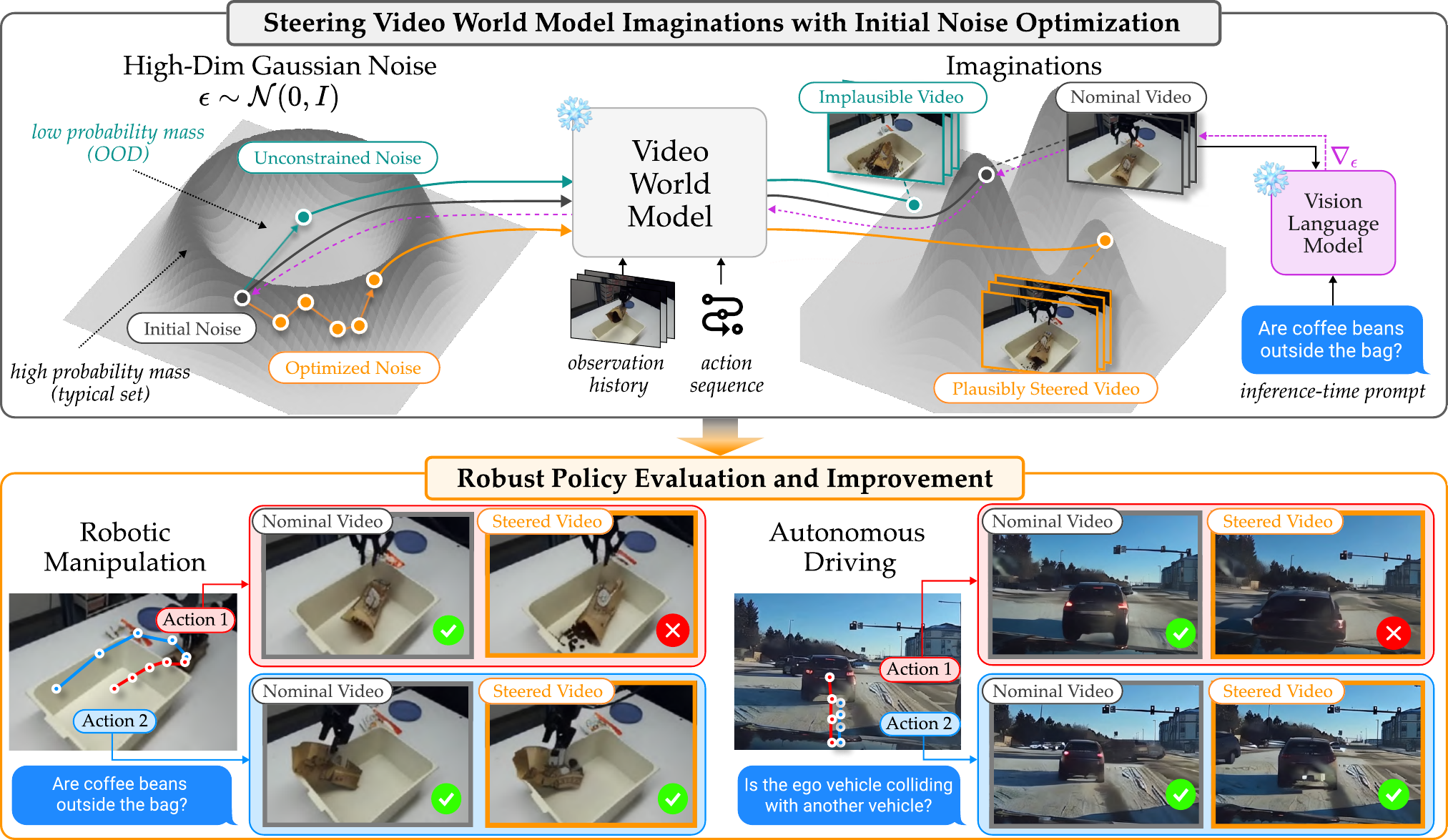}\vspace{-0.05in}
{\captionsetup{hypcap=false}%
\captionof{figure}{
\small{
   \textbf{Overview of \ours.} (\textit{Top}): The initial noise of diffusion-based WMs is optimized to steer imaginations toward a target event specified by an inference-time prompt. While nominal imaginations may miss high-impact outcomes of robot actions, such as spilling, \ours uses VLM guidance to steer imaginations toward such outcomes while preventing the high-dimensional noise from drifting into OOD, where generations become implausible. (\textit{Bottom}): \ours imagines high-impact but plausible outcomes of robot actions, such as task failures or collisions, enabling robust policy evaluation and improvement.
}
}
\label{fig:front-fig}}%
\bigskip}
\makeatother
\maketitle


\begin{abstract}
Video world models~(WMs) have shown promise for policy evaluation and improvement by imagining realistic future observations conditioned on ego-robot actions.
While WMs can model distributions over futures, policy evaluation and improvement typically rely on nominal imaginations, which can miss high-impact outcomes of robot actions unless prohibitively many samples are drawn.
To enable robust policy evaluation and improvement over WM imaginations, we propose \ours, which steers imaginations toward high-impact yet plausible outcomes specified at inference time by optimizing the initial noise of diffusion-based WMs.
However, optimizing high-dimensional noise is challenging: the optimization must reason about nuanced, scene-dependent target events in generated videos while avoiding out-of-distribution (OOD) noise that yields implausible imaginations.
We address this with two complementary objectives: a semantic objective with a Vision-Language Model that provides informative gradients by reasoning about the generated video, and a plausibility objective that prevents the optimized noise from drifting OOD.
With state-of-the-art video world models for autonomous driving and robotic manipulation, we show that \ours effectively steers imaginations toward high-impact yet plausible outcomes specified by text at inference time, such as task failures, enabling robust policy evaluation and improvement by identifying actions whose plausible futures include undesirable outcomes.
Video results can be found at \href{https://stressdream.github.io/}{\textcolor{vibrant_orange}{Project Website}}.

\end{abstract}


\vspace{-0.05in}
\section{Introduction}

Video world models~(WMs), or learned simulators of physical environments, have attracted significant interest in robotics, promising policy evaluation and improvement with substantially less costly real-world interaction~\cite{agarwal2025cosmos, wan2025wan, mei2026video}. At their core, video WMs are generative models, such as diffusion and flow matching~\cite{ho2020denoising, karras2022elucidating, lipman2023flow}, learning \textit{distributions} over future observations conditioned on ego robot actions, and thereby capturing the inherent uncertainty of physical interactions or behaviors of other agents~\cite{alonso2024diffusion, chen2024diffusion, gao2025adaworld, kerssies2026frame} (\textit{top}, Fig.~\ref{fig:front-fig}). For example, in robotic manipulation, dropping an open bag from high above the table may or may not cause its contents to spill, leading to either task success or failure; in autonomous driving, the same ego-vehicle decision can elicit qualitatively different outcomes depending on the responses of surrounding agents, such as braking or maintaining speed.

Despite learning rich distributions over future outcomes, policy evaluation and improvement with WMs typically rely on \textit{nominal} imaginations~\cite{chen2024diffusion, gao2024vista, guo2025ctrl, guo2026vlaw}, which can under-explore the diverse outcomes they are trained to model.
Naively sampling imaginations from WMs can easily miss critical outcomes without drawing prohibitively many samples, making it ill-suited to the role of WMs: exploring \textit{plausible}, \textit{high-impact} outcomes of robot actions for policy evaluation and improvement~\cite{zhang2025epona,quevedo2025worldgym, team2025evaluating, sharma2026world}, where plausibility refers to outcomes supported by the learned WM distribution.
For example, when a manipulator drops an open bag from high above the table, the WM's distribution may include outcomes both with and without spilling; in contrast, when the bag is placed close to the table, spilling may be unlikely or absent from the outcome distribution (\textit{bottom}, Fig.~\ref{fig:front-fig}). The ability to imagine plausible, high-impact outcomes can therefore allow the latter action to be preferred during policy evaluation, while the former is downweighted or suppressed during policy improvement. This motivates the central question of our work:\vspace{-0.05in}\begin{quote}
    \centering
    \textit{How can we efficiently steer video world models at inference time to imagine plausible, high-impact outcomes for policy evaluation and improvement?}
\end{quote}\vspace{-0.05in}

Our key idea is to optimize the \textit{initial noise} of a diffusion-based WM, steering generation toward high-impact events specified at inference time (e.g., spills or collisions), within the distribution of plausible outcomes.
The challenge, however, is to guide this optimization in \textit{extremely high-dimensional} noise spaces (e.g., \(1\)M dims~\cite{gao2024vista}): the optimization must reason about nuanced, scene-dependent target events from generated videos, while naive optimization can easily push the initial noise into out-of-distribution (OOD) regions~\cite{samuel2023norm,tang2024inference}, producing implausible videos. 
To address this, we propose \ours, which combines a \textit{semantic} objective, using a Vision-Language Model (VLM) that reasons about \textit{what happens in the generated video} to guide generation toward target events, with a \textit{plausibility} objective that keeps the optimized noise within the in-distribution region.
Together, \ours provides dense gradient signals for searching the high-dimensional noise space, while preserving plausibility of generations, enabling video WMs to explore plausible, high-impact outcome distributions for \textit{robust} policy evaluation and improvement.

\para{Statement of Contributions}
We propose \ours, a novel inference-time steering method for diffusion-based video world models. Our algorithm optimizes the high-dimensional Gaussian initial noise to steer generations toward high-impact outcomes while preserving plausibility. Through experiments, we show: (\romannumeral 1) in controlled settings with access to the true system dynamics, \ours steers WM imaginations toward target task-failure events only when such outcomes are plausible;
(\romannumeral 2) on state-of-the-art video world models for autonomous driving~\cite{gao2024vista} and robotic manipulation~\cite{guo2025ctrl}, \ours enables robust policy evaluation by detecting high-impact outcomes, such as task failures, with substantially higher recall (\(54\%\rightarrow94\%\));
and (\romannumeral 3) this robust evaluation improves policy by promoting robust actions that avoid potential failures, increasing the success rate of a Vision-Language-Action (VLA) policy~\cite{intelligence2025pi} (\(39\%\rightarrow71\%\)).

\newcommand{\wm}{f}
\newcommand{\wmParam}{\theta}
\newcommand{\denoise}{g}
\newcommand{\encoder}{\mathcal{E}}
\newcommand{\obs}{o}
\newcommand{\image}{I}
\newcommand{\joint}{q}
\newcommand{\obsSpace}{\mathcal{O}}
\newcommand{\latent}{z}
\newcommand{\latentSpace}{\mathcal{Z}}
\newcommand{\latentSpaceDimension}{C}
\newcommand{\dimension}{D}
\newcommand{\actionSequence}{\mathbf{a}}
\newcommand{\action}{a}
\newcommand{\horizon}{H}
\newcommand{\history}{h}
\newcommand{\historyLatent}{\textbf{\latent}_h}
\newcommand{\historyObs}{\mathbf{\obs}^\text{hist}}
\newcommand{\dynamics}{f}

\newcommand{\timeIdx}{t}
\newcommand{\diffusionTimeMax}{T}
\newcommand{\diffusionTime}{\tau}
\newcommand{\diffusionCoeff}{\alpha}
\newcommand{\noiseCoeff}{\sigma}
\newcommand{\noise}{\epsilon}
\newcommand{\data}{x}
\newcommand{\dataDistribution}{p^{\text{data}}}

\newcommand{\reward}{\mathcal{C}}
\newcommand{\rewardParam}{\psi}
\newcommand{\criterion}{\reward^{\text{test}}}
\newcommand{\denoising}{d}
\newcommand{\regularizer}{\reward^{\text{pla}}}

\newcommand{\prompt}{\bm{\ell}}

\addvalue{inner}{dark-gray}
\addvalue{outer}{vibrant_blue}

\section{Background: Action-conditioned Video World Models}\vspace{-0.05in}

\para{Video World Models}
An action-conditioned video world model \(\wm_\wmParam\), parameterized by \(\wmParam\), models distributions of future observations \(\mathbf{\obs} := \obs_{\timeIdx:\timeIdx+\horizon}\), where \(\timeIdx\) denotes the time and \(\horizon\) is the prediction horizon~\cite{guo2025ctrl, quevedo2025worldgym, team2025evaluating, li2025worldeval, kim2026cosmos, yin2026playworld}. Each observation \(\obs = [\image^1,\cdots,\image^n,\joint] \in \obsSpace\) consists of \(n\) camera images \(\image^i\) and proprioceptive state \(\joint\), such as the robot joint configuration. The prediction is conditioned on an observation history \(\historyObs := \obs_{\timeIdx-\history:\timeIdx}\) and a future ego-action sequence \(\actionSequence := \action_{\timeIdx:\timeIdx+\horizon}\):
\vspace{-0.1in}\begin{equation}
    \mathbf{\obs} \sim \wm_\wmParam(\cdot \mid \historyObs, \actionSequence),
    \qquad \text{where} \quad \mathbf{\obs}:= \obs_{\timeIdx:\timeIdx+\horizon}, \quad
    \historyObs := \obs_{\timeIdx-\history:\timeIdx}, \quad
    \actionSequence := \action_{\timeIdx:\timeIdx+\horizon}.
\end{equation}

\vspace{-0.1in}
\para{Diffusion Models as Video World Models} Recent video world models~\cite{gao2024vista, guo2025ctrl,yin2026playworld} often instantiate $\wm_\wmParam$ as a diffusion model~\cite{ho2020denoising, song2020denoising}, learning a transformation between the data distribution, $\mathbf{\data}^0 := \mathbf{\obs} \sim \dataDistribution(\mathbf{\obs})$, and a tractable noise distribution, typically a standard Gaussian $\mathbf{\data}^\diffusionTimeMax := \bm{\noise} \sim \mathcal{N}(\mathbf{0},\mathbf{I}_\dimension)$, where $\dimension$ is a noise dimension, indexed by diffusion time $\diffusionTime \in \{0,1,\dots,\diffusionTimeMax\}$:
\vspace{-0.05in}\begin{equation}
    \mathbf{\data}^\diffusionTime
    =
    \diffusionCoeff^\diffusionTime\,\mathbf{\data}
    +
    \noiseCoeff^\diffusionTime\,\bm{\noise}, \qquad \text{where}
    \quad
    \mathbf{\data}^0 := \mathbf{\obs}, \quad \mathbf{\data}^T := \bm{\noise} \sim \mathcal{N}(\mathbf{0},\mathbf{I}_\dimension),
\end{equation} \vspace{-0.3in}

\noindent where $\diffusionCoeff^\diffusionTime$ is a decreasing function of $\diffusionTime$ and $\noiseCoeff^\diffusionTime$ is an increasing function. Generation proceeds in the reverse direction by iteratively mapping a noisier sample to one at a less noisy diffusion time
:
\vspace{-0.05in}\begin{equation}\label{eq:denoising}
    \denoising^\diffusionTime(\mathbf{\data}^{\diffusionTime}, \historyObs, \actionSequence) := \,\mathbf{\data}^{\diffusionTime - 1}
    =  
    \frac{\diffusionCoeff^{\diffusionTime - 1}}{\diffusionCoeff^{\diffusionTime}} \left(
        \mathbf{\data}^{\diffusionTime}
        -
        \noiseCoeff^{\diffusionTime}\,
        \denoise^{\diffusionTime}_{\wmParam}\!\left(
            \mathbf{\data}^{\diffusionTime}, \historyObs,
            \actionSequence
        \right)
    \right)
    +
    \noiseCoeff^{\diffusionTime - 1}\,
    \denoise^{\diffusionTime}_{\wmParam}\!\left(
        \mathbf{\data}^{\diffusionTime}, \historyObs,
        \actionSequence
    \right),
\end{equation} \vspace{-0.25in}

\noindent where \(\denoise^{\diffusionTime}_{\wmParam}\) is the noise predictor (see Appendix~\ref{app:diffusion-background}). Starting from an initial Gaussian sample $\mathbf{\data}^\diffusionTimeMax=\bm{\noise}$, repeated application of this reverse update produces the predicted future observations $\mathbf{\data}^0 = \mathbf{\obs}$.

\vspace{-0.05in}
\para{Initial Noise as a Control Variable}\label{sec:diffusion-noise} Following \eqref{eq:denoising}, which corresponds to the probability-flow ODE~\cite{song2021scorebased}, for fixed conditioning inputs \((\historyObs,\actionSequence)\), the imaginations are a 
\textit{deterministic} function of the initial noise: \(\mathbf{\obs}=\wm_\wmParam(\bm{\noise}, \historyObs, \actionSequence).\) Thus, all stochasticity in generation is governed by the initial noise, implying that selecting the initial noise controls which video is generated from the WMs.

\vspace{-0.05in}
\section{Setup: Policy Evaluation and Improvement with Video World Models}\vspace{-0.05in}

\para{Policy Evaluation and Improvement with Video World Models} 
Given an action-conditioned video world model \(\wm_\wmParam\), we aim to evaluate candidate action sequences by imagining their future outcomes, which can then be used for policy evaluation~\cite{quevedo2025worldgym, team2025evaluating, li2025worldeval} and improvement~\cite{guo2025ctrl, guo2026vlaw, sharma2026world}.
The predicted futures \(\mathbf{\obs}\) are evaluated by a test-time criterion function \(\criterion(\mathbf{\obs})\in\mathbb{R}\), which measures whether \textit{plausible} outcomes with \textit{high impact} on policy evaluation and improvement occur in the generated videos.
These evaluations can be used to improve policies by selecting actions that avoid such outcomes~\cite{zhou2024dino,hansen2024tdmpc2}, or by training a policy \(\pi(\actionSequence\mid\historyObs)\) using world-model imaginations~\cite{guo2025ctrl,guo2026vlaw}.

However, when multiple futures are possible under the same action sequence, \textit{robust} policy~\cite{moos2022robust,akella2025risk} should account for the distribution of outcomes that are plausible under the ego robot action:
\vspace{0.2in}
\begin{equation}\label{eq:robust_control}
    \eqnhighlight{vibrant_blue}{
    \actionSequence^* =
        \tikzmarknode{outer}{\argmin_{\actionSequence \in \bm{\mathcal{A}}}}\,
        \eqnhighlight{dark-gray}{
        \tikzmarknode{inner}{
            \max_{\textcolor{vibrant_orange}{
                \bm{\noise} \in \mathbb{R}^{\dimension}
            }}
        }\;
            \criterion\!\left(
                \wm_\wmParam(\bm{\noise}, \historyObs, \actionSequence)
            \right)
        }
    } \hspace{0.3in}
\end{equation}
\annotate[yshift=-0.5em]{below}{outer}{\textbf{\textit{Outer opt.}}: action robust to the distribution over plausible futures}
\annotate[yshift=1.2em,xshift=-15.0em]{above}{inner}{\textbf{\textit{Inner opt.}}: \textcolor{vibrant_orange}{high-dim. (e.g., \(\dimension\approx\)1M) Gaussian noise} generating high-impact, yet plausible futures}
\vspace{0.02in}

\noindent where the inner optimization chooses a Gaussian noise to steer the WM generation toward a worst-case plausible future that maximizes the criterion, while the outer optimization selects a robust action sequence that keeps the criterion low across plausible futures, even under worst-case outcomes. 
Various methods can be used for the outer optimization: \(\bm{\mathcal{A}}\) may be a discrete set in a sampling-based solver~\cite{zhou2024dino}, or a continuous action space over which a policy is optimized~\cite{guo2026vlaw}. However, solving the inner optimization is challenging, as it requires optimizing over a high-dimensional noise space.

\para{Steering Imagination via Inference-time Noise Optimization} Since the initial noise acts as a control variable for the generation of WMs (Sec.~\ref{sec:diffusion-noise}), the goal of the inner optimization corresponds to steering the WM imagination toward a high-impact outcome that remains plausible under the learned model. However, because evaluating each noise requires an expensive denoising process, repeated random sampling is highly inefficient and likely to miss rare but plausible critical outcomes within a limited sampling budget. Instead, we leverage the \textit{gradient} of a differentiable criterion function to steer generation more efficiently, optimizing the initial noise directly via gradient ascent: \begin{equation}\label{eq:gradient-ascent}
    \bm{\noise}_{i+1}
    =
    \bm{\noise}_i
    +
    \eta \,
    \nabla_{\bm{\noise}_i}
    \left[\,
        \criterion\!\left(
            \mathbf{\obs}_i
        \right)
    \right], \quad \text{where} \quad \mathbf{\obs}_i = \wm_{\wmParam}(\bm{\noise}_i, \historyObs, \actionSequence),
\end{equation}where \(\eta\) denotes a step size. Although this first-order gradient ascent requires a small number of iterative updates, each involving forward and backward passes through both the WM and the criterion function, the dense gradient signal makes it substantially more efficient than zeroth-order random sampling for discovering noises that induce plausible, high-impact video generations~\cite{tang2024inference,karunratanakul2024optimizing,eyring2024reno}.

\para{\textit{Challenge}\label{sec:challenges}: High-Dimensional Noise Space}
While gradient-based optimization in \eqref{eq:gradient-ascent} can steer diffusion model generation at inference time~\cite{tang2024inference, eyring2024reno}, applying it to video world models is challenging because their noise spaces are extremely high-dimensional, e.g., \(\approx 1\)M in driving~\cite{gao2024vista} and \(\approx 50\)K in manipulation~\cite{guo2025ctrl}. The optimization must guide generation toward meaningful, scene-dependent target events while preserving plausibility under the learned WM distribution. However, naive optimization can easily push the noise into out-of-distribution (OOD) regions of Gaussian noises, producing implausible videos (Fig.~\ref{fig:front-fig}), and the test-time criterion \(\criterion\) must be differentiable while reasoning about nuanced, task-relevant events that vary across scenes. Moreover, solving \eqref{eq:gradient-ascent} is inefficient because exact gradients require backpropagating through the iterative denoising process.

\section{\ours: Steering High-Dimensional Video World Models}\vspace{-0.05in} 

To address the challenges of steering video world models, we propose \ours, which defines a differentiable optimization criterion by combining a \textit{semantic} objective using Vision-Language Models (VLMs) with a \textit{plausibility} objective that prevents OOD noise: \(\criterion = \reward^{\text{sem}} + \reward^{\text{pla}}\). The semantic objective uses inference-time text specifications of high-impact outcomes to reason about whether the target events occur in the generated video (Sec.~\ref{sec:method-vlm}). The plausibility objective preserves the statistical properties of the standard Gaussian noise prior, thereby maintaining the plausibility of WM generations (Sec.~\ref{sec:high-dim-gaussian}). Finally, we approximate the noise gradient to avoid full backpropagation through the iterative denoising process, enabling efficient optimization (Sec.~\ref{sec:grad-approx}).

\subsection{Semantic Objective: Steering towards Target Events with VLM Gradients}\label{sec:method-vlm}

Since video WMs operate across diverse scenes and tasks, the high-impact target events for steering may change with the policy context, requiring a \textit{semantic} objective that can differentiably score whether scene-dependent target events occur in generated videos. We therefore leverage the general video-understanding capability of a Vision-Language Model (VLM) as the semantic objective. Given an inference-time text prompt \(\prompt\) describing the target high-impact outcome, the VLM measures how well the generated video \(\mathbf{\obs}\) matches the prompt. For example, in robotic manipulation, one may specify ``the coffee beans spill,'' while in driving, one may specify ``a collision occurs.'' 

To obtain a differentiable score, we use Qwen-VL~\cite{Qwen-VL} and prompt it to output a single token, \texttt{yes} or \texttt{no}, indicating whether \(\mathbf{\obs}\) matches \(\prompt\). We then define $\reward^{\text{sem}}$ as the difference in log token probabilities:\begin{equation}\label{eq:vlm-token-difference}
\reward^{\text{sem}}(\mathbf{\obs}\,;\,\prompt)
=
\log p^{\text{VLM}}(\texttt{yes}\mid \mathbf{\obs}, \prompt)
-
\log p^{\text{VLM}}(\texttt{no}\mid \mathbf{\obs}, \prompt).
\end{equation} Using the VLM's single-token probabilities makes the objective differentiable, providing rich gradient signals for optimizing high-dimensional noise that generates an inference-time target event.

\subsection{Plausibility Objective: Keeping High-dimensional Noise in the Typical Set}\label{sec:high-dim-gaussian}

Diffusion models are trained to generate samples from noise drawn from a Gaussian prior. When the noise leaves the \textit{typical set}~\cite{cover1999elements, betancourt2017conceptual,nalisnick2019detecting}, where most training-time noise samples lie, the resulting out-of-distribution (OOD) noise can produce implausible videos that are unlikely under the WM distribution or have degraded visual quality. Importantly, in high dimensions, the typical set is not the same as the region of highest probability density: some noise vectors, such as the zero vector, may have high density but are extremely unlikely to be sampled from the Gaussian prior (See Appendix~\ref{app:regularizer} for more details). Since gradient-based noise optimization can perturb the noise away from this typical set~\cite{samuel2023norm,tang2024inference,eyring2024reno,samuel2024generating,harrington2026noisediv}, we employ a \textit{plausibility objective} that encourages the noise to remain within the \textit{typical set} of the Gaussian prior, with a combination of functions of the noise that capture key statistics of high-dimensional Gaussian noise: \(
    \regularizer(\bm{\noise})
    =
    \lambda_1 \reward^{\text{norm}}(\bm{\noise})
    +
    \lambda_2 \reward^{\text{iso}}(\bm{\noise})
    +
    \lambda_3 \reward^{\text{spec}}(\bm{\noise})
\).

\para{Norm Concentration}\label{sec:norm-concentration} First, since the squared norm of Gaussian noise follows a chi-square distribution,
\(\|\bm{\noise}\|_2^2 \sim \chi^2_{\dimension}\), it concentrates sharply around the noise dimension \(\dimension\)~\cite{samuel2023norm}. We therefore regularize deviations of the noise radius from the typical Gaussian shell:
\(
    \reward^{\text{norm}}(\bm{\noise})
    :=
    -\left(
        \|\bm{\noise}\|_2
        -
        \sqrt{\dimension}
    \right)^2.
\)

\para{Isotropy}
While norm concentration enforces typicality at a \textit{global} scale, the optimized noise may still contain \textit{local} correlations or structured coordinate patterns that are unlikely under an i.i.d. Gaussian prior~\cite{tang2024inference}. To encourage coordinate-wise independence, we randomly permute the entries of \(\bm{\noise}\) and partition them into \(m\) subvectors \(\{\bm{\noise}_i\}_{i=1}^m\), where \(\bm{\noise}_i \in \mathbb{R}^k\) and \(\dimension=mk\). Under the Gaussian prior, these subvectors should behave as i.i.d. samples from \(\mathcal{N}(\mathbf{0},\mathbf{I}_k)\), and therefore their empirical second moment \(\widehat{\mathbf{\Sigma}}=\frac{1}{m}\sum_{i=1}^{m}\bm{\noise}_i\bm{\noise}_i^\top\) should be close to \(\mathbf{I}_k\). We therefore penalize deviations from blockwise isotropy as \(\reward^{\mathrm{iso}}(\bm{\noise}) := -\frac{1}{k}\|\widehat{\mathbf{\Sigma}}-\mathbf{I}_k\|_F^2\), averaged over multiple random permutations.

\para{Spectral Whiteness}
Even when the noise remains typical in coordinate space, optimization can introduce \textit{frequency-domain} artifacts~\cite{durall2020watch, hwang2026moment, hwang2026gradient}. Since Gaussian noise has a flat expected power spectrum~\cite{stoica2005spectral,vaseghi2008advanced}, we encourage the noise to remain spectrally white. Let \(\mathcal{F}\) denote a unitary 2D discrete Fourier transform applied over the spatial dimensions of each noise slice, and compute the power spectrum \(\mathbf{P}=|\mathcal{F}(\bm{\noise})|^2\). We aggregate \(\mathbf{P}\) into \(B\) spatial-frequency bins to obtain bin-averaged powers \(\{\widehat{p}_b\}_{b=1}^{B}\). To preserve spectral flatness, we minimize the variance of these bin-averaged powers, \(\reward^{\mathrm{spec}}(\bm{\noise}):= -\frac{1}{B}\sum_{b=1}^{B}(\widehat{p}_b-\bar{p})^2\), where \(\bar{p}=\frac{1}{B}\sum_{b'=1}^{B}\widehat{p}_{b'}\) is the mean power across bins.

\subsection{Approximating Noise Gradients for Efficient Steering}\label{sec:grad-approx}

Noise optimization in \eqref{eq:gradient-ascent} requires computing the gradient of the criterion with respect to the initial noise,
\(
\nabla_{\bm{\noise}}\,
\criterion\!\left(
    \wm_{\wmParam}(\bm{\noise}, \historyObs, \actionSequence)
\right)
=
\nabla_{\bm{\noise}}\,
\criterion\!\left(
    \denoising^1\!\big(
        \cdots
        \denoising^T(\bm{\noise}, \historyObs, \actionSequence)
    \big)
\right),
\)
where \(\denoising^\diffusionTime(\cdot)\) denotes one denoising step in \eqref{eq:denoising}. Although diffusion models are differentiable in principle, backpropagating through the full iterative denoising process (e.g., \(50\) steps~\cite{blattmann2023stable}) can incur high memory cost and vanishing gradients~\cite{eyring2025noise,ahn2024noise}. To make optimization efficient, we adopt score-distillation~\cite{ahn2024noise} (Appendix~\ref{app:score-distillation}), which approximates the gradient with respect to the noise using the gradient at the generated sample:
\begin{equation}
     \nabla_{\bm{\noise}}
    \criterion\!\left(
        \mathbf{\obs}
    \right)
    \approx
    \beta\,
    \nabla_{\mathbf{\obs}}
    \criterion(\mathbf{\obs}),   \quad \text{where} \quad \mathbf{\obs}= \wm_{\wmParam}(\bm{\noise}, \historyObs, \actionSequence) = \denoising^1\!\big(
        \cdots
        \denoising^T(\bm{\noise}, \historyObs, \actionSequence)
    \big),
\end{equation}
where \(\beta\) is a scalar coefficient. This approximation only requires backpropagation through the differentiable criterion function, avoiding full backpropagation through the iterative denoising process. 

In sum, \ours optimizes the initial noise as in \eqref{eq:gradient-ascent}, using a gradient that combines the \textcolor{wine}{gradient approximation}, the \textcolor{violet}{semantic objective (VLM)}, and the \textcolor{green-reg}{plausibility objective (typical set)}:\begin{equation}\label{eq:gradient-ascent-final}
    \nabla_{\bm{\noise}}
    \,
        \criterion\!\left(
           \mathbf{\obs}
        \right)
     = \,
    \eqnmark[wine]{approx}{\beta \,\nabla_{\mathbf{\obs}}}
        \, \eqnmarkbox[violet]{vlm}{\reward^{\text{sem}}\left(
            \mathbf{\obs} ; \prompt
        \right)}
        +
        \eqnmarkbox[green-reg]{regularizer}{\, \nabla_{\bm{\noise}}\,\regularizer(\bm{\noise})},
\end{equation} where the coefficients \(\beta,\lambda_1,\lambda_2,\lambda_3\) are tuned depending on the WM, noise dimension, and VLMs.

\section{Case Study: Steering Dubins Car Video World Model}
\label{sec:dubins} 

We first study \ours in controlled settings where the system dynamics and uncertainty are known. Since we know the true dynamics, this setup allows us to test whether WM steering can effectively identify high-impact futures (i.e., failures) only when they are possible under the system.

\para{\textit{Naughty} 3D Dubins Car}
We consider an image-based setup of a discrete-time 3D Dubins car with state \(\state=[p_x,p_y,\theta]\), continuous angular-velocity action \(\action_t\in[-1.25,\,1.25]\,\mathrm{rad/s}\), fixed speed \(v=1\,\mathrm{m/s}\), and timestep \(\Delta t=0.05\,\mathrm{s}\). The \textit{naughty} Dubins car introduces uncertainty by randomly flipping the sign of the control input: \(\delta_t=-1\) with probability \(p=0.2\), and \(\delta_t=1\) otherwise:
\vspace{-0.00in}\begin{equation}\label{eq:naughty-dynamics}
    \state_{t+1}
    =
    \state_t
    +
    \Delta t\,[\, v\cos(\theta_t),\; v\sin(\theta_t),\; \delta_t \action_t \,],
    \qquad
    \delta_t \in \{-1,1\}.
\end{equation} Each state is rendered as a \(3\times128\times128\) RGB image showing the environment and vehicle. In this setup, instead of the VLM-based $\reward^{\text{sem}}$, we use the safety score defined as \(\reward(\state)=p_x^2+p_y^2-0.25^2\), which induces a circular failure set of radius \(0.25\,\mathrm{m}\) centered at the origin: \( \mathcal{F}:=\{\state \mid \reward(\state)\le 0\}.\)

\para{Video World Model} We train a WM on an offline dataset of \(4{,}000\) random observation--action trajectories. We train a safety-score predictor to regress the ground-truth score from an image, and train a diffusion world model, \(\wm_{\wmParam}(\obs_{t+1}\mid \noise,\obs_{t},\action_{t})\), to generate one-step future observation \((\horizon=1)\) conditioned on the current observation and action. The resulting noise dimension is \(1024\).

 \begin{wrapfigure}{r}{0.5\textwidth}
    \centering
    \vspace{-0.3in}
    \includegraphics[width=\linewidth]{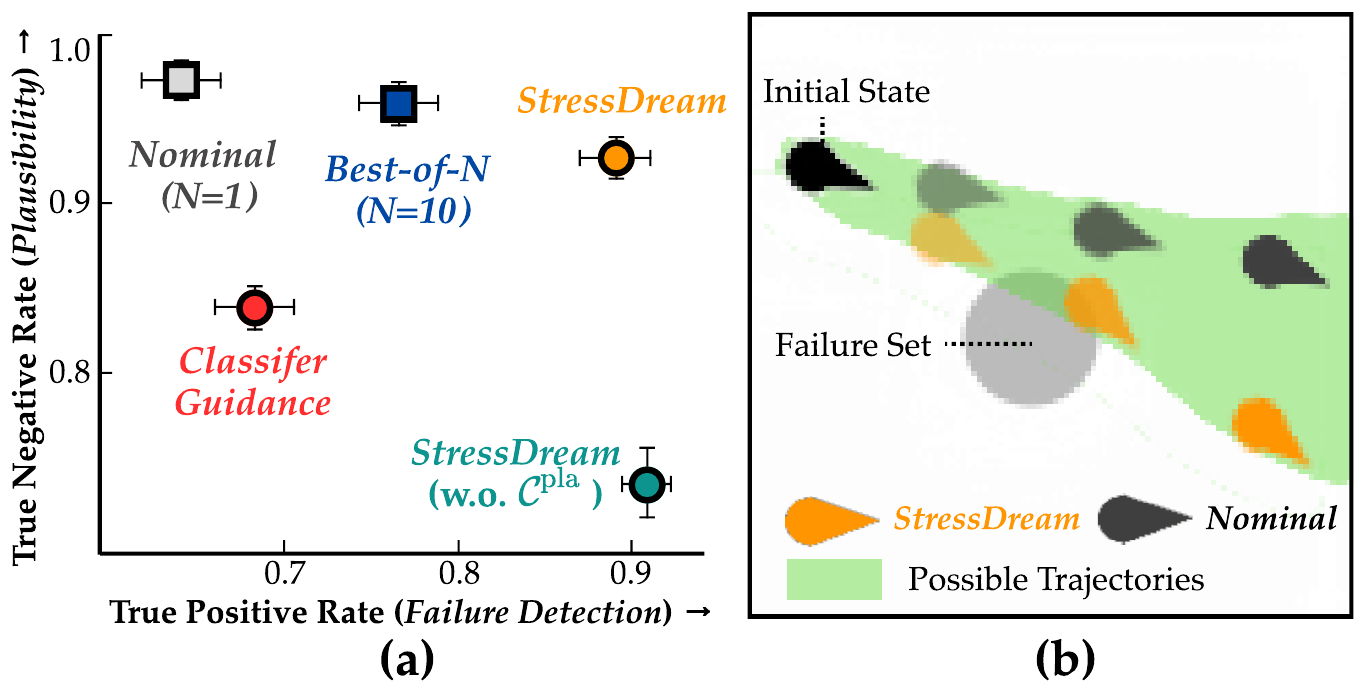} 
    \vspace{-0.25in}
    \caption{\small{\textit{Results: Naughty Dubins Car.}
    (a): TPR and TNR for detecting possible failures in WM imaginations.
    (b): Overlaid nominal and steered imaginations.
    \ours steers WM imaginations, detecting plausible failures that nominal imagination misses. See Appendix~\ref{sec:appendix-dubins} \& \href{https://stressdream.github.io/dubins/index.html}{Website} for more results and details.}}
    \vspace{-0.22in}
    \label{fig:dubins_results}
\end{wrapfigure}
\para{Experiment Setup} We aim to detect action sequences that have a possibility of entering the failure set using WM imaginations. We construct an evaluation dataset of \(5{,}000\) random initial state and action sequences. For each trajectory, we compute the ground-truth minimum safety score achievable under the stochastic dynamics in \eqref{eq:naughty-dynamics}, and label it as \textit{positive} if the g.t. minimum score is below zero, indicating that the trajectory can enter the failure set. We then perform open-loop world-model rollouts using only the initial image \(\obs_0\) and the action sequence, and classify an imagined trajectory as a failure, i.e., positive, if its minimum predicted safety score along the rollout is below zero.

\para{Baselines} \ours steers imagination by minimizing the predicted safety-score for \(10\) optimization steps. We compare against sampling without optimization: \nominal (\(N\)=\(1\)), and \bestofN (\(N\)=\(10\)). We ablate \ours without \(\regularizer\), and compare against classifier guidance~(\classGuidance)~\cite{dhariwal2021diffusion,li2022upainting}, which applies gradients during denoising instead of optimizing the initial noise.

\para{\ours reliably detects failures only when they are plausible} Fig.~\ref{fig:dubins_results} shows that \ours accurately detects potential failures in WM imaginations, achieving a high true positive rate, i.e., correctly identifying trajectories that can fail, while maintaining a high true negative rate, i.e.,  not falsely classifying safe trajectories as failures. In contrast, \ours without the plausibility objective in Sec.~\ref{sec:high-dim-gaussian} yields a low true-negative rate, indicating implausible failures are imagined without the typical set constraint. \classGuidance similarly produces many false positives by directly modifying the denoising trajectory, generating implausible imaginations. Random sampling methods (\nominal, \bestofN) generate plausible imaginations, but are less effective at identifying rare possible failures compared to \ours, often missing them even with multiple samples. 

\section{Experiments: Robust Policy Evaluation and Improvement}
\label{sec:sota}

In this section, we validate \ours for robust policy evaluation and improvement using state-of-the-art video world models in both autonomous driving~\cite{gao2024vista} and robotic manipulation~\cite{guo2025ctrl}.

\para{Autonomous Driving} We use \vista~\cite{gao2024vista} that predicts \(25\) future front-view camera frames at \(576 \times 1024\), conditioned on future waypoints as actions, with noise dimension \(\dimension = 921{,}600\). We fine-tune the model on the PhysicalAI-Autonomous-Vehicles (PAI-AV) dataset~\cite{nvidia_physicalai_av} and the Nexar Collision Prediction Dataset~\cite{moura2025nexar}, which include diverse driving scenarios such as collisions. We optimize for $20$ steps with Qwen2.5-VL-7B-Instruct~\cite{bai2025qwen2} fine-tuned with Wolf~\cite{li2025wolf} and X-CLIP~\cite{ma2022x}.

\para{Robotic Manipulation} We use \ctrlworld~\cite{guo2025ctrl} trained in the DROID setup~\cite{khazatsky2024droid}, which predicts \(5\) future frames from \(3\) camera views of resolution \(192 \times 320\), conditioned on joint-position actions, with \(\dimension = 57{,}600\). We consider six contact-rich manipulation tasks (See Fig.~\ref{fig:pi-qualitative}) and collect approximately \(150\) teleoperation trajectories per task containing both successes and failures to fine-tune the model. We optimize noise for $10$ steps using Qwen3-VL-4B-Instruct~\cite{Qwen3-VL}, providing the three-view videos along with a prompt describing the task failure (More details in Appendix~\ref{sec:app-sota-vwm}).

\subsection{Robust Policy Evaluation using State-of-the-Art Video World Models}\label{sec:experiments-inner} \vspace{-0.05in}
We first study whether \ours enables robust policy evaluation by efficiently steering WM imaginations toward high-impact but plausible outcomes specified at inference time. As a baseline, we compare against random generation (\bestofN), which samples multiple imaginations without optimization with the same budget as \ours, and selects the one with the highest score.

\begin{figure}[t]
    \centering
    \includegraphics[width=1.0\linewidth]{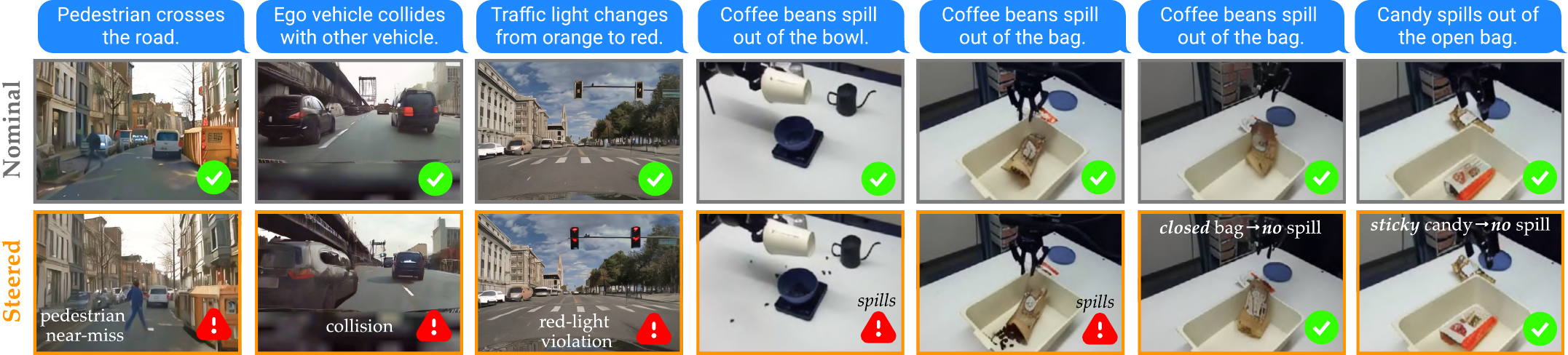}
    \vspace{-0.2in}
    \caption{\small{
        \textit{Nominal vs.\ Steered Imaginations.} Top texts show the inference-time prompts describing target outcomes. \ours steers WM imaginations toward specified high-impact outcomes that nominal generation misses. Imaginations are grounded in plausible outcomes: when target outcomes are not supported by the WM distribution, e.g., spilling sticky candies or from a closed bag, \ours does not imagine them.
        }}
    \label{fig:inner_opt_results}
    \vspace{-0.15in}
\end{figure}

\begin{wrapfigure}{r}{0.48\textwidth}
    \centering
    \vspace{-0.18in} 
    \includegraphics[width=\linewidth]{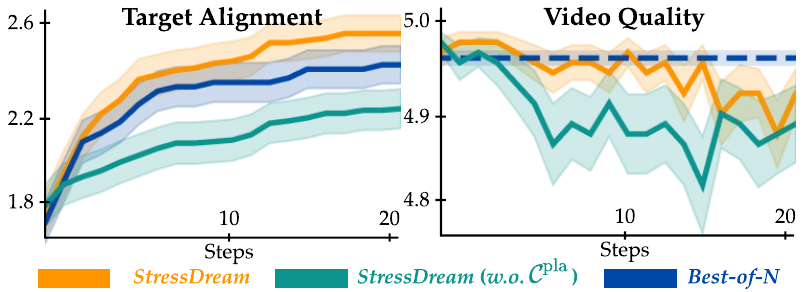} 
    \vspace{-0.2in}
    \caption{\small{\textit{Steering Driving Video World Models.} 
    \ours steers WM imaginations toward the target events more effectively than random sampling (\bestofN), while $\regularizer$ helps preserve video quality.
    }}
    \vspace{-0.1in}\label{fig:vista-results}
\end{wrapfigure}\vspace{-0.03in}\para{Evaluation Setup} We generate imaginations from an initial observation \(\obs_0\) and action sequence \(\actionSequence\), and evaluate whether the generated video matches the inference-time text specification \(\prompt\) of the target event. For driving, we curate \(100\) image--action--text pairs from the PAI-AV dataset~\cite{nvidia_physicalai_av} across \(8\) safety-critical event categories (see Appendix~\ref{sec:app-av-curation}), as well as \(200\) imminent-collision examples~\cite{moura2025nexar}. Each prompt describes a target event that has not yet occurred in the initial observation but occurs within the prediction horizon. For manipulation, we collect \(100\) failure trajectories across tasks as an evaluation dataset, and evaluate whether WM imaginations from the initial images and the action sequences can reliably identify failures (See Appendix~\ref{sec:appendix-manipulation-details}).

\para{Evaluation Metrics}
For driving, we report \textit{target alignment} and \textit{video quality} using a held-out evaluator, WorldModelBench~\cite{Li2025WorldModelBench}. Target alignment measures whether the video matches the text description of the target event, while video quality serves as a proxy for plausibility. For manipulation, we label each generation as a success or failure based on human judgment. (See Appendix~\ref{sec:app-evaluation-metrics})

\begin{wrapfigure}{r}{0.18\textwidth}
    \centering
    \vspace{-0.1in} 
    \includegraphics[width=\linewidth]{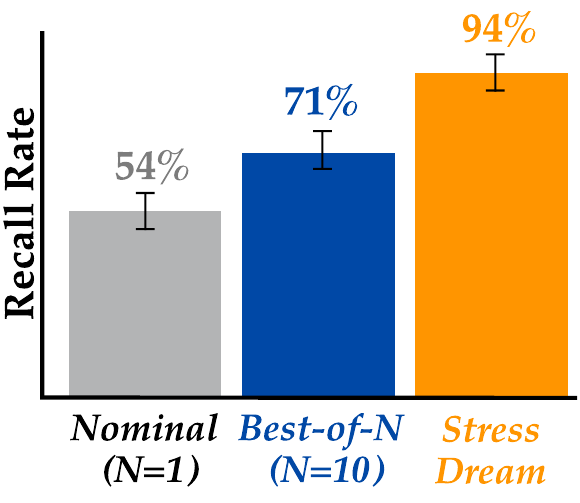} 
    \vspace{-0.2in}
    \caption{\small{\textit{Steering \ctrlworld:} Detecting task failures in WM imaginations.}}
    \vspace{-0.25in}\label{fig:manipulation-planning}
\end{wrapfigure} 
\vspace{-0.03in}\para{\ours robustly detects high-impact outcomes in imaginations} Fig.~\ref{fig:inner_opt_results} compares nominal and steered generations. \ours plausibly generates videos aligned with inference-time specifications of target safety-critical or failure events that nominal imaginations miss, enabling robust evaluation of action sequences through WM imaginations. Fig.~\ref{fig:vista-results} shows that, without the typical-set constraint, both target-alignment and video-quality scores decrease, indicating that the plausibility objective helps preserve plausibility. For manipulation, \ours detects task-failure events in imagination with higher recall, whereas random generations are often overly optimistic and miss possible failures, as shown in Fig.~\ref{fig:manipulation-planning}. More results in Appendix~\ref{sec:appendix-results-policy-eval}~\&~\href{https://stressdream.github.io/#policy-eval}{\textcolor{vibrant_orange}{Website}}.

\begin{wrapfigure}{r}{0.25\textwidth}
    \centering
    \vspace{-0.00in}
    \includegraphics[width=\linewidth]{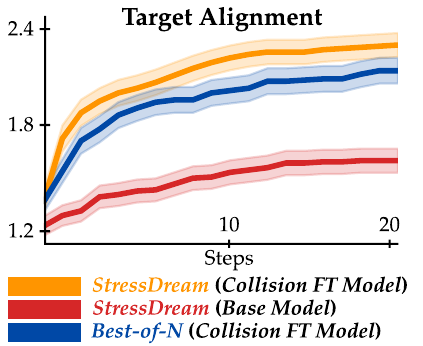} 
    \vspace{-0.18in}
    \caption{\small{\textit{Steering to collision}: \ours can induce target events only when they are plausible.}}
    \vspace{-0.2in}
    \label{fig:vista-ablation}
\end{wrapfigure}\vspace{-0.03in}\para{Steering is grounded in plausible outcomes} We further study whether steering remains grounded in plausible outcomes supported by the WM's distribution, rather than hallucinating implausible futures. To test this, we steer generations toward collision events using both the \textit{collision-finetuned} \vista model and the \textit{base} \vista model, which is not trained on collision events. Fig.~\ref{fig:vista-ablation} shows that the base model cannot imagine collision outcomes even with steering, achieving lower target-alignment scores than random sampling from the collision-finetuned model. This indicates that \ours steers only toward events supported by the WM outcome distribution, rather than synthesizing implausible events. Similarly, the right two columns of Fig.~\ref{fig:inner_opt_results} show that target spilling outcomes are induced only when plausible under the WM distribution: \ours does not induce spilling for the open \textit{sticky} candy bag or \textit{closed} coffee bag, but does so for the open coffee bag. Since sticky candies are unlikely to spill and a closed bag does not expose its contents, spilling is not supported by the WM distribution in those scenes. More details and results in Appendix~\ref{sec:appendix-plausibility}.

\subsection{Policy Improvement with Steered Video World Model Imaginations}
\label{sec:exp-pi}

\begin{figure}[t]
    \centering
    \includegraphics[width=1.0\linewidth]{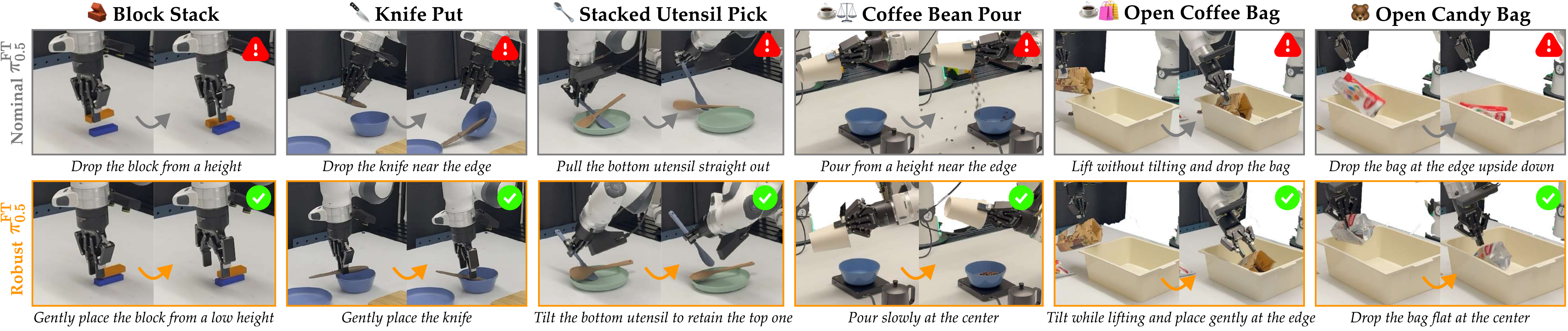}
    \vspace{-0.15in}
    \caption{\small{\textit{Policy Improvement via \ours.} Fine-tuning \(\pi_{0.5}\) with steered WM imaginations favors robust actions that remain successful under worst-case plausible outcomes, whereas nominal fine-tuning can propose actions where failures are plausible under the outcome distribution. Videos available at the \href{https://stressdream.github.io/\#policy-improve}{website}.}}
    \label{fig:pi-qualitative}
    \vspace{-0.1in}
\end{figure}

\para{Setup}
We study whether \ours can improve a robotic policy through steered WM imaginations, promoting robust actions that remain successful over the distribution of outcomes. We consider improving a behavior-cloning policy, the Vision-Language-Action (VLA) model \(\pi_{0.5}\)~\cite{intelligence2025pi}, with action evaluation in WM imaginations. Specifically, we fine-tune \(\pi_{0.5}\)-\textsc{droid} with \(40\) successful demonstrations per task using a weighted regression objective~\cite{guo2026vlaw}, and compare two settings: \textit{Nominal} \(\pi_{0.5}^{\mathrm{FT}}\), which assigns a uniform weight of \(1.0\) to all trajectories, and \textit{Robust} \(\pi_{0.5}^{\mathrm{FT}}\), which assigns weight \(1.0\) to trajectories that remain successful under steering toward task failures and weight \(0.1\) to trajectories that fail in steered imaginations. This encourages \textit{robust} actions whose plausible outcome distribution does not include failures, while downweighting \textit{risky} actions that can fail.

\begin{wrapfigure}{r}{0.15\textwidth}
    \centering
    \vspace{-0.15in} 
    \includegraphics[width=\linewidth]{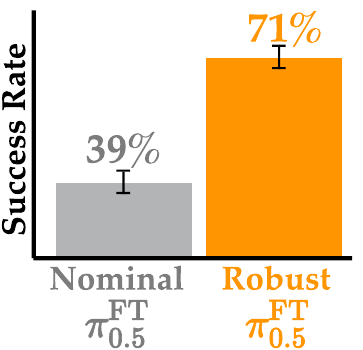} 
    \vspace{-0.22in}
    \caption{\small{\textit{$\pi_{0.5}$ Rollout Results.}}}
    \vspace{-0.12in}\label{fig:manipulation-pi}
\end{wrapfigure}
\para{\ours improves policies by promoting robust actions} Fig.~\ref{fig:pi-qualitative} illustrates rollouts of fine-tuned \(\pi_{0.5}\) policies. \textit{Robust} \(\pi_{0.5}^{\mathrm{FT}}\) favors actions whose plausible outcomes do not include task failures, as identified by steering imaginations. In contrast, \textit{Nominal} \(\pi_{0.5}^{\mathrm{FT}}\) proposes risky actions whose plausible outcome distribution includes failures, even if they happened to succeed during data collection. Fig.~\ref{fig:manipulation-pi} reports success rates over \(20\) rollouts per task, showing that \textit{Robust} \(\pi_{0.5}^{\mathrm{FT}}\) achieves higher success by proposing robust actions that avoid potential failures.

\vspace{-0.1in}
\section{Conclusion}\label{sec:conclusion}\vspace{-0.1in}
In this work, we propose \ours, which steers video world model imaginations toward high-impact outcomes specified at inference-time, while remaining within the distribution of plausible outcomes. \ours optimizes the high-dimensional initial noise of diffusion-based video world models using a combination of semantic and plausibility objectives: VLM gradients steer generation toward high-impact outcomes specified by inference-time prompts, while typical-set constraints keep the high-dimensional noise typical under the Gaussian prior. With WMs for autonomous driving and robotic manipulation, we demonstrate that \ours effectively steers imaginations while preserving plausibility, enabling robust policy evaluation by identifying actions with potential failures and improving policies by favoring robust actions that avoid such failures.

\para{Limitations}
\ours requires specifying high-impact outcomes with text descriptions, so its effectiveness depends on prompt quality and may suffer from reward hacking, where the score increases without meaningful changes in generations~\cite{eyring2024reno,hong2026understanding}, highlighting the need for generalizable and robust reward models for robotics~\cite{tan2025robo, liang2026robometer}. In addition, \ours may fail to preserve plausibility or to discover relevant failures when the base WM is limited: it may steer within flawed videos produced by the WM, or fail to imagine target outcomes absent from the WM's training distribution. Thus, our notion of plausibility is limited to what the WM supports, and does not necessarily imply physical plausibility in the real world, underscoring the need for diverse robot data to train high-fidelity WMs with physically consistent predictions. Runtime also remains expensive, as current WMs take several minutes per imagination~\cite{gao2024vista,guo2025ctrl}, though this could be reduced by efficient models~\cite{song2023consistency,kim2024consistency,esser2024scaling}. We refer readers to Appendix~\ref{sec:appendix-limitations} for a more detailed discussion.

\acknowledgments{JS would like to thank Luca Eyring for insightful discussions on noise optimization, and Lasse Peters for invaluable feedback on the writing. This research was supported in part by funding from the Defense Advanced Research Projects Agency (DARPA) Young Faculty Award (YFA). The views, opinions and/or findings expressed are those of the author and should not be interpreted as representing the official views or policies of DARPA or the U.S. Government. This work used Bridges-2 at Pittsburgh Supercomputing Center through allocation CIS251277 from the Advanced Cyberinfrastructure Coordination Ecosystem: Services \& Support (ACCESS) program, which is supported by National Science Foundation grants \#2138259, \#2138286, \#2138307, \#2137603, and \#2138296.
}


\bibliography{reference}  

\clearpage

{\Large\bfseries Project Website\par}
\vspace{0.5em}

{\large
Video results and interactive demos are available at
\href{https://stressdream.github.io/}{\texttt{stressdream.github.io}}.
\par}

\appendix

\let\addcontentsline\oldaddcontentsline
\renewcommand{\contentsname}{\Large Appendix}
\setcounter{tocdepth}{4}

\makeatletter

\renewcommand*\l@section[2]{%
  \ifnum \c@tocdepth >\z@
    \addvspace{0.8em}
    \begingroup
      \parindent \z@
      \rightskip \@pnumwidth
      \parfillskip -\@pnumwidth
      \leavevmode
      \large\bfseries
      \setlength\@tempdima{1.8em}%
      \advance\leftskip\@tempdima
      \hskip -\leftskip
      #1\nobreak\hfill\nobreak
      \hb@xt@\@pnumwidth{\hss #2}\par
    \endgroup
  \fi}

\renewcommand*\l@subsection[2]{%
  \addvspace{0.25em}
  \@dottedtocline{2}{1.8em}{2.6em}{#1}{#2}}

\renewcommand*\l@subsubsection[2]{%
  \addvspace{0.15em}%
  \@dottedtocline{3}{4.4em}{3.4em}{#1}{#2}}

\renewcommand*\l@paragraph[2]{%
  \addvspace{0.1em}%
  \@dottedtocline{4}{7.8em}{4.2em}{#1}{#2}}

\makeatother

\begingroup
\large
\tableofcontents

\endgroup

\clearpage

\appendix

\section{Background}\label{app:diffusion-background}

\subsection{Related Works}

\para{World Models for Robotics} World models have emerged as a promising paradigm for robotics, enabling the prediction of future states directly from high-dimensional sensory observations conditioned on actions~\cite{finn2017deep, ebert2018visual, hafner2019learning, wu2023daydreamer, zhou2024dino, micheli2023transformers, bruce2024genie, agarwal2025cosmos, maes2026leworldmodel}. By generating synthetic trajectories, they have proven effective for vision-based policy learning~\cite{Hafner2020Dream, hafner2021mastering,hafner2023dreamerv3,hafner2025training, nakamura2025generalizing, agrawal2026anysafe, sharma2026world, guo2026vlaw}, planning in latent space~\cite{zhou2024dino, psenka2026parallel}, and policy evaluation~\cite{hansen2024tdmpc2, guo2025ctrl, quevedo2025worldgym}. Advances in video generation models~\cite{blattmann2023stable}, which produce realistic and temporally consistent predictions, have further increased the expressiveness of world models in diverse real-world scenarios, especially when fine-tuned with embodiment data for action conditioning~\cite{gao2024vista, hassan2025gem, guo2025ctrl}. World models typically operate under partial observability, which naturally gives rise to uncertainty and multimodality in future predictions~\cite{walker2016uncertain, zhang2023storm, kerssies2026frame}, particularly in settings such as contact-rich manipulation~\cite{guo2025ctrl} or autonomous driving~\cite{gao2024vista, russell2025gaia, bartoccioni2025vavim, hassan2025gem}. In particular, diffusion- and flow-matching-based~\cite{ho2020denoising, karras2022elucidating, lipman2023flow} approaches have shown strong performance in modeling high-dimensional, complex environments due to their ability to capture rich, multimodal distributions over future observations~\cite{alonso2024diffusion, chen2024diffusion, gao2025adaworld, kerssies2026frame}. Despite this, many existing methods using world models either rely on nominal rollouts without reasoning over distribution over plausible futures~\cite{chen2024diffusion, guo2025ctrl, seo2025uncertainty} or adopt deterministic world models that cannot represent such diversity~\cite{zhou2024dino, maes2026leworldmodel}, thereby failing to explicitly account for risks arising from uncertain futures.

\para{Inference-time Steering of Generative Models} Inference-time steering of diffusion~\cite{ho2020denoising, song2021scorebased} or flow-matching~\cite{lipman2023flow, liu2023flow} model aims to guide a pretrained generative model toward samples that maximize an inference-time reward~\cite{tang2024inference, uehara2024fine, uehara2025inference, holderrieth2026diamond}. While directly fine-tuning diffusion models with reward objectives can be effective for reward alignment~\cite{fan2023dpok, wallace2024diffusion, prabhudesai2024video, clark2024directly, uehara2024fine, domingo-enrich2025adjoint}, it is costly, complex, tied to a specific reward~\cite{tang2024inference, holderrieth2026diamond}, and may shift the model away from the original data distribution~\cite{eyring2025noise}. In contrast, inference-time steering keeps the pretrained model fixed and instead modifies the sampling procedure~\cite{chung2023diffusion, he2024manifold, tang2024inference, uehara2025inference, singhal2025a, yuan2026inference}, for example through classifier guidance~\cite{song2021scorebased} or particle-based sampling methods such as Sequential Monte Carlo (SMC)~\cite{wu2023practical, dou2024diffusion, oshima2026inferencetime, holderrieth2026diamond}. One particularly relevant direction is \textit{noise optimization}, which optimizes the noise sampled at the denoising process~\cite{wallace2023end, guo2024initno, karunratanakul2024optimizing, eyring2024reno, novack2024ditto, tang2024inference}. Prior work has shown that this can improve image quality~\cite{guo2024initno, chen2024find, eyring2025noise}, mimic the effects of diffusion guidance~\cite{ahn2024noise}, generate rare~\cite{samuel2023norm,samuel2024generating, samuel2025lightningfast} or diverse concepts~\cite{kim2026diverse}, and adapt diffusion policies in robotic control~\cite{wagenmaker2025steering}. However, a central challenge, especially when the noise space becomes high-dimensional, is out-of-distribution noise: samples may achieve high reward while no longer remaining in the support of the pretrained distribution once the noise deviates from the prior~\cite{samuel2023norm, ben-hamu2024dflow, tang2024inference, eyring2024reno, zhai2025mira, eyring2025noise, harrington2026noisediv}. Moreover, these methods often require backpropagating reward gradients through the denoising process; in practice, many therefore rely on step-distilled models to make backpropagation tractable~\cite{Novack2024Ditto2, eyring2024reno, eyring2025noise, ding2025dollar}, while recent works propose multistep score distillation to approximate reward gradients to the initial noise more efficiently~\cite{ahn2024noise}. Our approach adopts initial noise optimization to steer action-conditioned video world models~\cite{gao2024vista, guo2025ctrl} with a practical solution for high-dimensional noise optimization.

\subsection{Diffusion Models}
Recent generative models are often based on diffusion or flow-matching formulations, which generate samples from noise by simulating a learned differential equation. This section expands on the diffusion-model formulation in \eqref{eq:denoising}, which provides the foundation for recent video world models capable of generating high-quality videos. We omit conditioning on \((\historyObs, \actionSequence)\) throughout for brevity.

\para{Forward Probability Path} Score-based diffusion models~\cite{ho2020denoising, karras2022elucidating, song2020denoising} and flow-matching models~\cite{albergo2023building, lipman2023flow} define generative models \(\wm_{\wmParam}\) whose sampling process can be described by a stochastic or ordinary differential equation (SDE or ODE). The denoising network \(\denoise^{\diffusionTime}_{\wmParam}\), which parameterizes these dynamics, is trained to recover the reverse generative process, typically by predicting the score of the perturbed data distribution, the injected noise, or the conditional vector field.

Starting from clean data \(\mathbf{\data}=\mathbf{\obs}\sim\dataDistribution(\obs)\) and isotropic Gaussian noise \(\bm{\noise}\sim\mathcal{N}(0,\mathbf{I})\), the forward process, indexed by diffusion time \(\diffusionTime\in\{0,1,\dots,\diffusionTimeMax\}\), defines a perturbed sample as
\begin{equation}\label{eq:app-forward}
    \mathbf{\data}^{\diffusionTime}
    =
    \diffusionCoeff^{\diffusionTime}\,\mathbf{\data}
    +
    \noiseCoeff^{\diffusionTime}\,\bm{\noise}.
\end{equation}
For variance-preserving parameterizations such as DDPM, the boundary conditions satisfy
\(\diffusionCoeff^{0}=1\), \(\noiseCoeff^{0}=0\), and
\(\diffusionCoeff^{\diffusionTimeMax}\to 0\),
\(\noiseCoeff^{\diffusionTimeMax}\to 1\).
More generally, we only require that the terminal state be noise-dominated. Representative examples include the variance-preserving schedule of DDPM~\cite{ho2020denoising}, where \((\diffusionCoeff^{\diffusionTime})^2+(\noiseCoeff^{\diffusionTime})^2=1\), and the variance-exploding schedule of EDM~\cite{karras2022elucidating}, where \(\diffusionCoeff^{\diffusionTime}\equiv 1\) and \(\noiseCoeff^{\diffusionTime}\) increases with diffusion time.

\para{Training Objective}
The denoiser \(\denoise^{\diffusionTime}_{\wmParam}\) is trained to recover the injected Gaussian noise from a perturbed future observation, conditioned on the observation history and action sequence. With the noise-prediction parameterization,
\begin{equation}\label{eq:app-train}
    \mathcal{L}(\wmParam)
    =
    \mathbb{E}_{\mathbf{\data}\sim\dataDistribution,\,\bm{\noise}\sim\mathcal{N}(0,\mathbf{I}),\,\diffusionTime}
    \left[\,
        w(\diffusionTime)\,
        \big\|\,
            \bm{\noise}
            -
            \denoise^{\diffusionTime}_{\wmParam}\!\big(
                \diffusionCoeff^{\diffusionTime}\mathbf{\data}
                +
                \noiseCoeff^{\diffusionTime}\bm{\noise}
            \big)
        \big\|_2^2
    \,\right],
\end{equation}
where \(w(\diffusionTime)\) is a time-dependent weighting~\cite{karras2022elucidating}. The clean-sample, score, and velocity parameterizations are linearly equivalent~\cite{karras2022elucidating,song2021scorebased}; we take \(\denoise^{\diffusionTime}_{\wmParam}\) as a noise predictor, and the analysis transfers to the other parameterizations.

\para{Generation via Deterministic ODE Updates}
At inference time, generation starts from an initial Gaussian noise sample \(\mathbf{\data}^{\diffusionTimeMax}=\bm{\noise}\sim\mathcal{N}(0,\mathbf{I})\) and iteratively maps it to a clean sample by following the reverse probability-flow ODE. The deterministic update used in \eqref{eq:denoising} can be viewed as a first-order Euler step of this ODE, with DDIM~\cite{song2020denoising} and EDM~\cite{karras2022elucidating} as representative instances. Given the current noisy sample \(\mathbf{\data}^{\diffusionTime}\), the predicted clean sample is obtained by inverting the forward path \eqref{eq:app-forward}:
\begin{equation}\label{eq:app-clean-pred}
    \hat{\mathbf{\data}}^{0\mid\diffusionTime}
    =
    \frac{
        \mathbf{\data}^{\diffusionTime}
        -
        \noiseCoeff^{\diffusionTime}\,
        \denoise^{\diffusionTime}_{\wmParam}\!\left(
            \mathbf{\data}^{\diffusionTime}
        \right)
    }{
        \diffusionCoeff^{\diffusionTime}
    }.
\end{equation}
The sample is then advanced to the next, less-noisy time \(\diffusionTime-1\) by reapplying the forward path to \(\hat{\mathbf{\data}}^{0\mid\diffusionTime}\), while using the predicted noise as the noise direction:
\begin{equation}\label{eq:app-ddim}
    \mathbf{\data}^{\diffusionTime - 1}
    =
    \diffusionCoeff^{\diffusionTime - 1}\,
    \hat{\mathbf{\data}}^{0\mid\diffusionTime}
    +
    \noiseCoeff^{\diffusionTime - 1}\,
    \denoise^{\diffusionTime}_{\wmParam}\!\left(
        \mathbf{\data}^{\diffusionTime}
    \right).
\end{equation}
Iterating this deterministic update from \(\diffusionTimeMax\) to \(0\) defines the generative model sampling \(\mathbf{\data}\sim\wm_\wmParam\). Substituting \eqref{eq:app-clean-pred} into \eqref{eq:app-ddim} recovers the denoising update in \eqref{eq:denoising}.

\para{Probability-Flow ODE and Initial-Noise Determinism}
In the continuous-time limit, the deterministic updates in \eqref{eq:app-ddim} correspond to an Euler discretization of the \emph{probability-flow ODE} associated with the forward path in \eqref{eq:app-forward}~\cite{song2021scorebased,karras2022elucidating}. This ODE transports samples from the Gaussian prior \(\mathcal{N}(0,\mathbf{I})\) to the data distribution \(\dataDistribution\) along deterministic trajectories. Therefore, the generated sample can be written as
\(
    \bm{\noise}
    \mapsto
    \wm_{\wmParam}(\bm{\noise})
    =
    \mathbf{\data}^{0},
\)
where \(\wm_{\wmParam}\) denotes the deterministic sampler induced by the trained denoiser. This makes the generated sample a deterministic and differentiable function of the initial noise, justifying the use of \(\bm{\noise}\) as a control variable and allowing gradients of downstream criteria to be backpropagated through the sampler.
\subsection{Action-conditioned Video World Models}

Video world models for robotics have shown promise for policy evaluation and improvement by simulating hard-to-model systems directly from high-dimensional sensor observations~\cite{mei2026video}. In robotic manipulation, for example, a world model can anticipate how objects may spill, slip, or topple during physical interaction from camera observations~\cite{nakamura2025generalizing,seo2025uncertainty,wiedemer2025video}. In autonomous driving, it can simulate future scene evolution in complex environments, capturing agent interactions and safety-critical situations directly in image space~\cite{gao2024vista,hassan2025gem}. These capabilities can support the evaluation and improvement of visuomotor control policies in simulation, reducing reliance on costly, time-consuming, and potentially unsafe real-world trials~\cite{mei2026video, yin2026playworld}.

\para{Stable Video Diffusion}
Recent video world models often build on video generation models~\cite{agarwal2025cosmos,wan2025wan}, which have demonstrated strong ability to synthesize realistic and temporally coherent videos. Rather than learning from scratch, many recent approaches initialize from large-scale video diffusion models and then fine-tune them on domain-specific datasets such as driving or robotic manipulation. A widely used backbone is Stable Video Diffusion (SVD)~\cite{blattmann2023stable}, a latent diffusion model for image-to-video generation. SVD uses an encoder--decoder architecture to map RGB frames into a learned latent space,
\(
    \latent = \encoder(\obs) \in \mathbb{R}^{4 \times H/8 \times W/8},
\)
where \(H\) and \(W\) denote the image height and width. The diffusion model is then trained to denoise latent video representations, following an objective of the form in \eqref{eq:app-train}, and the generated latents are decoded back into RGB video frames using the decoder.

\para{Action Conditioning}
For robotic applications, video world models are conditioned on actions to enable controllable future generation~\cite{gao2024vista,guo2025ctrl}. These actions may correspond to future ego-vehicle controls or trajectories in autonomous driving~\cite{gao2024vista,hassan2025gem}, or to robot action sequences such as end-effector poses, joint-position commands, or low-level control commands in robotic manipulation~\cite{guo2025ctrl,kim2026cosmos}. In the notation of \eqref{eq:app-train}, the denoiser is conditioned on histories and actions,
\(
    \denoise^{\diffusionTime}_{\wmParam}
    \left(
        \mathbf{\data}^{\diffusionTime},
        \historyObs,
        \actionSequence
    \right),
\)
so that the generated future is constrained to be consistent with both the observed history and the specified action sequence. While earlier world models often learned latent dynamics from task-specific observation--action trajectories~\cite{Hafner2020Dream,hafner2023dreamerv3,hansen2024tdmpc2}, recent diffusion-based video world models built on SVD typically start from a pretrained video diffusion model and post-train it with action-conditioned data. When trained with conditional dropout, these models can also use classifier-free guidance at sampling time to trade off visual realism and action controllability~\cite{ho2022classifier}.

\para{History Conditioning}
To generate coherent future frames that remain consistent with the conditioning observations, recent video world models incorporate observation history in several ways. Some methods inject conditioning frames by replacing the corresponding noisy latent frames with clean observed latents during sampling~\cite{gao2024vista}; others use progressive or autoregressive sampling with per-frame noise schedules to encourage causal structure across time~\cite{chen2024diffusion,hassan2025gem}; and some provide sparse history frames or memory tokens as additional network inputs through attention-based conditioning~\cite{guo2025ctrl}. These mechanisms help preserve temporal continuity, keep generated futures consistent with the observed history, supporting longer-horizon, autoregressive video generation by conditioning future predictions on previous observations or generated frames.

\para{Application: Policy Evaluation \& Improvement}
Action-conditioned video world models can support robot learning without repeated real-world rollouts, enabling scalable data generation, planning, policy evaluation, and policy improvement. Because current state-of-the-art video generation models are often slow and memory-intensive, they are typically used offline rather than for real-time planning. Prior work has primarily leveraged action-conditioned video prediction in two ways: (\romannumeral 1) using imagined rollouts to evaluate or optimize policies without repeated hardware trials~\cite{quevedo2025worldgym,guo2025ctrl,tseng2025scalable,li2025worldeval,yin2026playworld}, and (\romannumeral 2) generating synthetic trajectories for data augmentation or policy fine-tuning~\cite{guo2025ctrl,jang2025dreamgen,sharma2026world,guo2026vlaw,yin2026playworld}. In particular, recent works report that video world models can serve as simulators for closed-loop evaluation of visuomotor policies, with predicted performance correlating with real-world success rates~\cite{guo2025ctrl}. Similarly, policies can be improved using video world models, either through reinforcement learning inside imagined rollouts~\cite{sharma2026world} or by searching for successful synthetic trajectories that are then used to fine-tune the policy~\cite{guo2026vlaw}.

While our method can be applied to these broader settings, our experiments focus specifically on the value of steering video world models for robust policy evaluation and improvement. Rather than evaluating a complete robotic policy in a closed loop, we evaluate candidate trajectories that could be produced by a policy and assess their potential failure risks. Similarly, rather than improving policies through closed-loop exploration inside the world model, we use the steered world model to identify failure-prone expert trajectories and reweight them during policy fine-tuning.

\subsection{Inference-time Steering of Generative Models}

Given a generative model \(\wm_\wmParam\) that maps an initial noise sample \(\bm{\noise}\sim\mathcal{N}(0,\mathbf{I})\) to a generated sample \(\bm{\data}=\wm_\wmParam(\bm{\noise})\), inference-time optimization aims to find generations that maximize a user-specified criterion \(\mathcal{C}\), while preserving the plausibility of the generated samples with a regularizer \(\mathrm{Reg}(\bm{\data})\) that penalizes generations that deviate from the original model distribution. This can be written as
\begin{equation}
    \bm{\data}^*
    =
    \argmax_{\bm{\data}}
    \;
    \mathcal{C}(\bm{\data})
    -
    \alpha\,\mathrm{Reg}(\bm{\data}).
\end{equation}

A natural way to realize this objective is to adapt the generative model itself so that it samples from a reward-tilted distribution,
\begin{equation}
    \wm_{\wmParam}^{*}(\bm{\data})
    \propto
    \wm_{\wmParam}(\bm{\data})
    \exp\!\left(\frac{\mathcal{C}(\bm{\data})}{\alpha}\right),
\end{equation}
which upweights high-criterion samples while remaining close to the original model distribution. However, directly fine-tuning the generator is often impractical for inference-time control: the criterion may change across tasks or users, and fine-tuning requires significant computation for each new criterion. Moreover, preserving closeness to the original distribution can require density estimates that are intractable or prohibitively expensive for high-dimensional generative models~\cite{wallace2024diffusion, ai2026joint}. As a result, model adaptation can drift away from the original data distribution and produce implausible generations~\cite{eyring2025noise, holderrieth2026diamond}.

\para{Initial Noise Optimization} Initial noise optimization instead keeps the generative model fixed and optimizes only the input noise according to an inference-time criterion evaluated on the generated sample:
\begin{equation}
    \bm{\noise}^{*}
    =
    \argmax_{\bm{\noise}}
    \;
    \mathcal{C}\!\left(
        \wm_\wmParam(\bm{\noise})
    \right).
\end{equation}
This formulation allows arbitrary criteria to be specified at inference time without modifying the model parameters. If the optimized noise remains plausible under the original Gaussian prior, the resulting generation is still induced by the pretrained model distribution, so additional regularization on the generated sample itself is not required.

When \(\mathcal{C}\) is differentiable, the optimization can be solved by backpropagating through the diffusion sampler~\cite{wallace2023end, karunratanakul2024optimizing, eyring2024reno}. For non-differentiable criteria, prior work estimates gradients of the criterion~\cite{tang2024inference}, amortizes the optimization with a hypernetwork~\cite{eyring2025noise}, or formulates steering as a reinforcement learning problem~\cite{chen2024find, wagenmaker2025steering}. Since backpropagating through many iterative denoising steps is computationally expensive, several methods further distill diffusion models into one-step generators to make optimization more tractable~\cite{eyring2024reno,Novack2024Ditto2,eyring2025noise}.

\subsection{Typical Set of High-Dimensional Gaussian Noise}
\label{app:typical-set}

The typical set describes where samples from a distribution are expected to lie with high probability. 
It is therefore a notion of probability \textit{mass}, rather than simply probability \textit{density}~\cite{cover1999elements,nalisnick2019detecting}. 
For a random vector \(\bm{\noise}\sim\mathcal{N}(\mathbf{0},\mathbf{I}_{\dimension})\), let
\(p(\bm{\noise})\) denote the joint density of the full \(\dimension\)-dimensional Gaussian prior. 
The \(\delta\)-typical set can be written as
\begin{equation}
\mathcal{T}_{\delta}^{(\dimension)}
=
\left\{
\bm{\noise}:
\left|
-\frac{1}{\dimension}\log p(\bm{\noise})
-
h(\mathcal{N}(0,1))
\right|
\le \delta
\right\},
\end{equation}
where \(h(\mathcal{N}(0,1))=\frac{1}{2}\log(2\pi e)\) is the differential entropy of a one-dimensional standard Gaussian. 
The one-dimensional entropy appears because the standard Gaussian prior factorizes across coordinates,
\(p(\bm{\noise})=\prod_{i=1}^{\dimension}p(\noise_i)\). 
Thus, \(-\frac{1}{\dimension}\log p(\bm{\noise})\) is the average information content per coordinate, analogous to the empirical entropy of an i.i.d. sequence. 
For the isotropic Gaussian prior,
\begin{equation}
-\frac{1}{\dimension}\log p(\bm{\noise})
=
\frac{1}{2}\log(2\pi)
+
\frac{1}{2\dimension}\|\bm{\noise}\|_2^2.
\end{equation}
Therefore, typicality implies
\(\|\bm{\noise}\|_2^2/\dimension \approx 1\). 
In high dimensions, Gaussian samples concentrate on a thin shell of radius approximately \(\sqrt{\dimension}\), rather than near the mode. 
This illustrates the key distinction between density and typicality: the zero vector has the highest density under the Gaussian prior, but is extremely unlikely to be sampled in high dimensions. Although the Gaussian distribution has full support over \(\mathbb{R}^{\dimension}\), most of its probability mass concentrates on a small subset of that space as \(\dimension\) grows. 
Equivalently, a typical sample is not characterized only by its likelihood, but by a collection of sample statistics that take values close to their expectations under the Gaussian prior.

\section{Inference-time Steering of Video World Models}

\subsection{Regularizing Noise within the Typical Set of High-Dimensional Gaussian Distribution}\label{app:regularizer}

In this section, we elaborate on Sec.~\ref{sec:high-dim-gaussian}, where we impose typical-set constraints during noise optimization. In high-dimensional Gaussian distributions, regions of high probability \textit{density} do not necessarily coincide with regions of high probability \textit{mass}, where typical samples concentrate. For example, the zero vector has the highest density under a standard Gaussian, yet it is almost never sampled in high dimensions; instead, typical samples lie on a thin shell with squared norm approximately \(d\). We refer to this high-mass region as the \textit{typical set}. As illustrated in Fig.~\ref{fig:appendix-high-dim}, the typical set is distinct from the high-density region, and sampling noise outside the typical set can lead to implausible generations. This distinction is important for diffusion models, which are trained to denoise samples initialized from typical Gaussian noise. If optimization pushes the initial noise toward atypical, low-mass regions of the prior, the denoiser may be evaluated on inputs rarely seen during training, leading to approximation errors and, consequently, implausible or OOD generations.

\begin{figure}[ht]
    \centering
    \vspace{-0.1in}
    \includegraphics[width=1.0\linewidth]{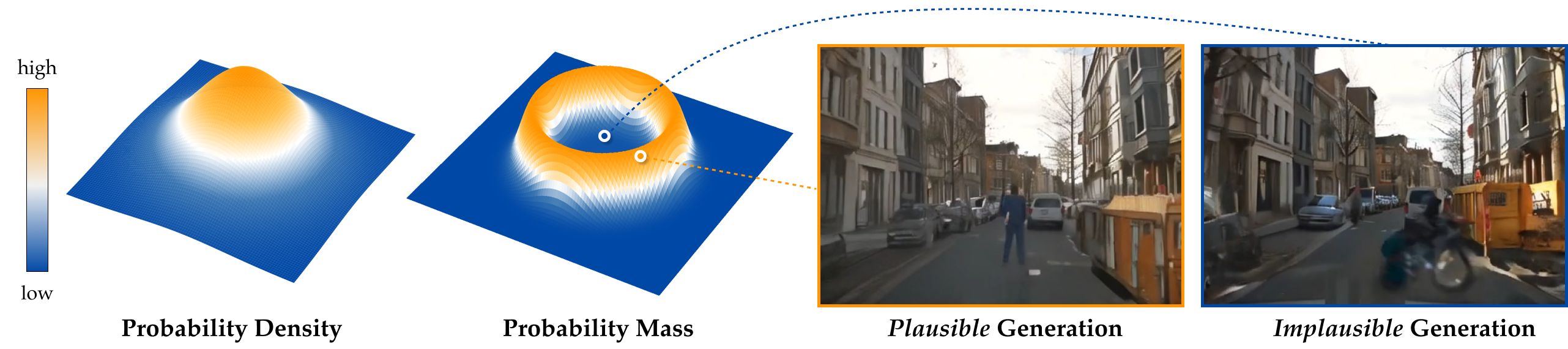}
    \vspace{-0.25in}
    \caption{\small{\textit{Probability density vs.\ mass in a high-dimensional Gaussian distribution.} Noise sampled from the typical set leads to plausible generations, whereas noise outside the typical set--despite having high probability density--produces implausible generations, e.g., humans becoming blurry or transforming into vehicles.}}
    \label{fig:appendix-high-dim}
    \vspace{-0.15in}
\end{figure}

A common way to encourage typicality is to regularize the norm of the optimized noise, since Gaussian samples exhibit strong norm concentration in high dimensions. Several prior methods have found norm-based regularization sufficient for preserving plausibility in diffusion generation~\cite{samuel2023norm, eyring2024reno, tang2024inference}. However, the norm captures only one of the numerous sample statistics of a Gaussian sample. In practice, we find that additional typicality constraints are important for high-dimensional video world models. Thus, as described in Sec.~\ref{sec:high-dim-gaussian}, we regularize multiple statistics of the optimized noise, including norm concentration, isotropy, and spectral whiteness, each motivated by concentration properties of high-dimensional Gaussian samples.

\begin{figure}[ht]
    \centering
    \vspace{-0.15in}
    \includegraphics[width=1.0\linewidth]{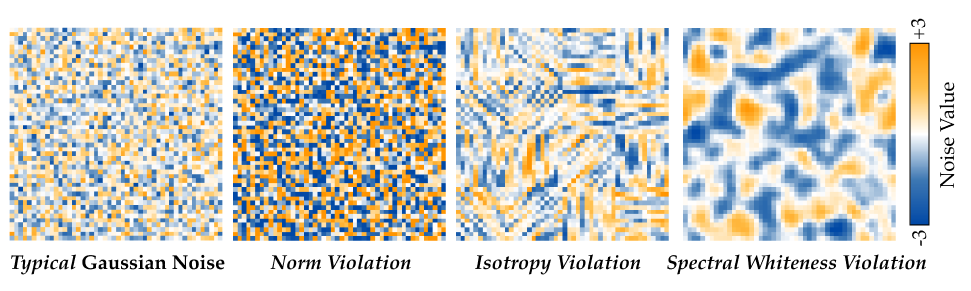}
    \vspace{-0.3in}
    \caption{\small{\textit{Examples of atypical Gaussian noise.} \textit{Left}: typical noise sampled from a Gaussian prior. \textit{Others}: noises perturbed to violate norm concentration, isotropy, and spectral whiteness, respectively.}}
    \label{fig:appendix-noise-violated} \vspace{-0.1in}
\end{figure}

Fig.~\ref{fig:appendix-noise-violated} illustrates examples of Gaussian noise visualized as 2D images, where each atypical sample violates one statistic of the Gaussian typical set. Violating norm concentration produces noise with globally biased values compared to typical Gaussian noise. Violating isotropy introduces local coordinate patterns that do not appear in typical i.i.d. noise, while violating spectral whiteness produces structured, non-white patterns in the frequency domain.

\subsection{Concentration of Gaussian Typical Set Regularizers}

\para{Norm Concentration}
For Gaussian noise \(\bm{\noise}\sim\mathcal{N}(\mathbf{0},\mathbf{I}_{\dimension})\), the squared norm satisfies
\(\|\bm{\noise}\|_2^2\sim\chi^2_{\dimension}\), and therefore concentrates sharply around \(\dimension\). Equivalently, the Euclidean norm, or radius \(\|\bm{\noise}\|_2\) concentrates around \(\sqrt{\dimension}\). Moreover, if \(R=\|\bm{\noise}\|_2\), then \(R\) follows a chi distribution with \(\dimension\) degrees of freedom, whose standard deviation remains bounded independently of \(\dimension\) and approaches \(1/\sqrt{2}\) as \(\dimension\to\infty\). Thus, although the typical radius grows as \(\sqrt{\dimension}\), the absolute width of the shell containing most Gaussian probability mass remains approximately constant. This motivates penalizing deviations from the typical Gaussian shell in radius:
\[
    \reward^{\mathrm{norm}}(\bm{\noise})
    :=
    -\left(
        \|\bm{\noise}\|_2
        -
        \sqrt{\dimension}
    \right)^2.
\]

\begin{lemma}[Norm concentration of Gaussian noise]
\label{lem:gaussian-norm-concentration}
Let \(\bm{\noise}\sim\mathcal{N}(\mathbf{0},\mathbf{I}_{\dimension})\). Then, for any \(\delta\in(0,1)\),
\[
    \mathbb{P}\!\left(
        \left|
            \frac{\|\bm{\noise}\|_2^2}{\dimension}
            -
            1
        \right|
        >
        2\sqrt{\frac{\log(2/\delta)}{\dimension}}
        +
        2\frac{\log(2/\delta)}{\dimension}
    \right)
    \le
    \delta .
\]
Equivalently, the probability that \(\bm{\noise}\) lies outside this typical norm shell is at most \(\delta\).
\end{lemma}

\begin{proof}
Since \(\bm{\noise}\sim\mathcal{N}(\mathbf{0},\mathbf{I}_{\dimension})\), its coordinates
\(\noise_1,\ldots,\noise_{\dimension}\) are independent standard normal random variables. Hence,
\[
    \|\bm{\noise}\|_2^2
    =
    \sum_{j=1}^{\dimension}\noise_j^2
    \sim
    \chi^2_{\dimension}.
\]
Let \(X=\|\bm{\noise}\|_2^2\). By the Laurent--Massart inequality~\cite{laurent2000adaptive}, for any \(x>0\),
\[
    \mathbb{P}\!\left(
        X-\dimension \ge 2\sqrt{\dimension x}+2x
    \right)
    \le e^{-x},
    \qquad
    \mathbb{P}\!\left(
        \dimension-X \ge 2\sqrt{\dimension x}
    \right)
    \le e^{-x}.
\]
Taking \(x=\log(2/\delta)\), each tail event has probability at most
\(\delta/2\). Therefore, by a union bound, with probability at least
\(1-\delta\),
\[
    -2\sqrt{\dimension\log(2/\delta)}
    \le
    X-\dimension
    \le
    2\sqrt{\dimension\log(2/\delta)}
    +
    2\log(2/\delta).
\]
Dividing by \(\dimension\) gives
\[
    \left|
        \frac{X}{\dimension}
        -
        1
    \right|
    \le
    2\sqrt{\frac{\log(2/\delta)}{\dimension}}
    +
    2\frac{\log(2/\delta)}{\dimension}.
\]
Substituting \(X=\|\bm{\noise}\|_2^2\) completes the proof.
\end{proof}

\begin{corollary}[High-probability scale of the radial norm regularizer]
\label{cor:norm-regularizer-scale}
For
\[
    |\,\reward^{\mathrm{norm}}(\bm{\noise})\,|
    =
    \left(
        \|\bm{\noise}\|_2
        -
        \sqrt{\dimension}
    \right)^2,
\]
we have, for any \(\delta\in(0,1)\), with probability at least \(1-\delta\),
\[
    |\,\reward^{\mathrm{norm}}(\bm{\noise})\,|
    \le
    \left(
        2\sqrt{\log(2/\delta)}
        +
        2\frac{\log(2/\delta)}{\sqrt{\dimension}}
    \right)^2.
\]
Consequently,
\[
    |\,\reward^{\mathrm{norm}}(\bm{\noise})\,|
    =
    O\!\left(
        \log(1/\delta)
        +
        \frac{\log^2(1/\delta)}{\dimension}
    \right)
\]
with high probability.
\end{corollary}

\begin{proof}
By Lemma~\ref{lem:gaussian-norm-concentration}, with probability at least \(1-\delta\),
\[
    \left|
        \|\bm{\noise}\|_2^2
        -
        \dimension
    \right|
    \le
    2\sqrt{\dimension\log(2/\delta)}
    +
    2\log(2/\delta).
\]
Let \(R=\|\bm{\noise}\|_2\). Since
\[
    |R-\sqrt{\dimension}|
    =
    \frac{|R^2-\dimension|}{R+\sqrt{\dimension}}
    \le
    \frac{|R^2-\dimension|}{\sqrt{\dimension}},
\]
we obtain
\[
    |R-\sqrt{\dimension}|
    \le
    2\sqrt{\log(2/\delta)}
    +
    2\frac{\log(2/\delta)}{\sqrt{\dimension}}.
\]
Squaring both sides gives the desired bound.
\end{proof}

\begin{corollary}[Dimension-independent width of the Gaussian shell]
\label{cor:dimension-independent-radius-width}
Let \(R=\|\bm{\noise}\|_2\) with
\(\bm{\noise}\sim\mathcal{N}(\mathbf{0},\mathbf{I}_{\dimension})\). Then the standard deviation of \(R\) is bounded independently of \(\dimension\):
\[
    \operatorname{std}(R)
    =
    O(1).
\]
More specifically, \(R\) follows a chi distribution with \(\dimension\) degrees of freedom, and
\[
    \operatorname{std}(R)
    \to
    \frac{1}{\sqrt{2}}
    \qquad
    \text{as} \qquad
    \dimension\to\infty .
\]
\end{corollary}

\begin{proof}
Since \(R=\|\bm{\noise}\|_2\) and
\(\|\bm{\noise}\|_2^2\sim\chi^2_{\dimension}\), the random variable \(R\) follows a chi distribution with \(\dimension\) degrees of freedom. The chi distribution has moments
\[
    \mathbb{E}[R^2]=\dimension,
    \qquad
    \mathbb{E}[R]
    =
    \sqrt{2}\,
    \frac{\Gamma((\dimension+1)/2)}
         {\Gamma(\dimension/2)}.
\]
Therefore,
\[
    \operatorname{Var}(R)
    =
    \dimension
    -
    2
    \left(
        \frac{\Gamma((\dimension+1)/2)}
             {\Gamma(\dimension/2)}
    \right)^2.
\]
This variance remains bounded as \(\dimension\) grows, and in fact converges to \(1/2\). Hence
\[
    \operatorname{std}(R)\to \frac{1}{\sqrt{2}},
\]
showing that the width of the Gaussian shell is asymptotically independent of the dimension.
\end{proof}

\para{Isotropy Concentration}
We now show that the isotropy regularizer is small with high probability under the Gaussian prior. Let
\(\bm{\noise}\sim\mathcal{N}(\mathbf{0},\mathbf{I}_{\dimension})\), and suppose
\(\dimension=mk\). We randomly permute the coordinates of \(\bm{\noise}\) and partition them into
\(m\) blocks
\[
    \bm{\noise}_1,\ldots,\bm{\noise}_m \in \mathbb{R}^k .
\]
Since the Gaussian prior has i.i.d. coordinates, this random permutation does not change the distribution, and the blocks satisfy
\[
    \bm{\noise}_i
    \overset{\mathrm{i.i.d.}}{\sim}
    \mathcal{N}(\mathbf{0},\mathbf{I}_k).
\]
Thus, their empirical second moment
\[
    \widehat{\mathbf{\Sigma}}
    :=
    \frac{1}{m}
    \sum_{i=1}^{m}
    \bm{\noise}_i\bm{\noise}_i^\top
\]
should be close to \(\mathbf{I}_k\) under the Gaussian prior. We measure deviations from this blockwise isotropy using
\[
    \reward^{\mathrm{iso}}(\bm{\noise})
    :=
    -\frac{1}{k}
    \left\|
        \widehat{\mathbf{\Sigma}}
        -
        \mathbf{I}_k
    \right\|_F^2 .
\]

\begin{lemma}[Small probability of anisotropic Gaussian block covariance]
\label{lem:isotropy-concentration}
There exists a universal constant \(C>0\) such that, for any
\(\delta\in(0,1)\), if \(m \gtrsim \log(2k^2/\delta)\), then
\[
    \mathbb{P}\!\left(
        |\,\reward^{\mathrm{iso}}(\bm{\noise}) \,|
        >
        C
        \frac{k\log(2k^2/\delta)}{m}
    \right)
    \le
    \delta .
\]
\end{lemma}

\begin{proof}
For each entry of \(\widehat{\mathbf{\Sigma}}\),
\[
    \widehat{\Sigma}_{ab}
    =
    \frac{1}{m}
    \sum_{i=1}^{m}
    \noise_{i,a}\noise_{i,b}.
\]
First, consider the diagonal terms. Since
\(\noise_{i,a}\sim\mathcal{N}(0,1)\),
\[
    \widehat{\Sigma}_{aa}
    =
    \frac{1}{m}
    \sum_{i=1}^{m}
    \noise_{i,a}^2,  \quad  \text{therefore,} \quad \mathbb{E}[\widehat{\Sigma}_{aa}]
    =
    1,
    \quad
    \operatorname{Var}(\widehat{\Sigma}_{aa})
    =
    \frac{2}{m}, \quad \Rightarrow\, 
    \mathbb{E}
    \left[
        (\widehat{\Sigma}_{aa}-1)^2
    \right]
    =
    \frac{2}{m}.
\]

For the off-diagonal terms \(a\neq b\), the products
\(\noise_{i,a}\noise_{i,b}\) have mean zero and variance one. Hence
\[
    \mathbb{E}[\widehat{\Sigma}_{ab}]
    =
    0,
    \qquad
    \operatorname{Var}(\widehat{\Sigma}_{ab})
    =
    \frac{1}{m}, \quad \text{which gives} \quad
    \mathbb{E}
    \left[
        \widehat{\Sigma}_{ab}^2
    \right]
    =
    \frac{1}{m}.
\]
Therefore,
\[
\begin{aligned}
    \mathbb{E}
    \left[
        \left\|
            \widehat{\mathbf{\Sigma}}
            -
            \mathbf{I}_k
        \right\|_F^2
    \right]
    &=
    \sum_{a=1}^{k}
    \mathbb{E}
    \left[
        (\widehat{\Sigma}_{aa}-1)^2
    \right]
    +
    \sum_{a\neq b}
    \mathbb{E}
    \left[
        \widehat{\Sigma}_{ab}^2
    \right] \\
    &=
    k\cdot\frac{2}{m}
    +
    k(k-1)\cdot\frac{1}{m} \\
    &=
    \frac{k^2+k}{m}.
\end{aligned}
\]
Dividing by \(k\) gives
\[
    \mathbb{E}
    \left[
        |\,\reward^{\mathrm{iso}}(\bm{\noise})\,|
    \right]
    =
    \frac{k+1}{m}.
\]

For the high-probability statement, note that each diagonal error
\(\widehat{\Sigma}_{aa}-1\) is an average of centered sub-exponential random variables, and each off-diagonal entry \(\widehat{\Sigma}_{ab}\) is an average of products of independent Gaussians, which are also sub-exponential. Therefore, by Bernstein-type concentration~\cite{hsu2012tail}, for all entries,
\[
    \mathbb{P}
    \left(
        |\widehat{\Sigma}_{ab} - \delta_{ab}| > t
    \right)
    \le
    2\exp(-cmt^2)
\]
for \(0<t\lesssim 1\), where \(c>0\) is a universal constant. Applying a union bound over the \(k^2\) entries and taking
\[
    t
    =
    C'
    \sqrt{
        \frac{\log(2k^2/\delta)}{m}
    },
\]
we obtain, with probability at least \(1-\delta\),
\[
    |\widehat{\Sigma}_{ab}-\delta_{ab}|
    \le
    C'
    \sqrt{
        \frac{\log(2k^2/\delta)}{m}
    }
    \qquad
    \text{for all } a,b.
\]
Hence,
\[
    \left\|
        \widehat{\mathbf{\Sigma}}
        -
        \mathbf{I}_k
    \right\|_F^2
    \le
    k^2
    {C'}^2
    \frac{\log(2k^2/\delta)}{m}.
\]
Dividing by \(k\) yields
\[
    \reward^{\mathrm{iso}}(\bm{\noise})
    \le
    C
    \frac{k\log(2k^2/\delta)}{m}
\]
with probability at least \(1-\delta\), for a universal constant \(C>0\).
\end{proof}

This shows that the normalized isotropy penalty is small under the training-time Gaussian prior whenever \(m\gg k\). Unlike norm concentration, which only constrains the total energy
\(\operatorname{tr}(\widehat{\mathbf{\Sigma}})\), the isotropy penalty also discourages unequal coordinate variances and nonzero off-diagonal correlations. Therefore, it detects structured coordinate patterns that may still have a typical global norm but are unlikely under an i.i.d. Gaussian prior.

\para{Spectral-Whiteness Concentration}
The isotropy regularizer checks whether randomly grouped coordinates have approximately identity covariance, but it does not directly rule out spatial frequency structure. For example, optimized noise may have excessive low-frequency energy, producing smooth spatial patterns, or excessive high-frequency energy, producing oscillatory artifacts. Under an i.i.d. Gaussian prior, the noise should instead be spectrally white: its energy should be spread evenly across spatial frequencies~\cite{hwang2026noisetilt}.

To formalize this, consider one spatial slice of the noise and let
\[
    \widetilde{\bm{\noise}}
    =
    U\bm{\noise}
\]
be its discrete Fourier transform, where \(U\) is an orthonormal Fourier basis. For real-valued noise, this can be interpreted as the equivalent real sine--cosine basis. Since the standard Gaussian distribution is rotationally invariant,
\[
    \widetilde{\bm{\noise}}
    \sim
    \mathcal{N}(\mathbf{0},\mathbf{I}).
\]
Thus, no frequency band should systematically contain more energy than another. Individual Fourier powers \(|\widetilde{\noise}_j|^2\) are noisy, but averages over sufficiently large frequency bins concentrate.

Let \(\{\mathcal{B}_b\}_{b=1}^{B}\) be disjoint frequency bins, and define the average power in each bin by
\[
    \widehat{p}_b
    :=
    \frac{1}{|\mathcal{B}_b|}
    \sum_{j\in\mathcal{B}_b}
    |\widetilde{\noise}_j|^2 .
\]
We penalize deviations from a flat binned spectrum using
\[
    \reward^{\mathrm{spec}}(\bm{\noise})
    :=
    -\frac{1}{B}
    \sum_{b=1}^{B}
    \left(
        \widehat{p}_b-\bar{p}
    \right)^2,
    \qquad
    \bar{p}
    :=
    \frac{1}{B}
    \sum_{b=1}^{B}
    \widehat{p}_b .
\]

\begin{lemma}[Small probability of atypical spectral energy]
\label{lem:spectral-whiteness-concentration}
Let \(\bm{\noise}\sim\mathcal{N}(\mathbf{0},\mathbf{I})\), and let
\(\widetilde{\bm{\noise}}=U\bm{\noise}\), where \(U\) is an orthonormal discrete Fourier transform. Let \(\{\mathcal{B}_b\}_{b=1}^{B}\) be disjoint frequency bins, and define \(\reward^{\mathrm{spec}}\) as above. Let
\[
    n_{\min}:=\min_b |\mathcal{B}_b|.
\]
Then, for any \(\delta\in(0,1)\),
\[
    \mathbb{P}\!\left(
        |\,\reward^{\mathrm{spec}}(\bm{\noise})\,|
        >
        4
        \left(
            2\sqrt{
                \frac{\log(2B/\delta)}{n_{\min}}
            }
            +
            2
            \frac{\log(2B/\delta)}{n_{\min}}
        \right)^2
    \right)
    \le
    \delta .
\]
\end{lemma}

\begin{proof}
Since \(U\) is orthonormal and the standard Gaussian distribution is rotationally invariant,
\[
    \widetilde{\bm{\noise}}
    =
    U\bm{\noise}
    \sim
    \mathcal{N}(\mathbf{0},\mathbf{I}).
\]
Thus, in the equivalent real orthonormal Fourier representation, the Fourier coefficients are independent standard normal random variables.

For each frequency bin \(\mathcal{B}_b\),
\[
    |\mathcal{B}_b|\widehat{p}_b
    =
    \sum_{j\in\mathcal{B}_b}
    |\widetilde{\noise}_j|^2
    \sim
    \chi^2_{|\mathcal{B}_b|}.
\]
By the Laurent--Massart inequality~\cite{laurent2000adaptive} and a union bound over the \(B\) bins, with probability at least \(1-\delta\), for all \(b=1,\ldots,B\),
\[
    \left|
        \widehat{p}_b-1
    \right|
    \le
    2\sqrt{
        \frac{\log(2B/\delta)}{|\mathcal{B}_b|}
    }
    +
    2
    \frac{\log(2B/\delta)}{|\mathcal{B}_b|}.
\]
Since \(|\mathcal{B}_b|\ge n_{\min}\), we have
\[
    \max_b
    \left|
        \widehat{p}_b-1
    \right|
    \le
    \eta,
    \qquad
    \eta
    :=
    2\sqrt{
        \frac{\log(2B/\delta)}{n_{\min}}
    }
    +
    2
    \frac{\log(2B/\delta)}{n_{\min}}.
\]
Moreover,
\[
    |\bar{p}-1|
    =
    \left|
        \frac{1}{B}
        \sum_{b=1}^{B}
        (\widehat{p}_b-1)
    \right|
    \le
    \eta.
\]
Therefore, for every bin,
\[
    |\widehat{p}_b-\bar{p}|
    \le
    |\widehat{p}_b-1|
    +
    |\bar{p}-1|
    \le
    2\eta.
\]
Substituting this into the definition of the regularizer gives
\[
    |\,\reward^{\mathrm{spec}}(\bm{\noise})\,|
    =
    \frac{1}{B}
    \sum_{b=1}^{B}
    \left(
        \widehat{p}_b-\bar{p}
    \right)^2
    \le
    4\eta^2.
\]
Equivalently,
\[
    \mathbb{P}\!\left(
        |\,\reward^{\mathrm{spec}}(\bm{\noise})\,|
        >
        4\eta^2
    \right)
    \le
    \delta,
\]
which proves the claim.
\end{proof}

Lemma~\ref{lem:spectral-whiteness-concentration} shows that the spectral-whiteness penalty is small under the Gaussian prior when each frequency bin contains sufficiently many coefficients. Thus, a large value indicates an atypical spatial frequency structure. This regularizer is complementary to isotropy: isotropy checks covariance after random coordinate grouping, whereas spectral whiteness checks whether energy is evenly distributed across spatial frequencies.

\subsection{Vision-Language Models as General Inference-Time Criterion}\label{app:vlm-reward}

Rather than learning a task-specific reward function or failure classifier for a narrow set of events~\cite{nakamura2025generalizing,seo2025uncertainty}, we use a VLM as an inference-time verifier that can flexibly specify steering targets across diverse scenes and event categories~\cite{wu2025foresight}. This is particularly useful for general-purpose video world models, whose possible futures may vary widely with the scene and action context. \begin{wrapfigure}{l}{0.5\textwidth}
    \centering
    \vspace{-0.2in} 
    \includegraphics[width=\linewidth]{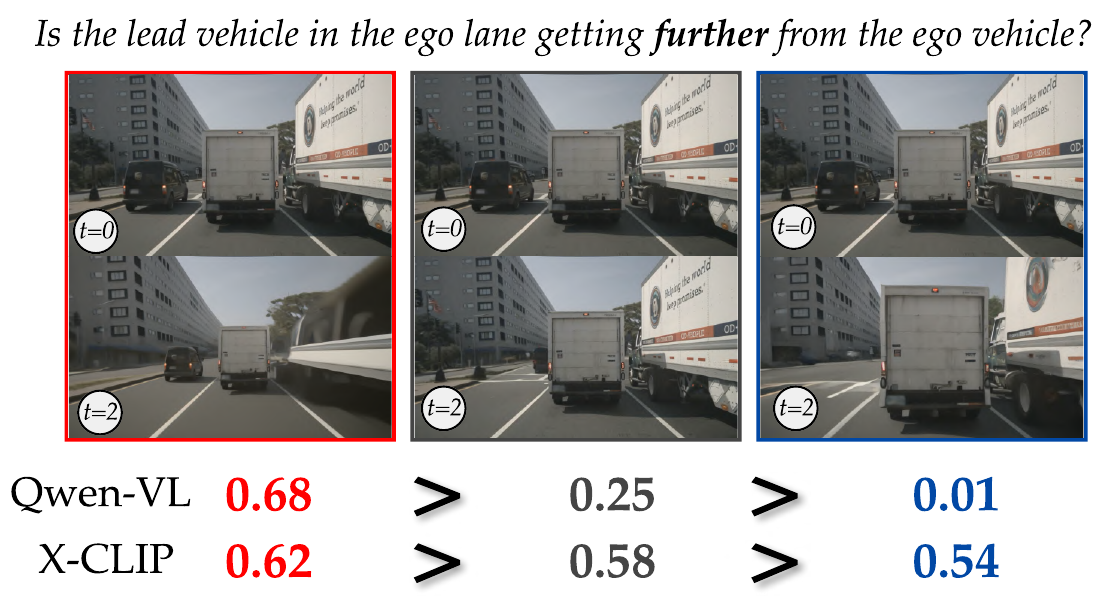} 
    \vspace{-0.25in}
    \caption{\small{\textit{VLM score for driving scenes.}
The three videos are generated from the same scene but exhibit different temporal relationships to the truck in the ego lane. Both the Qwen-VL-based model and X-CLIP correctly distinguish these temporal relationships.}}
    \vspace{-0.2in}\label{fig:vlm-driving}
\end{wrapfigure}VLMs also provide semantically meaningful gradient signals: task-specific reward models trained on limited datasets may fail to cover the large space of possible video generations, making them vulnerable to \textit{reward hacking}~\cite{clark2024directly, eyring2024reno}, where optimization increases the reward through adversarial or visually meaningless changes rather than by producing the intended event. In contrast, because VLMs are trained on large-scale vision-language data, they can provide a more general semantic signal for whether the generated video depicts the specified event~\cite{wu2025rewarddance}.

Since our method relies on dense gradient signals from VLMs, we use open-source models rather than API-based evaluators~\cite{team2023gemini}, despite the strong scene-understanding capabilities of recent proprietary models. We adopt different VLMs for autonomous driving and robotic manipulation based on the semantic structure of each domain.

\para{Autonomous Driving}
In autonomous driving, the target events we aim to steer often depend on temporal relationships among agents, such as pedestrian motion, vehicle movement, or changes in relative distance. We empirically find that off-the-shelf Qwen models~\cite{Qwen-VL} are less reliable at understanding driving scenes and their temporal structure. We therefore use the Qwen2.5-VL model fine-tuned on autonomous driving videos by \citet{li2025wolf}, which provides improved driving-scene understanding. We also use X-CLIP~\cite{ma2022x}, which we find effective at capturing temporal relationships. As shown in Fig.~\ref{fig:vlm-driving}, for three generated videos where the relative distance to the lead vehicle is increasing, unchanged, or decreasing, both the fine-tuned Qwen model and X-CLIP~\cite{ma2022x} assign scores that correctly distinguish these cases. For CLIP-style models, we define a positive prompt \(\prompt_+\) and a negative prompt \(\prompt_-\). We embed both prompts and the generated video into the shared embedding space, and use the difference in cosine similarities as the steering objective:
\[
    \reward^\text{sem}(\mathbf{\obs})
    =
    \mathrm{sim}\!\left(
        e^{\text{video}}(\mathbf{\obs}), e^{\text{text}}(\prompt_+)
    \right)
    -
    \mathrm{sim}\!\left(
        e^{\text{video}}(\mathbf{\obs}), e^{\text{text}}(\prompt_-)
    \right).
\]

\begin{figure}[ht]
    \centering
    \includegraphics[width=1.0\linewidth]{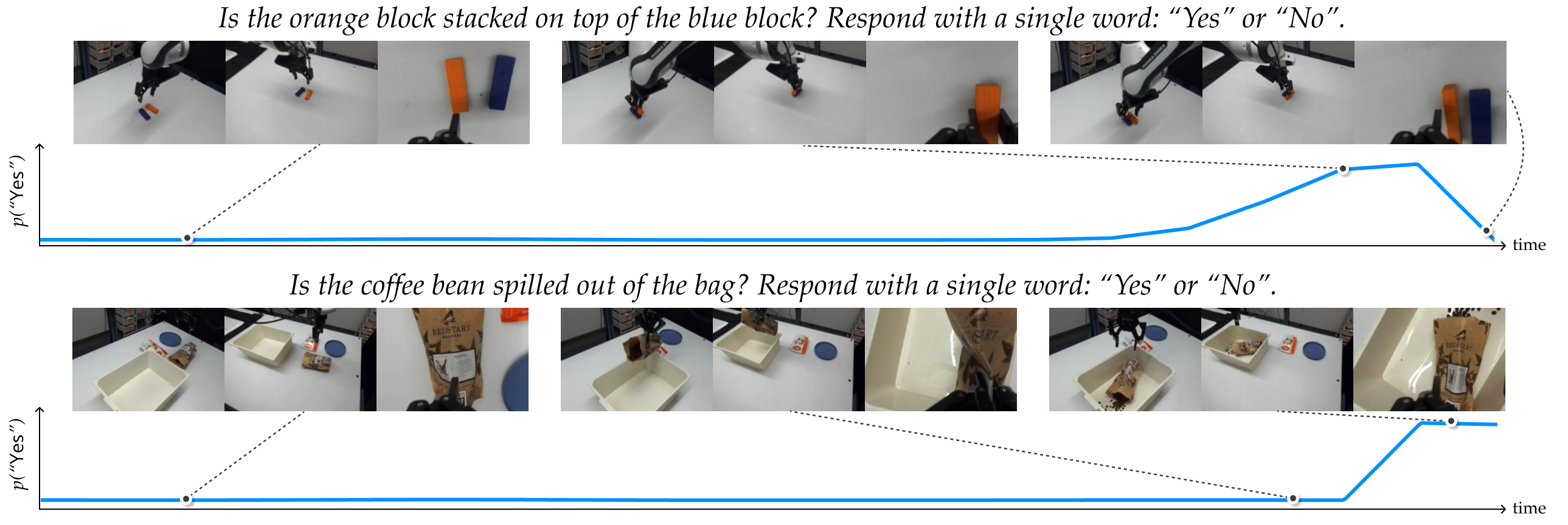}
    \caption{\small{\textit{Robotic Manipulation: VLM scores for video understanding.}The plot shows the token probability of answering ``Yes'' to a given prompt. The VLM takes videos from three camera views simultaneously and assigns high scores when the specified events occur, correctly detecting task-relevant moments.}}
    \label{fig:vlm-manipulation}\vspace{-0.2in}
\end{figure} 

\para{Robotic Manipulation}
In tabletop manipulation, our steering objective focuses on detecting task-failure events from generated videos. We use Qwen3-VL~\cite{Qwen3-VL} with a text prompt that defines the failure condition, and compute the score from the token probability of the target response, similar to recent generalizable reward-modeling approaches~\cite{chen2026topreward}. Fig.~\ref{fig:vlm-manipulation} shows that the model score identifies task-relevant events at the correct moments in generated videos. We provide the five-frame videos generated by the world model as input and backpropagate dense gradient signals through the video. We find that using all three camera views simultaneously is essential for reliable scoring: for the same scene and text prompt, using a single viewpoint leads to substantially poorer failure-event detection. We also find that CLIP-style models, including X-CLIP~\cite{ma2022x}, are less effective at understanding these manipulation scenes.

\subsection{Score-Distillation for Approximating Initial-Noise Gradients}
\label{app:score-distillation}

This section justifies the gradient approximation used in \eqref{eq:gradient-ascent-final}. We omit the conditioning variables \((\historyObs,\actionSequence)\) for readability. Let the deterministic sampler define the reverse trajectory
\[
    \mathbf{\data}^{\diffusionTime-1}
    =
    \denoising^{\diffusionTime}(\mathbf{\data}^{\diffusionTime}),
    \qquad
    \mathbf{\data}^{\diffusionTimeMax}=\bm{\noise},
    \qquad
    \mathbf{\data}^{0}=\wm_{\wmParam}(\bm{\noise}),
\]
where each denoising step follows the update in \eqref{eq:denoising}:
\[
    \denoising^{\diffusionTime}(\mathbf{\data}^{\diffusionTime})
    =
    \frac{\diffusionCoeff^{\diffusionTime - 1}}{\diffusionCoeff^{\diffusionTime}}
    \left(
        \mathbf{\data}^{\diffusionTime}
        -
        \noiseCoeff^{\diffusionTime}
        \denoise_{\wmParam}^{\diffusionTime}(\mathbf{\data}^{\diffusionTime})
    \right)
    +
    \noiseCoeff^{\diffusionTime - 1}
    \denoise_{\wmParam}^{\diffusionTime}(\mathbf{\data}^{\diffusionTime}).
\]
For a differentiable criterion \(\reward(\mathbf{\data}^{0})\), the exact initial-noise gradient is obtained by the chain rule:
\[
    \nabla_{\bm{\noise}}
    \reward(\mathbf{\data}^{0})
    =
    \left[
        \prod_{\diffusionTime=\diffusionTimeMax}^{1}
        \frac{
            \partial \denoising^{\diffusionTime}
        }{
            \partial \mathbf{\data}^{\diffusionTime}
        }
    \right]^{\!\top}
    \nabla_{\mathbf{\data}^{0}}
    \reward(\mathbf{\data}^{0}),
\]
where the product is ordered along the reverse denoising trajectory. The Jacobian of each step is
\[
    \frac{
        \partial \denoising^{\diffusionTime}
    }{
        \partial \mathbf{\data}^{\diffusionTime}
    }
    =
    \frac{\diffusionCoeff^{\diffusionTime - 1}}{\diffusionCoeff^{\diffusionTime}}
    \mathbf{I}
    +
    \left(
        \noiseCoeff^{\diffusionTime - 1}
        -
        \frac{\diffusionCoeff^{\diffusionTime - 1}}{\diffusionCoeff^{\diffusionTime}}
        \noiseCoeff^{\diffusionTime}
    \right)
    \mathbf{J}^{\diffusionTime}_{\denoise}, \, \text{where} \,
    \mathbf{J}^{\diffusionTime}_{\denoise}
    =
    \frac{
        \partial
        \denoise_{\wmParam}^{\diffusionTime}
        (\mathbf{\data}^{\diffusionTime})
    }{
        \partial \mathbf{\data}^{\diffusionTime}
    }.
\]
Thus, exact backpropagation requires multiplying by all denoiser Jacobians and storing the full denoising computation graph.

\para{Proportional-Jacobian approximation} Ahn et al.~\cite{ahn2024noise} empirically observe that the denoiser Jacobian is approximately diagonal and often behaves like a timestep-dependent multiple of the identity. Under the idealized approximation
\[
    \mathbf{J}^{\diffusionTime}_{\denoise}
    \approx
    \rho^{\diffusionTime}\mathbf{I},
\]
the sampler Jacobian becomes
\[
    \prod_{\diffusionTime=\diffusionTimeMax}^{1}
    \left[
        \frac{\diffusionCoeff^{\diffusionTime - 1}}{\diffusionCoeff^{\diffusionTime}}
        \mathbf{I}
        +
        \left(
            \noiseCoeff^{\diffusionTime - 1}
            -
            \frac{\diffusionCoeff^{\diffusionTime - 1}}{\diffusionCoeff^{\diffusionTime}}
            \noiseCoeff^{\diffusionTime}
        \right)
        \mathbf{J}^{\diffusionTime}_{\denoise}
    \right]
    \approx
    \beta\mathbf{I}, \, \text{where} \, 
    \beta
    :=
    \prod_{\diffusionTime=\diffusionTimeMax}^{1}
    \left[
        \frac{\diffusionCoeff^{\diffusionTime - 1}}{\diffusionCoeff^{\diffusionTime}}
        +
        \left(
            \noiseCoeff^{\diffusionTime - 1}
            -
            \frac{\diffusionCoeff^{\diffusionTime - 1}}{\diffusionCoeff^{\diffusionTime}}
            \noiseCoeff^{\diffusionTime}
        \right)
        \rho^{\diffusionTime}
    \right].
\]
Substituting this into the exact chain rule yields
\[
    \nabla_{\bm{\noise}}
    \reward(\wm_{\wmParam}(\bm{\noise}))
    \approx
    \beta\,
    \nabla_{\mathbf{\data}^{0}}
    \reward(\mathbf{\data}^{0}),
    \qquad
    \mathbf{\data}^{0}
    =
    \wm_{\wmParam}(\bm{\noise}).
\]
This is the approximation used in the main text: instead of differentiating through the sampler, we compute the criterion gradient with respect to the generated clean latent and use it as a scaled proxy for the initial-noise gradient.

\para{Diagonal-Jacobian variant} If the denoiser Jacobian is diagonal but not exactly proportional to the identity,
\[
    \mathbf{J}^{\diffusionTime}_{\denoise}
    \approx
    \operatorname{diag}(\bm{\rho}^{\diffusionTime}),
\]
then the same derivation gives a diagonal preconditioner
\[
    \nabla_{\bm{\noise}}
    \reward(\mathbf{\data}^{0})
    \approx
    \mathbf{B}\,
    \nabla_{\mathbf{\data}^{0}}
    \reward(\mathbf{\data}^{0}),
\]
where
\[
    \mathbf{B}
    =
    \prod_{\diffusionTime=\diffusionTimeMax}^{1}
    \operatorname{diag}\!\left[
        \frac{\diffusionCoeff^{\diffusionTime - 1}}{\diffusionCoeff^{\diffusionTime}}
        \mathbf{1}
        +
        \left(
            \noiseCoeff^{\diffusionTime - 1}
            -
            \frac{\diffusionCoeff^{\diffusionTime - 1}}{\diffusionCoeff^{\diffusionTime}}
            \noiseCoeff^{\diffusionTime}
        \right)
        \bm{\rho}^{\diffusionTime}
    \right].
\]
The scalar approximation \(\mathbf{B}\approx\beta\mathbf{I}\) is therefore a practical simplification: it assumes the diagonal entries vary slowly enough that their average effect can be absorbed into the optimization step size.

\para{Connection to stop-gradient MSD} Multistep score distillation~\cite{ahn2024noise} keeps the forward denoising trajectory unchanged but applies a stop-gradient operator to the denoiser outputs during backpropagation. In our notation, the stop-gradient version of one denoising step is
\[
    \denoising^{\diffusionTime}_{\mathrm{SG}}(\mathbf{\data}^{\diffusionTime})
    =
    \frac{\diffusionCoeff^{\diffusionTime - 1}}{\diffusionCoeff^{\diffusionTime}}
    \left(
        \mathbf{\data}^{\diffusionTime}
        -
        \noiseCoeff^{\diffusionTime}
        \operatorname{SG}\!\left[
            \denoise_{\wmParam}^{\diffusionTime}(\mathbf{\data}^{\diffusionTime})
        \right]
    \right)
    +
    \noiseCoeff^{\diffusionTime - 1}
    \operatorname{SG}\!\left[
        \denoise_{\wmParam}^{\diffusionTime}(\mathbf{\data}^{\diffusionTime})
    \right],
\]
where \(\operatorname{SG}[\cdot]\) is the identity in the forward pass and has zero derivative in the backward pass. Therefore, this removes \(\mathbf{J}^{\diffusionTime}_{\denoise}\) from the backward pass, leaving only the explicit linear dependence on \(\mathbf{\data}^{\diffusionTime}\):
\[
    \frac{
        \partial \denoising^{\diffusionTime}_{\mathrm{SG}}
    }{
        \partial \mathbf{\data}^{\diffusionTime}
    }
    =
    \frac{\diffusionCoeff^{\diffusionTime - 1}}{\diffusionCoeff^{\diffusionTime}}
    \mathbf{I}.
\]
Under the same proportional-Jacobian approximation, both the full-gradient and stop-gradient Jacobian products are scalar multiples of the identity:
\[
    \prod_{\diffusionTime=\diffusionTimeMax}^{1}
    \frac{
        \partial \denoising^{\diffusionTime}
    }{
        \partial \mathbf{\data}^{\diffusionTime}
    }
    \approx
    \beta\mathbf{I},
    \qquad
    \prod_{\diffusionTime=\diffusionTimeMax}^{1}
    \frac{
        \partial \denoising^{\diffusionTime}_{\mathrm{SG}}
    }{
        \partial \mathbf{\data}^{\diffusionTime}
    }
    =
    \left(
        \prod_{\diffusionTime=\diffusionTimeMax}^{1}
        \frac{\diffusionCoeff^{\diffusionTime - 1}}{\diffusionCoeff^{\diffusionTime}}
    \right)
    \mathbf{I}.
\]
Therefore, the stop-gradient gradient and the full sampler gradient differ primarily by timestep-dependent scalar factors. Our approximation uses this same principle, but applies it directly to optimize the initial noise: the clean-latent criterion gradient is used as a proxy direction for the initial-noise gradient, with the unknown scalar absorbed into \(\beta\) and the learning rate.

\begin{wrapfigure}{r}{0.45\textwidth}
    \centering
    \vspace{-0.1in} 
    \includegraphics[width=\linewidth]{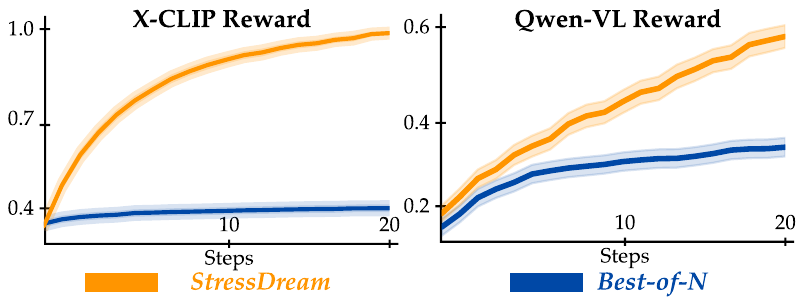} 
    \vspace{-0.22in}
    \caption{\small
    {
        \textit{Average VLM scores across optimization steps in \vista.} With the gradient approximation, \ours successfully increases the criterion values, demonstrating the effectiveness of the proposed approximation. 
    }
    }
    \vspace{-0.1in}\label{fig:appendix-gradient}
\end{wrapfigure} \para{Efficacy of Gradient Approximation for Noise Optimization} The noise-gradient approximation enables efficient gradient-based steering without backpropagating through the iterative denoising process of large-scale video diffusion models. However, it relies on the assumption that the Jacobian of the diffusion sampler is approximately diagonal~\cite{ahn2024noise}. Indeed, prior work on noise optimization for text-to-image generation often relies on one-step distilled generators to make noise optimization tractable without differentiating through the full iterative sampling process.

We therefore empirically evaluate whether this approximation provides a useful gradient signal in the initial-noise space for generating high-reward samples in the output space. Fig.~\ref{fig:appendix-gradient} shows the average VLM score during noise optimization for steering Vista~\cite{gao2024vista} generations in Sec.~\ref{sec:experiments-inner}. The score increases throughout optimization, indicating that the approximate gradient can improve the output-space VLM objective by updating only the initial noise. On an H100 GPU with \(80\)GB VRAM, computing the full gradient through the denoising trajectory is impractical even with gradient checkpointing~\cite{chen2016training}, whereas the proposed approximation makes the optimization feasible.

\begin{wrapfigure}{r}{0.4\textwidth}
    \centering
    \vspace{-0.0in} 
    \includegraphics[width=\linewidth]{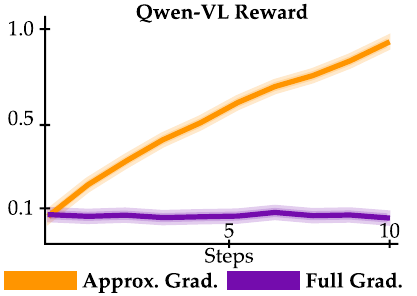} 
    \vspace{-0.22in}
    \caption{\small{\textit{Noise optimization with approximate vs.\ exact gradients.}}}
    \vspace{-0.1in}\label{fig:appendix-gradient-ctrlworld}
\end{wrapfigure}
For Ctrl-World~\cite{guo2025ctrl}, where the full noise gradient can be computed on an H100 GPU using gradient checkpointing~\cite{chen2016training}, we compare noise optimization with and without the gradient approximation. Fig.~\ref{fig:appendix-gradient-ctrlworld} shows that the approximate gradient successfully increases the VLM score, whereas optimization with the full gradient fails to improve it. We hypothesize two possible causes. First, backpropagating through \(50\) denoising steps may lead to inaccurate or vanishing gradients due to finite-precision computation. Second, while the VLM gradient is computed in the video space with relatively high numerical precision, the denoising process is typically run in lower precision due to memory constraints, which may further degrade gradient quality.

\subsection{Implementation Details}

We summarize the overall procedure of \ours in Algorithm~\ref{alg:noise_optim}. To improve optimization stability, we use both global gradient-norm clipping and per-coordinate gradient clipping. We additionally monitor the Euclidean norm of the noise and project it back to the typical shell, with norm \(\sqrt{\dimension}\), whenever it deviates by more than \(3.0\). This threshold is conservative: the standard deviation of the Gaussian norm is approximately \(0.707\) in high dimensions, as shown in Corollary~\ref{cor:dimension-independent-radius-width}, so the interval contains more than \(99.99\%\) of typical Gaussian samples. In addition, we omit momentum in the optimizer, since accumulated updates can easily push the noise outside the typical set.

\begin{algorithm}[H]
\caption{\ours}
\label{alg:noise_optim}

\SetKwInput{Input}{Input}
\SetKwInput{Initialize}{Initialize}

\Input{
    Action-conditioned video world model $\wm_{\wmParam}$, 
    Inference-time prompt $\prompt$, 
    Observation history $\historyObs$, 
    Action sequence $\actionSequence$, 
    Optimization steps $m$, 
    Learning rate $\eta$,
    Grad. approx. constant $\beta$
}

\Initialize{
    Initial noise $\bm{\noise}_0 \sim \mathcal{N}(0, \mathbf{I})$\;
    Best score $\reward^* \leftarrow -\infty$\;
    Best generation $\mathbf{\obs}^* \leftarrow \text{None}$\;
}

\For{$i = 0$ \KwTo $m-1$}{
    \tcp{Generate future observations from current noise}
    $\mathbf{\obs}_i = \wm_{\wmParam}(\bm{\noise}_i, \historyObs, \actionSequence)$\    

    \tcp{Evaluate Generation with VLMs (Sec.~\ref{sec:method-vlm})}
    $\reward_i = \log p^{\text{VLM}}(\texttt{yes}\mid \mathbf{\obs}_i, \prompt) - \log p^{\text{VLM}}(\texttt{no}\mid \mathbf{\obs}_i, \prompt)$ \
    
    \tcp{Compute Raw Gradient Components}
    $\mathbf{g}_{\text{vlm}} = \beta \cdot \nabla_{\mathbf{\obs}_i} \reward_i$\;
    $\mathbf{g}_{\text{reg}} = \nabla_{\bm{\noise}_i} \regularizer(\bm{\noise}_i)$\;
    $\mathbf{g}_{\text{total}} = \mathbf{g}_{\text{vlm}} + \mathbf{g}_{\text{reg}}$\;

    \tcp{1) Gradient clipping for each dimension}
    $\mathbf{g}_{\text{total}} = \text{clamp}(\mathbf{g}_{\text{total}}, -0.3, 0.3)$\;

    \tcp{2) Gradient norm clipping}
    \If{$\|\mathbf{g}_{\text{total}}\|_2 > 1.0$}{
        $\mathbf{g}_{\text{total}} = 1.0 \cdot \frac{\mathbf{g}_{\text{total}}}{\|\mathbf{g}_{\text{total}}\|_2}$\;
    }
    
    \tcp{Update Initial Noise via Gradient Ascent}
    $\bm{\noise}_{i+1} = \bm{\noise}_i + \eta \cdot \mathbf{g}_{\text{total}}$\;

    \tcp{Hard Normalize to Typical Shell if outside $\pm 3.0$ range}
    \If{$| \|\bm{\noise}_{i+1}\|_2 - \sqrt{\dimension} | > 3.0$}{
        $\bm{\noise}_{i+1} = \sqrt{\dimension} \cdot \frac{\bm{\noise}_{i+1}}{\|\bm{\noise}_{i+1}\|_2}$\;
    }
    
    \If{$\reward_i > \reward^*$}{
        $\reward^* = \reward_i$\;
        $\mathbf{\obs}^* = \mathbf{\obs}_i$\;
    }
}

\Return{Worst-case generation $\mathbf{\obs}^*$}
\end{algorithm}

\subsection{Time Complexity}\label{sec:appendix-time-complexity}

We analyze the time complexity of the proposed world model steering algorithm. Let \(\mathcal{T}_{\mathrm{denoise}}\) denote the cost of one denoising step and \(K\) the number of denoising steps used by the video world model. A single deterministic generation costs
\(
    \mathcal{T}_{\mathrm{gen}} = K \times \mathcal{T}_{\mathrm{denoise}}.
\)
Let \(\mathcal{T}_{\mathrm{verifier}}\) denote the cost of evaluating the verifier, such as a VLM-based criterion, on a generated video.

\begin{wraptable}{r}{0.6\textwidth}
    \vspace{-0.15in}
    \scriptsize
    \centering
    \setlength{\tabcolsep}{4pt}
    \renewcommand{\arraystretch}{1.15}
    \resizebox{1.0\linewidth}{!}{
    \begin{tabular}{lc}
        \toprule
        \textbf{Method} & \textbf{Time Complexity} \\
        \midrule
        Nominal generation
        &
        \(K\times\mathcal{T}_{\mathrm{denoise}}
        +
        \mathcal{T}_{\mathrm{verifier}}\)
        \\
        Best-of-\(N\)
        &
        \(N\times\left(
        K\times\mathcal{T}_{\mathrm{denoise}}
        +
        \mathcal{T}_{\mathrm{verifier}}
        \right)\)
        \\
        \ours
        &
        \(N\times\left(
        K\times\mathcal{T}_{\mathrm{denoise}}
        +
        2\,\mathcal{T}_{\mathrm{verifier}}
        \right)\)
        \\
        \ours w/o Grad. Approx.~(Sec.\ref{sec:grad-approx})
        &
        \(N\times\left(
        2K\times\mathcal{T}_{\mathrm{denoise}}
        +
        2\,\mathcal{T}_{\mathrm{verifier}}
        \right)\)
        \\
        \bottomrule
    \end{tabular}
    }
    \vspace{-0.05in}
    \caption{\small{Time complexity comparison. Here, \(K\) is the number of denoising steps, \(N\) is the number of sampled candidates or steering iterations, \(\mathcal{T}_{\mathrm{denoise}}\) is the cost of one denoising step, and \(\mathcal{T}_{\mathrm{verifier}}\) is the cost of one verifier evaluation.}}
    \label{tab:time-complexity}
    \vspace{-0.15in}
\end{wraptable}

Table~\ref{tab:time-complexity} compares nominal generation, Best-of-\(N\) sampling, and \ours under the same budget of \(N\) sampled candidates or steering iterations. Compared to Best-of-\(N\) sampling, \ours adds the cost of computing a verifier gradient at each steering iteration. With the gradient approximation in Sec.~\ref{sec:grad-approx}, however, \ours avoids backpropagating through the full denoising trajectory, so the additional cost is mainly the backward pass through the verifier, represented by the second \(\mathcal{T}_{\mathrm{verifier}}\) term. Without this approximation, differentiating through the \(K\)-step sampler adds a backward pass through the denoising trajectory, substantially increasing both computation and memory.

In practice, using Vista~\cite{gao2024vista} with \(K=50\) denoising steps on a single H100 GPU, a single generation takes approximately \(1\)--\(2\) minutes, while noise optimization with \(N=20\) iterations takes about \(30\) minutes. Since the dominant cost comes from video generation, the runtime of \ours is expected to improve with faster video world models, such as shortcut~\cite{frans2025one, hafner2025training} models.

\section{\textit{Naughty} Dubins Car}\label{sec:appendix-dubins}

\subsection{Implementation Details}
\para{Video World Model}
We implement a video world model for the Naughty Dubins car environment using a smaller-scale SVD-style latent diffusion architecture~\cite{blattmann2023stable} trained with the EDM formulation~\cite{karras2022elucidating}. Following \citet{seo2025uncertainty}, each observation is a rendered \(128\times128\times3\) RGB image showing the vehicle and the circular failure set at the center of the environment. 
\[
    \text{Encoder: } \latent_t = \encoder_\wmParam(\obs_t),
    \qquad
    \text{Transition: } \latent_{t+1} \sim \wm_\wmParam(\latent_{t+1}\mid \noise, \latent_t,\action_t),
    \qquad
    \text{Failure: } \hat{\ell}_t = \ell_\wmParam(\latent_t).
\]
\begin{wraptable}{r}{0.6\textwidth}
    \vspace{-0.15in}
    \scriptsize
    \centering
    \setlength{\tabcolsep}{4pt}
    \renewcommand{\arraystretch}{1.15}
    \resizebox{1.0\linewidth}{!}{
    \begin{tabular}{lc}
        \toprule
        \textbf{\textsc{World Model Hyperparameter}} & \textbf{\textsc{Value}} \\
        \midrule
        \textsc{Image Resolution}            & \(128 \times 128 \times 3\) \\
        \textsc{Latent Dimension}            & 1024 \\
        \textsc{Prediction Horizon \(\horizon\)} & 1 \\
        \textsc{Action Dimension}            & 1 (continuous) \\
        \textsc{VAE Base Channels}           & 32 \\
        \textsc{VAE Latent Channels}         & 4 \\
        \textsc{Denoiser Hidden Dims.}       & [32, 64, 128] \\
        \textsc{Denoising Steps}             & 5 \\
        \textsc{\(\sigma_{\min}\), \(\sigma_{\max}\)} & 0.002, 80.0 \\
        \textsc{EDM \(\rho\)}                & 7 \\
        \textsc{Batch Size}                  & 128 \\
        \textsc{Training Epochs}             & 100 \\
        \textsc{Learning Rate}               & \(1\times10^{-4}\) \\
        \bottomrule
    \end{tabular}
    }
    \vspace{-0.05in}
    \caption{\small{Dubins Car WM Hyperparameters.}}
    \label{tab:dubins_hyperparams}
    \vspace{-0.15in}
\end{wraptable} We train the model on an offline dataset of observation--action trajectories generated from random initial states and random actions, with each trajectory annotated by the ground-truth safety margin. Each trajectory terminates after \(T=100\) timesteps or earlier if the ground-truth \(x\)- or \(y\)-position leaves the environment bounds \([-1.5\,\mathrm{m},1.5\,\mathrm{m}]\). The encoder maps observations \(\obs_t\in\mathbb{R}^{H\times W\times C}\) to latent representations
\(\latent_t\in\mathbb{R}^{H/8\times W/8\times4}\), which are decoded back to RGB observations using a lightweight decoder. We train the latent autoencoder with a reconstruction objective and train the failure-margin predictor \(\ell_\wmParam\) using an \(L_2\) regression loss against the ground-truth safety margin. We then train the latent diffusion transition model \(\wm_\wmParam\) to generate the next latent state conditioned on the current latent state and action. In our experiments, we use one-step latent prediction \((\horizon=1)\), yielding a latent and noise dimension of \(1024\). At inference time, we use an Euler sampler with \(5\) denoising steps. The world-model hyperparameters are summarized in Tables~\ref{tab:dubins_hyperparams}.

\para{Evaluation Setup}
We evaluate steering on autoregressive rollouts of the video world model. Each rollout starts from the initial image of an evaluation trajectory and is conditioned on the corresponding action sequence. We assign a ground-truth label to each evaluation trajectory using Monte Carlo rollouts of the true uncertain dynamics in \eqref{eq:naughty-dynamics}. For each initial state and action sequence, we sample \(10{,}000\) ground-truth trajectories. We label the trajectory as \textit{positive} if at least one sampled trajectory enters the failure set, indicating that failure is possible. \begin{wraptable}{r}{0.6\textwidth}
    \vspace{-0.15in}
    \scriptsize
    \centering
    \setlength{\tabcolsep}{4pt}
    \renewcommand{\arraystretch}{1.15}
    \resizebox{1.0\linewidth}{!}{
    \begin{tabular}{lc}
    \toprule
    \textbf{\textsc{Steering Hyperparameter}} & \textbf{\textsc{Value}} \\
    \midrule
    \textsc{Optimization Iterations}      & 10 \\
    \textsc{Step Size \(\eta\) in \eqref{eq:gradient-ascent-final}} & 1.0 \\
    \textsc{Gradient Scaling Coefficient\(\beta\) in \eqref{eq:gradient-ascent-final}} & 10.0 \\
    \textsc{Norm Concentration Coefficient \(\lambda_1\)} & 1.0 \\
    \textsc{Isotropy Coefficient \(\lambda_2\)} & 0.5 \\
    \textsc{Isotropy Regularizer Subvector Size \(k\)} & 16 \\
    \textsc{Isotropy Regularizer Permutation Number} & 100 \\
    \textsc{Spectral Coefficient \(\lambda_3\)} & 5.0 \\
    \bottomrule
\end{tabular}
} \vspace{-0.05in}
\caption{\small Dubins Car Steering hyperparameters.}
\label{tab:dubins_opt_hyperparams} \vspace{-0.2in}
\end{wraptable} We label it as \textit{negative} if none of the sampled trajectories enter the failure set. This allows us to evaluate whether pessimistic steering identifies genuinely failure-prone trajectories while avoiding spurious failures on trajectories that are effectively safe under the true dynamics. We steer the generated rollout using the hyperparameters in Table~\ref{tab:dubins_opt_hyperparams}. Since this setup is small-scale, we use the full gradient without the approximation. This allows us to directly evaluate the effect of optimizing the initial noise and the proposed regularizers without introducing additional approximation error from the gradient estimator.

\para{Baseline: Classifier Guidance}
We also compare against classifier guidance~\cite{dhariwal2021diffusion,li2022upainting}, which modifies the denoising process at each diffusion step using the gradient of the failure criterion. In the world-model setting, however, such guidance is difficult to apply directly. The failure predictor \(\ell_{\wmParam}\) is trained on clean latent states, not on noisy intermediate diffusion states. Therefore, evaluating \(\ell_{\wmParam}\) directly on \(\latent_{t}^{\diffusionTime}\) can provide unreliable gradients, while exact guidance would require backpropagating through the denoiser at every denoising step.

To make classifier guidance tractable, we use a one-step clean-latent approximation~\cite{li2022upainting}. Given a noisy latent \(\latent_{t+1}^{\diffusionTime}\), we estimate the corresponding clean latent by inverting the forward diffusion path:
\[
    \hat{\latent}_{t+1}^{0\mid\diffusionTime}
    =
    \frac{
        \latent_{t+1}^{\diffusionTime}
        -
        \noiseCoeff^{\diffusionTime}
        \denoise_{\wmParam}^{\diffusionTime}
        \left(
            \latent_{t+1}^{\diffusionTime},
            \latent_t,
            \action_t
        \right)
    }{
        \diffusionCoeff^{\diffusionTime}
    }.
\]
The guidance gradient is then computed on this approximately clean latent:
\[
    \nabla_{\latent_{t+1}^{\diffusionTime}}
    \ell_{\wmParam}
    \left(
        \hat{\latent}_{t+1}^{0\mid\diffusionTime}
    \right)
    \approx
    \frac{1}{\diffusionCoeff^{\diffusionTime}}
    \nabla_{\hat{\latent}_{t+1}^{0\mid\diffusionTime}}
    \ell_{\wmParam}
    \left(
        \hat{\latent}_{t+1}^{0\mid\diffusionTime}
    \right),
\]
which avoids backpropagating through the denoising network, but it can introduce bias because the one-step clean prediction may be misaligned with the final generated latent after the full denoising trajectory. This issue is particularly pronounced in action-conditioned world models, where small guidance errors can push the generation toward high-failure latents that are inconsistent with the conditioning action sequence or the learned dynamics.

\begin{wrapfigure}{r}{0.25\textwidth}
    \centering
    \vspace{-0.2in}
    \includegraphics[width=\linewidth]{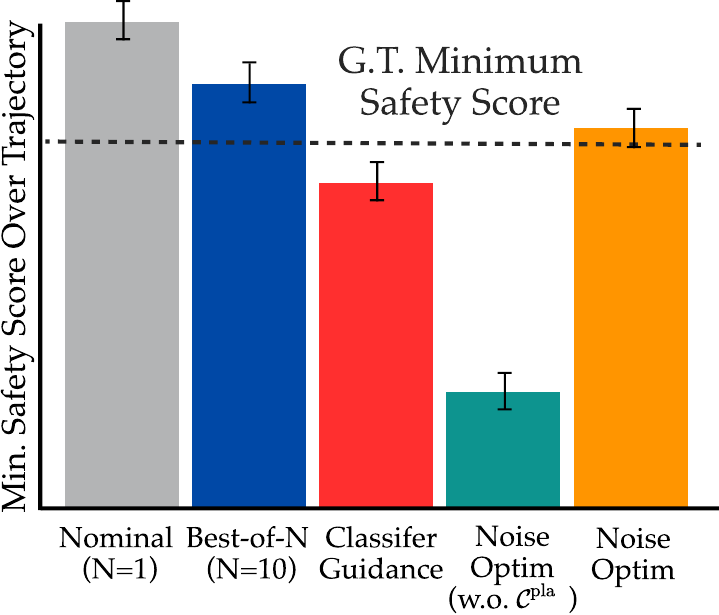} 
    \vspace{-0.22in}
    \caption{\small{\textit{Average minimum safety scores over imagined trajectories.}}}
    \vspace{-0.2in}
    \label{fig:app-dubins-pessimistic}
\end{wrapfigure}
\para{Results: Pessimistic Steering}
Fig.~\ref{fig:app-dubins-pessimistic} reports the average minimum safety score over imagined trajectories for the experiments in Sec.~\ref{sec:dubins}. \ours effectively steers imaginations toward lower safety scores that closely match the minimum achievable safety scores of the ground-truth system, suggesting that WM imaginations can provide meaningful signals for policy evaluation. In contrast, \classGuidance and \ours without the plausibility objective predict implausibly low safety scores, indicating that they can over-steer generations beyond plausible outcomes.

\subsection{Ablation: Optimistic Steering}
\begin{wrapfigure}{l}{0.3\textwidth}
    \centering
    \vspace{-0.1in}
    \includegraphics[width=\linewidth]{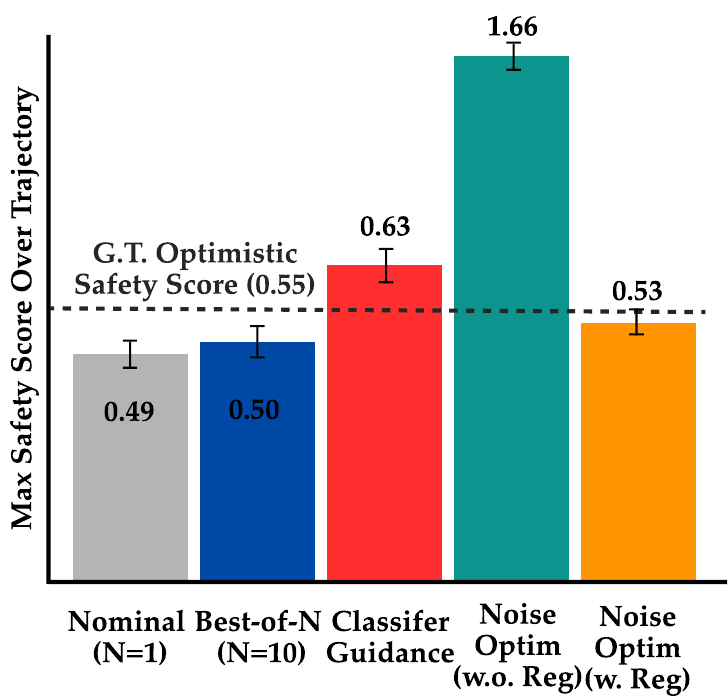} 
    \vspace{-0.22in}
    \caption{\small{\textit{Ablation: Optimistic Steering of Naughty Dubins Car.}}}
    \vspace{-0.2in}
    \label{fig:app-dubins-optimistic}
\end{wrapfigure}While Sec.~\ref{sec:dubins} focuses on pessimistic steering to detect potential entries into the failure set, we also ablate the opposite direction by performing optimistic steering. Specifically, we change the sign of the steering objective so that optimization seeks imaginations with higher safety scores. Fig.~\ref{fig:app-dubins-optimistic} shows the resulting maximum safety score for each imagined trajectory, together with the maximum safety score achievable under the ground-truth stochastic dynamics as a horizontal reference line. The results are consistent with the findings in Sec.~\ref{sec:dubins}: \ours can steer the world model effectively, here toward optimistic outcomes, while keeping the generated trajectories plausibly bounded by the ground-truth uncertainty. In contrast, removing the regularizers or using classifier guidance can also increase the imagined safety score, but often produces overly optimistic generations that exceed what is plausible under the ground-truth system dynamics.

\subsection{Ablation: Impact of plausibility objective}

We ablate the impact of the regularizers in Sec.~\ref{sec:high-dim-gaussian} on keeping generations plausible, or in-distribution. Unlike video world models trained on internet-scale data, this controlled setting gives us tractable access to the training distribution, allowing us to directly inspect whether steering keeps generations consistent with the training-time distribution. We report typicality scores of the optimized initial noise, including the chi-square log-probability of the squared noise norm, the isotropy regularizer value, and the spectral-whiteness regularizer value. We also remove each regularizer individually to evaluate its contribution.

\para{Empirical Density Estimation}
We additionally evaluate whether the generated samples themselves remain \textit{in-distribution} by fitting a flow-matching density surrogate in the latent observation space~\cite{xu2025can}. Let
\(\latent=\encoder_\wmParam(\obs)\) denote the latent representation of an observation from the training dataset. We train a flow-matching vector field \(q_\psi\) that transports training latents to standard Gaussian noise. Specifically, for \(w\sim\mathcal{N}(0,\mathbf{I})\) and \(\tau\in[0,1]\), define the linear interpolation
\[
    \latent^\tau
    =
    \latent
    +
    \tau(w-\latent).
\]
The vector field is trained to predict the displacement toward the Gaussian endpoint:
\[
    q_\psi(\latent^\tau,\tau)
    \approx
    w-\latent.
\]
Given a generated latent \(\hat{\latent}_t\), we estimate its corresponding Gaussian endpoint with a one-step prediction from \(\tau=0\):
\[
    \widehat{w}_t
    =
    \hat{\latent}_t
    +
    q_\psi(\hat{\latent}_t,0).
\]
If \(\hat{\latent}_t\) is in-distribution, then \(\widehat{w}_t\) should be close to a sample from \(\mathcal{N}(0,\mathbf{I})\). Since the Gaussian log-density is proportional to \(-\|\widehat{w}_t\|_2^2\), we use
\[
    s_{\mathrm{OOD}}(\hat{\latent}_t)
    :=
    \|\widehat{w}_t\|_2^2
\]
as the latent OOD score. A higher score indicates that the generated latent is farther from the distribution of encoded in-distribution observations. This density-based OOD score complements the noise-space regularizers by checking whether the resulting generated samples, not only the optimized initial noise, remain plausible.

\para{Results}
Table~\ref{tab:app-dubins-noise-metrics} reports the initial-noise regularizer values and the latent OOD scores of the generated samples. Norm concentration is measured by the minimum chi-square log-probability of the noise norm, where higher values indicate more typical Gaussian noise. Isotropy and spectral whiteness are regularizer values, where lower values are better.

\begin{wraptable}{r}{0.7\textwidth}
    \scriptsize
    \centering
    \setlength{\tabcolsep}{4pt}
    \renewcommand{\arraystretch}{1.15}
    \resizebox{1.0\linewidth}{!}{
    \begin{tabular}{lcccc}
        \toprule
        \textbf{Method}
        &
        \textbf{Norm Conc.} \(\uparrow\)
        &
        \textbf{Isotropy} \(\downarrow\)
        &
        \textbf{Spectral White.} \(\downarrow\)
        &
        \textbf{OOD Score} \(\downarrow\)
        \\
        \midrule
        Nominal
        & \(-7.6\)
        & \(0.02\)
        & \(0.03\)
        & \(167.1\)
        \\
        Best-of-\(N\)
        & \(-7.6\)
        & \(0.02\)
        & \(0.03\)
        & \(167.4\)
        \\
        Classifier Guidance
        & \(-7.6\)
        & \(0.02\)
        & \(0.03\)
        & \(194.5\)
        \\
        \ours
        & \({-6.3}\)
        & \({0.02}\)
        & \({0.00}\)
        & \({171.7}\)
        \\
        \midrule
        \multicolumn{5}{l}{\textit{Ablations without regularization}} \\
        \ours~(w/o all reg.)
        & \(-717.6\)
        & \(0.83\)
        & \(0.36\)
        & \(696.6\)
        \\
        \ours~(w/o \(\regularizer_{\mathrm{norm}}\))
        & \(-31.7\)
        & \(0.08\)
        & \(0.06\)
        & \(206.1\)
        \\
        \ours~(w/o \(\regularizer_{\mathrm{iso}}\))
        & \(-5.6\)
        & \(0.03\)
        & \(0.11\)
        & \(187.3\)
        \\
        \ours~(w/o \(\regularizer_{\mathrm{spec}}\))
        & \(-6.3\)
        & \(0.00\)
        & \(0.11\)
        & \(172.2\)
        \\
        \bottomrule
    \end{tabular}
    }
    \vspace{-0.05in}
    \caption{\small{Initial-noise regularizer values and latent OOD scores of generated samples. \ours with all regularizers keeps the optimized noise close to the Gaussian prior and produces low-OOD generations, whereas removing regularization leads to atypical noise and more out-of-distribution generations.}}
    \label{tab:app-dubins-noise-metrics}
    \vspace{-0.1in}
\end{wraptable}
\ours with all regularizers keeps the noise statistics close to nominal generation while maintaining a low OOD score. In contrast, steering without regularization produces highly atypical noise and substantially larger OOD scores, suggesting that the optimized generations leave the training distribution. Classifier guidance keeps the initial noise statistics close to nominal because it modifies the denoising trajectory rather than the initial noise, but it still increases the OOD score, indicating that path-level guidance can push generated samples toward implausible regions. Removing individual regularizers also increases the OOD score compared to using all regularizers, highlighting the importance of combining norm concentration, isotropy, and spectral whiteness.

\section{Experiment Details}\label{sec:app-sota}

\subsection{Video World Models}\label{sec:app-sota-vwm}

\para{Hyperparameters} Table~\ref{tab:app-vista-hyperparam} and Table~\ref{tab:app-ctrl-world-hyperparam} summarize the video world model and noise-optimization hyperparameters for autonomous driving and robotic manipulation, respectively. For the driving world model, Vista~\cite{gao2024vista}\footnote{\url{https://github.com/OpenDriveLab/Vista}}, a single generation produces \(25\) future frames at \(10\)Hz. The action input is an \(8\)-dimensional vector representing the top-down future ego trajectory as four future waypoints. For Ctrl-World~\cite{guo2025ctrl}\footnote{\url{https://github.com/Robert-gyj/Ctrl-World}}, videos are subsampled to \(5\)Hz, and the action input consists of joint-position and gripper commands. Since each prediction generates \(5\) future frames, the action dimension is \(40=8\times\horizon\). To imagine longer manipulation trajectories, we use autoregressive generation, where the last generated frame is used as the initial conditioning frame for the next prediction window.

\begin{table}[ht]
\centering
\begin{minipage}[t]{0.48\linewidth}
    \vspace{0pt}
    \scriptsize
    \centering
    \setlength{\tabcolsep}{4pt}
    \renewcommand{\arraystretch}{1.15}
    \resizebox{\linewidth}{!}{
    \begin{tabular}{lc}
    \toprule
    \textbf{\textsc{Hyperparameter}} & \textbf{\textsc{Value}} \\
    \midrule
    \textsc{Number of Cameras}            & 1 \\
    \textsc{Image Resolution}             & \(576 \times 1024 \times 3\) \\
    \textsc{Video Frequency}              & 10 Hz \\
    \textsc{Prediction Horizon}           & 25 \\
    \textsc{Noise Dimension}              & 921,600 \\
    \textsc{Action Dimension}             & 8 \\
    \textsc{Denoising Steps}              & 50 \\
    \textsc{CFG Scale}                    & 2.5 \\
    \textsc{Optimization Iterations}      & 20 \\
    \textsc{Step Size \(\eta\)}           & 1.0 \\
    \textsc{Gradient Scaling \(\beta\)}   & 300.0 \\
    \textsc{Norm Coefficient \(\lambda_1\)} & 0.5 \\
    \textsc{Isotropy Coefficient \(\lambda_2\)} & 10.0 \\
    \textsc{Isotropy Subvector Size \(k\)} & 192 \\
    \textsc{Isotropy Permutations}        & 1,000 \\
    \textsc{Spectral Coefficient \(\lambda_3\)} & 100.0 \\
    \bottomrule
    \end{tabular}
    }\vspace{0.05in}
    \caption{\small World model and steering hyperparameters for Vista~\cite{gao2024vista}.}
    \label{tab:app-vista-hyperparam}
\end{minipage}
\hfill
\begin{minipage}[t]{0.46\linewidth}
    \vspace{0pt}
    \scriptsize
    \centering
    \setlength{\tabcolsep}{4pt}
    \renewcommand{\arraystretch}{1.15}
    \resizebox{\linewidth}{!}{
    \begin{tabular}{lc}
    \toprule
    \textbf{\textsc{Hyperparameter}} & \textbf{\textsc{Value}} \\
    \midrule
    \textsc{Number of Cameras}            & 3 \\
    \textsc{Image Resolution}             & \(192 \times 320 \times 3\) \\
    \textsc{Video Frequency}              & 5 Hz \\
    \textsc{Prediction Horizon}           & 5 \\
    \textsc{Noise Dimension}              & 57,600 \\
    \textsc{Action Dimension}             & \(40\;(8\times\horizon)\) \\
    \textsc{Denoising Steps}              & 50 \\
    \textsc{CFG Scale}                    & 2.0 \\
    \textsc{Optimization Iterations}      & 10 \\
    \textsc{Step Size \(\eta\)}           & 1.0 \\
    \textsc{Gradient Scaling \(\beta\)}   & 100.0 \\
    \textsc{Norm Coefficient \(\lambda_1\)} & 0.2 \\
    \textsc{Isotropy Coefficient \(\lambda_2\)} & 0.1 \\
    \textsc{Isotropy Subvector Size \(k\)} & 240 \\
    \textsc{Isotropy Permutations}        & 1,000 \\
    \textsc{Spectral Coefficient \(\lambda_3\)} & 100.0 \\
    \bottomrule
    \end{tabular}
    }\vspace{0.05in}
    \caption{\small World model and steering hyperparameters for Ctrl-World~\cite{guo2025ctrl}.}
    \label{tab:app-ctrl-world-hyperparam}
\end{minipage}
\vspace{-0.2in}
\end{table}

\para{Fine-tuning}
We find that fine-tuning the base checkpoints of the video world models is important for plausibly imagining task-relevant rare events, such as collisions or coffee-bean spills, and for aligning the train- and test-time distributions to improve video quality. While the base models are trained on diverse datasets, such as nuScenes~\cite{caesar2020nuscenes} or DROID~\cite{khazatsky2024droid}, their generated futures can suffer from degraded dynamics accuracy and visual quality when evaluated on our specific scenarios. We therefore fine-tune each video world model for a small number of iterations on task-relevant data.

\begin{table}[ht]
    \centering
    \small
    \setlength{\tabcolsep}{4pt}
    \renewcommand{\arraystretch}{1.15}
    \resizebox{0.95\linewidth}{!}{
    \begin{tabular}{lcccccc}
    \toprule
    \textbf{Task}
    &
    \textbf{Block Stack}
    &
    \textbf{Knife Put}
    &
    \textbf{Stacked Utensil Pick}
    &
    \textbf{Coffee Bean Pour}
    &
    \textbf{Open Coffee Bag}
    &
    \textbf{Open Candy Bag}
    \\
    \midrule
    \textbf{\# Trajectories}
    & \(250\)
    & \(170\)
    & \(200\)
    & \(100\)
    & \(170\)
    & \(150\)
    \\
    \bottomrule
    \end{tabular}
    }\vspace{0.05in}
    \caption{\small Number of teleoperation trajectories used to fine-tune video world model.}
    \label{tab:app-manipulation-traj-num}
    \vspace{-0.2in}
\end{table}

For autonomous driving, we fine-tune the video world model using a subset of the PhysicalAI-Autonomous-Vehicles (PAI-AV) dataset~\cite{nvidia_physicalai_av}, nuScenes~\cite{caesar2020nuscenes}, and the Nexar Collision Prediction Dataset~\cite{moura2025nexar}. We train for \(13{,}750\) iterations with a batch size of \(256\). For manipulation, we fine-tune Ctrl-World using teleoperation trajectories that include both successful and failure cases. We use a learning rate of \(1\times10^{-6}\), a batch size of \(64\), and train for \(10{,}000\) iterations. Table~\ref{tab:app-manipulation-traj-num} summarizes the number of teleoperation trajectories used for each manipulation task.

\newcommand{\vlmprompt}[1]{%
    \noindent\textit{VLM prompt:}
    {\small\ttfamily #1}%
    \vspace{-0.05in}
}

\subsection{Robotic Manipulation: Task Details}\label{sec:appendix-manipulation-details}

We describe the manipulation task setups, the criteria used to determine success or failure, and the language prompts used for steering the generated futures. Visual examples of the task setups are shown in Fig.~\ref{fig:pi-qualitative} and on the \href{https://stressdream.github.io/}{project website}.

\para{Block Stack}
Two Jenga blocks, one orange and one blue, are placed on the table. The robot is tasked with stacking the orange block on top of the blue block. A trajectory is considered a \textit{success} if the orange block remains on top of the blue block, and a \textit{failure} if the robot fails to stack the block or if the orange block falls off after stacking.
\vlmprompt{Are the two blocks stacked? Respond with a single word: Yes or No.}

\para{Knife Put}
A lightweight plastic bowl and a knife are placed on the table. The robot is tasked with picking up the knife and placing it inside the bowl without causing either object to fall. This task is challenging because the knife has an uneven center of mass and a complex shape, while the lightweight bowl can easily tip over. A trajectory is considered a \textit{success} if the knife is placed inside the bowl and remains there, and a \textit{failure} if the knife falls out or the bowl tips over.
\vlmprompt{Is the knife placed inside of the blue bowl without falling out? Respond with a single word: Yes or No.}

\para{Stacked Utensil Pick}
A blue knife and a red spoon are overlapped on a green plate, with the red spoon placed on top of the blue knife. The robot is instructed to pick up the blue knife while keeping the red spoon on the plate. A trajectory is considered a \textit{success} if the blue knife is successfully picked up and the spoon remains on the plate, and a \textit{failure} if the blue knife is not picked up or if the spoon falls off the plate. The robot should therefore remove the lower utensil while safely leaving the upper utensil on the plate.
\vlmprompt{Is the red spoon in the green plate? Respond with a single word: Yes or No.}

\para{Coffee Bean Pour}
A bowl is on top of a coffee scale, and a cup contains coffee beans. The robot is tasked with pouring the coffee beans from the cup into the bowl on the scale. A trajectory is considered a \textit{success} if the coffee beans are poured into the bowl without spilling, and a \textit{failure} if any coffee beans spill onto the table.
\vlmprompt{Is the coffee bean spilled on the table? Respond with a single word: Yes or No.}

\para{Open Coffee Bag}
An open coffee bag with a variable and partially unobservable amount of coffee beans inside is on the table, along with a container. The robot is instructed to place the coffee bag inside the container. A trajectory is considered a \textit{failure} if any coffee beans spill out of the bag, either into the container or onto the table. Thus, the robot must avoid spilling while lifting, carrying, and placing the bag into the container.
\vlmprompt{Is the coffee bean spilled? Respond with a single word: Yes or No.}

\para{Open Candy Bag}
An open gummy bear candy bag is on the table, and the robot is instructed to place it inside the container without spilling. Compared to coffee beans, gummy bears are stickier and heavier, creating a smaller risk of spilling out of the bag. A trajectory is considered a \textit{success} if the gummy bears remain inside the bag, and a \textit{failure} if any gummy bears spill out.
\vlmprompt{Does candy spill from an open bag? Respond with a single word: Yes or No.}

\subsection{Autonomous Driving: Evaluation Data Curation}
\label{sec:app-av-curation}

To evaluate steering for autonomous driving video world models, we curate an evaluation set from the PAI-AV dataset~\cite{nvidia_physicalai_av}. We focus on scenarios in which the future evolution is uncertain or interaction-dependent, such as pedestrian crossing, vehicle merging, traffic-light changes, and stop-sign behavior. Table~\ref{tab:app-av-curation} summarizes the event categories and the verifier prompts used for Qwen-VL~\cite{Qwen-VL} and X-CLIP~\cite{ma2022x}. For X-CLIP, the first text description is treated as the positive event description, while the remaining descriptions are used as negative alternatives.

To retrieve candidate clips from the large-scale dataset, we first extract VLM reasoning traces from Alpamayo-R1~\cite{wang2025alpamayo} on PAI-AV clips. We then embed these traces using Sentence-BERT~\cite{reimers-2019-sentence-bert} and retrieve the top-\(k\) clips whose embeddings are most similar to text descriptions of each event category. Human annotators then verify whether the target event occurs in each retrieved clip and label the event occurrence time. Since the video world model generates \(2.5\)s futures, we select the frame \(2.5\)s before the annotated event time as the initial observation for evaluation. We similarly curate an evaluation set of \(200\) collision clips from the Nexar Collision Prediction Dataset~\cite{moura2025nexar} by randomly sampling from the provided collision annotations.

\begin{table*}[t]
\centering
\scriptsize
\setlength{\tabcolsep}{3pt}
\renewcommand{\arraystretch}{1.15}
\resizebox{\textwidth}{!}{
\begin{tabular}{p{0.12\textwidth}p{0.53\textwidth}p{0.35\textwidth}}
\toprule
\textbf{Event Category}
&
\textbf{Qwen-VL Prompt}
&
\textbf{X-CLIP Text Descriptions}
\\
\midrule

Pedestrian Crossing
&
You are a driving assistant interpreting sequential frames from a car's front camera. There is a pedestrian in front of the ego vehicle. Is it true that a pedestrian crosses in front of the ego vehicle? Reply with exactly one word: 'Yes' if a pedestrian crosses in front of the ego vehicle, or 'No' if a pedestrian crossing in front of the ego vehicle is not happening.
&
\textbf{Positive:} A pedestrian in front of the ego vehicle is crossing the road.

\textbf{Negative:} A pedestrian in front of the ego vehicle is not crossing the road.

\textbf{Negative:} There is no pedestrian in front of the ego vehicle.
\\
\midrule

Cyclist Crossing Crosswalk
&
You are an expert driving assistant interpreting sequential frames from a car's front camera. Does a cyclist in front of the ego vehicle cross a crossway? Reply with exactly one word: 'Yes' if a cyclist crosses a crossway, or 'No' if a cyclist crosses a crossway is not happening.
&
\textbf{Positive:} A cyclist in front of the ego vehicle crosses a crossway.

\textbf{Negative:} A cyclist in front of the ego vehicle does not cross a crossway.
\\
\midrule

Vehicle Merging from Adjacent Lane
&
You are an expert driving assistant interpreting sequential frames from a car's front camera. Does a vehicle in the [left$\mid$right] lane merge into the ego vehicle's lane? Reply with exactly one word: 'Yes' if a vehicle in the [left$\mid$right] lane merges into the ego vehicle's lane, or 'No' if a vehicle in the [left$\mid$right] lane merges into the ego vehicle's lane is not happening.
&
\textbf{Positive:} A vehicle in the [left$\mid$right] lane merges into the ego vehicle's lane.

\textbf{Negative:} A vehicle in the [left$\mid$right] lane does not merge into the ego vehicle's lane.

\textbf{Negative:} A vehicle in the [left$\mid$right] lane does not move.
\\
\midrule

Opposing Vehicle Stops at Intersection
&
You are an expert driving assistant interpreting sequential frames from a car's front camera. Does an opposing vehicle at the intersection stop? Reply with exactly one word: 'Yes' if an opposing vehicle at the intersection stops, or 'No' if an opposing vehicle at the intersection stops is not happening.
&
\textbf{Positive:} An opposing vehicle at the intersection stops.

\textbf{Negative:} An opposing vehicle at the intersection does not stop.

\textbf{Negative:} An opposing vehicle at the intersection keeps moving.
\\
\midrule

Lead Vehicle Closing Distance
&
You are an expert driving assistant interpreting sequential frames from a car's front camera. Is the lead vehicle getting closer to the view of the front camera? The leading vehicle is the vehicle ahead in view only on the same lane, not in the other lane. Reply with exactly one word: 'Yes' if the lead vehicle is getting closer to the view of the front camera, or 'No' if the lead vehicle is not getting closer to the view of the front camera.
&
\textbf{Positive:} Ego vehicle is getting closer to the leading vehicle

\textbf{Negative:} The leading vehicle is getting further away from the ego vehicle

\textbf{Negative:} The distance between the ego vehicle and the leading vehicle is not changing
\\
\midrule

Lead Vehicle Pulling Away
&
You are an expert driving assistant interpreting sequential frames from a car's front camera. Is the lead vehicle getting further from the view of the front camera? The leading vehicle is the vehicle ahead in view only on the same lane, not in the other lane. Reply with exactly one word: 'Yes' if the lead vehicle is getting further from the view of the front camera, or 'No' if the lead vehicle is not getting further from the view of the front camera.
&
\textbf{Positive:} The leading vehicle is getting further away from the ego vehicle

\textbf{Negative:} Ego vehicle is getting closer to the leading vehicle

\textbf{Negative:} The distance between the ego vehicle and the leading vehicle is not changing
\\
\midrule

Traffic Signal Changing
&
You are an expert driving assistant interpreting sequential frames from a car's front camera. Initially, the [red$\mid$orange$\mid$green] traffic light is on. Is it true that, in the driving scene, the traffic signal visible in the front camera is [red$\mid$orange$\mid$green]? Reply with exactly one word: 'Yes' if the traffic signal is [red$\mid$orange$\mid$green], or 'No' if the traffic signal is not [red$\mid$orange$\mid$green].
&
\textbf{Positive:} The traffic signal is [red$\mid$orange$\mid$green].

\textbf{Negative:} The traffic signal is [red$\mid$orange$\mid$green].

\textbf{Negative:} The traffic signal is [red$\mid$orange$\mid$green].

\textbf{Negative:} The traffic signal is off.
\\
\midrule

Front Vehicle Stop-Sign Violation
&
You are an expert driving assistant interpreting sequential frames from a car's front camera. In the driving scene, the ego vehicle is driving on the road, following the front car. Is it true that, in the driving scene, the front car did not stop at the stop sign? Reply with exactly one word: 'Yes' if a front car did not stop at the stop sign, or 'No' if a front car stopped at the stop sign.
&
\textbf{Positive:} The vehicle in front does not stop at the stop sign.

\textbf{Negative:} The vehicle in front stops at the stop sign.
\\
\midrule

Ego-Vehicle Collision
&
You are a driving safety assistant analyzing sequential frames from a dashcam. The clip shows 2.5 seconds of a driving scene recording. Is a collision or crash between the ego vehicle and another vehicle occurring? A collision means the ego vehicle crashes into another vehicle visible in the front. Reply with exactly one word: 'Yes' if the ego vehicle collides with another vehicle, or 'No' if the situation is safe.
&
\textbf{Positive:} The ego vehicle collides with another vehicle.

\textbf{Negative:} The ego vehicle does not collide with another vehicle.

\textbf{Negative:} There is no collision between the ego vehicle and another vehicle.\\

\bottomrule
\end{tabular}
}
\caption{\small Autonomous driving event categories and verifier prompts used for evaluation data. For X-CLIP, the first description is used as the positive event text and the remaining descriptions are used as negative alternatives.}
\label{tab:app-av-curation} \vspace{-0.25in}
\end{table*}

\subsection{Evaluation Metrics}\label{sec:app-evaluation-metrics}

\para{Autonomous Driving}
To evaluate steering, we use held-out metrics that are not used during optimization. This is important because a steered generation can increase the optimization score without producing a meaningful or plausible change in the video. We therefore evaluate two complementary aspects: (\romannumeral 1) \textit{target alignment}, i.e., whether the generated video clearly depicts the target event described by the text prompt, and (\romannumeral 2) \textit{plausibility}, i.e., whether the generated video remains physically reasonable and free from severe visual artifacts that indicate out-of-distribution generation.

For autonomous driving, we adopt WorldModelBench (WMB)~\cite{Li2025WorldModelBench}, which evaluates video world-model generations using both text-prompt alignment and video-quality criteria. We use WMB because it includes driving data in its evaluation distribution and, empirically, provides meaningful scores for whether generated driving videos both match the target prompt and remain visually plausible. WMB reports three scores: instruction following \((0\text{--}3)\), physics adherence \((0\text{--}5)\), and commonsense consistency \((0\text{--}2)\). We use the instruction-following score to measure target alignment, and use physics adherence and commonsense consistency as proxies for generation plausibility.

In addition, we report a target-alignment score from an external VLM judge. This judge is used only for evaluation and is not used during steering. We use the following prompt template with Gemini~\cite{team2023gemini} to obtain a \(0\text{--}10\) alignment score:
\begin{quote}
\scriptsize\ttfamily
You are an expert evaluator of autonomous driving world models.  
Your ONLY task is to judge how well the video follows the text instruction.  
Ignore visual quality issues such as blur, flickering, warping, or low resolution: these are expected artifacts of world-model generation and must NOT affect your score.

TEXT DESCRIPTION: \textbraceleft prompt\textbraceright

SCORING RUBRIC (0-10):

10: The video is aligned with the text description clearly and completely.  

7-9: The video is aligned with the text description but with minor deviations (e.g., slightly wrong direction or timing).  

4-6: The video is vaguely related to the text description, but the specific action does not happen.  

1-3: The video is completely unrelated to the text description.  

0: The video content is completely unrelated to the text description.
\end{quote}

\para{Robotic Manipulation} We use human evaluation to determine whether each generated trajectory corresponds to task success or failure. We additionally report scores from Robometer~\cite{liang2026robometer}, a general-purpose robotic reward model. However, we rely on human evaluation as the primary metric for a more accurate assessment of generated trajectories, since reward models provide continuous scores and require prompt engineering to infer task success. 

\subsection{Fine-tuning \texorpdfstring{\(\pi_{0.5}\)-\textsc{droid}}{pi0.5-DROID} Policy}

We fine-tune the \(\pi_{0.5}\)-\textsc{droid} policy\footnote{\url{https://github.com/Physical-Intelligence/openpi}} using \(40\) expert trajectories per task with a weighted flow-matching objective~\cite{guo2026vlaw}. After steering the imagined rollout of each trajectory toward task failure, we assign a weight of \(1.0\) to trajectories classified as successful and a weight of \(0.1\) to trajectories steered to failure. We use two camera views, consisting of a wrist camera and one third-person camera. The weights are applied both to the fine-tuning loss and to dataset sampling. We use joint-position actions for fine-tuning, although the base policy was originally trained with joint-velocity actions, because our controller~\cite{zhu2022viola} does not support joint-velocity control. We fine-tune the pretrained checkpoint for \(10\)k steps on a single H100 GPU. At inference time, we use an open-loop action horizon of \(16\). We found that training a diffusion policy~\cite{chi2023diffusionpolicy,chi2024diffusionpolicy} from scratch performs poorly in this setting, likely due to the contact-rich nature of our tasks and the limited size of the fine-tuning dataset. In contrast, fine-tuning \(\pi_{0.5}\)-\textsc{droid} produces substantially better behavior and clearly imitates the expert demonstrations after fine-tuning.

\section{Additional Results}\label{sec:app-additional-results}

\subsection{Robust Policy Evaluation via Video World Model Steering}\label{sec:appendix-results-policy-eval}

In this section, we provide additional results for Sec.~\ref{sec:experiments-inner}. The evaluation focuses on steering video world-model generations toward text-defined failure events in each scene, enabling \textit{robust} evaluation of policy-proposed action sequences through pessimistic world-model imaginations. We measure both target alignment, i.e., how well the generated videos match the failure prompt, and plausibility, i.e., whether the generations remain visually and physically reasonable. Full results, including videos and interactive demos, are available on the \href{https://stressdream.github.io/}{project website}.

\begin{figure}[ht]
    \centering
    \includegraphics[width=1.0\linewidth]{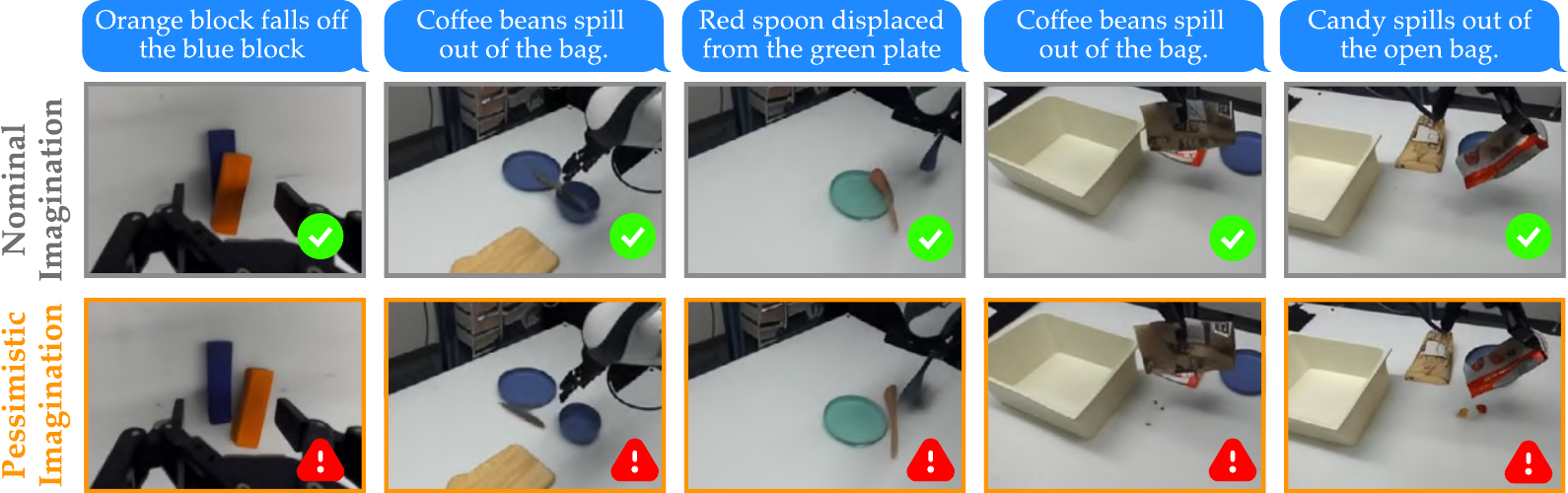}
    \vspace{-0.2in}
    \caption{\small{\textit{Additional Results: Nominal Generation vs.\ Steered Generation.} Given the same observation history and action sequence, \ours steers the video world model toward plausible task-failure events when failure is possible, whereas nominal generation often misses these failure modes.}}
    \label{fig:app-inner-ctrlworld}
    \vspace{-0.2in}
\end{figure}

\begin{figure}[ht]
    \centering
    \includegraphics[width=1.0\linewidth]{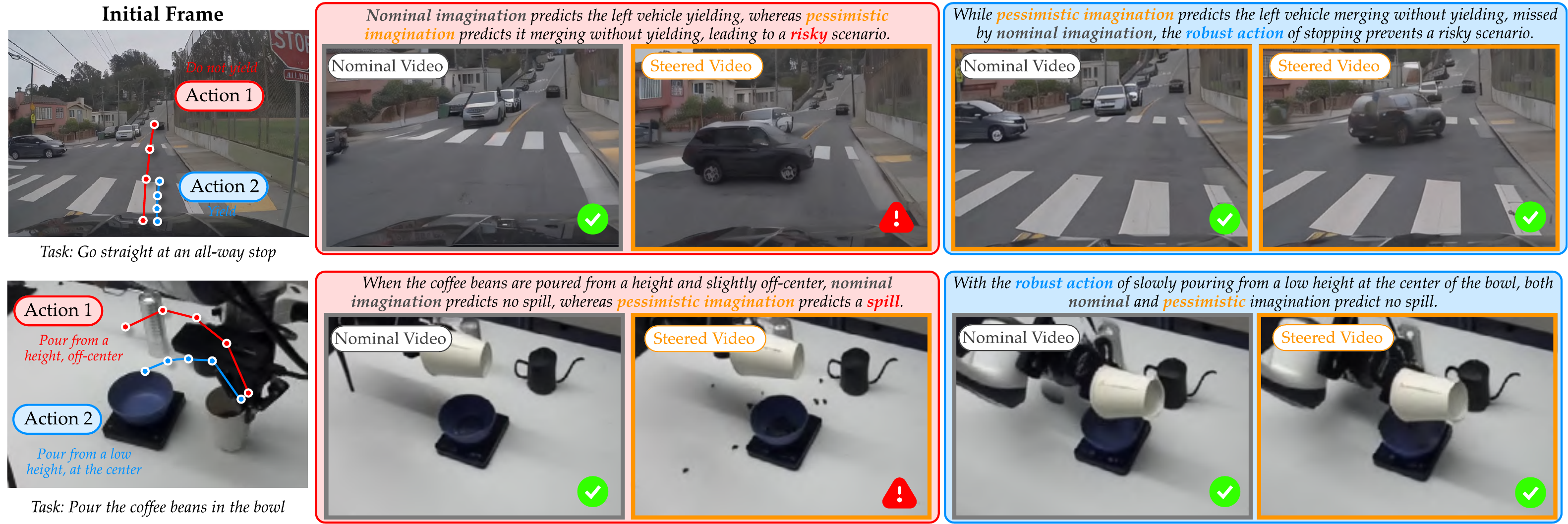}
    \vspace{-0.2in}
    \caption{\small{\textit{Robust Policy Evaluation with Video World Models.} Steering the video world model enables robust detection of safety-critical or task-failure events in imagined futures that nominal generation may miss, thereby enabling the selection of \textit{robust} action sequences that remain safe even under worst-case imaginations.}}
    \label{fig:app-inner-ctrlworld-planning}
    \vspace{-0.1in}
\end{figure}

\para{Robust Policy Evaluation with Pessimistic Imaginations}
Fig.~\ref{fig:app-inner-ctrlworld} shows additional results for the inner optimization in \eqref{eq:robust_control} on robotic manipulation tasks. For the same observation history and action sequence, nominal imaginations often do not predict failure events, whereas \ours generates plausible task-failure futures from the same input. This enables robust policy evaluation using the video world model by detecting possible failures through synthetic trajectories, without requiring repeated hardware rollouts. In contrast, nominal generation may miss such failures because it can sample non-failing futures even when failure is possible. Similarly, Fig.~\ref{fig:app-inner-ctrlworld-planning} shows that, across different candidate action sequences, \ours with pessimistic steering can reveal potential failures or safety-critical events in the imagined futures, while robust action sequences remain safe even under pessimistic imagination. This supports robust evaluation and action selection by using the video world model as a predictive model for planning and control.

\begin{wrapfigure}{r}{0.4\textwidth}
    \centering
    \vspace{-0.12in}
    \includegraphics[width=\linewidth]{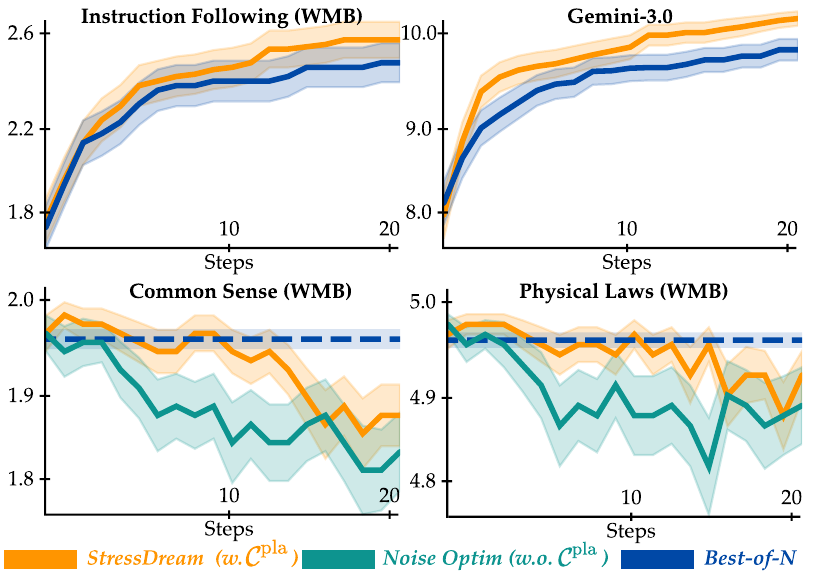} 
    \vspace{-0.22in}
    \caption{\small{\textit{Quantitative}: Driving WM.}}
    \vspace{-0.2in}
    \label{fig:app-vista-quantitative}
\end{wrapfigure}
\para{Detailed Results: Driving Video World Model}
Fig.~\ref{fig:app-vista-quantitative} shows the full quantitative results for driving world-model~\cite{gao2024vista} steering, evaluated with WorldModelBench (WMB)~\cite{Li2025WorldModelBench} and Gemini-3.0~\cite{team2023gemini}. The top two plots measure target alignment between the generated videos and the steering text prompts, while the bottom two plots report WMB video-quality metrics. Across both alignment metrics, \ours with noise optimization more effectively increases target-event alignment scores than the baselines with random generation, indicating that the proposed steering procedure reliably guides video world-model generation toward the desired events. At the same time, the proposed typical set constraint helps maintain plausibility: although video quality slightly decreases at later optimization steps, ours consistently preserves better quality in both common sense scores and physical laws scores than steering without regularization.

\para{Typical Set Constraints Preserve Plausibility of Steered Video Generations}
We have evaluated the plausibility of steered generations on VLM-based scores, such as the commonsense and physical-law scores in WorldModelBench~\cite{Li2025WorldModelBench}. To further assess the role of the plausibility objective in Sec.~\ref{sec:high-dim-gaussian}, we additionally evaluate generated driving videos using WorldLens~\cite{liang2025worldlens}. We focus on generation-quality metrics that assess whether generated videos are visually realistic, temporally stable, and semantically consistent, without requiring multi-view images or ground-truth labels. Unlike VLM-based scores, these metrics are computed from analytic feature-based measures\footnote{WorldLens also introduces a learned video score trained from large-scale human annotations of world realism, physical plausibility, and behavioral safety; however, this model is not released.}.

\begin{wrapfigure}{l}{0.6\textwidth}
    \centering
    \vspace{-0.1in}
    \includegraphics[width=\linewidth]{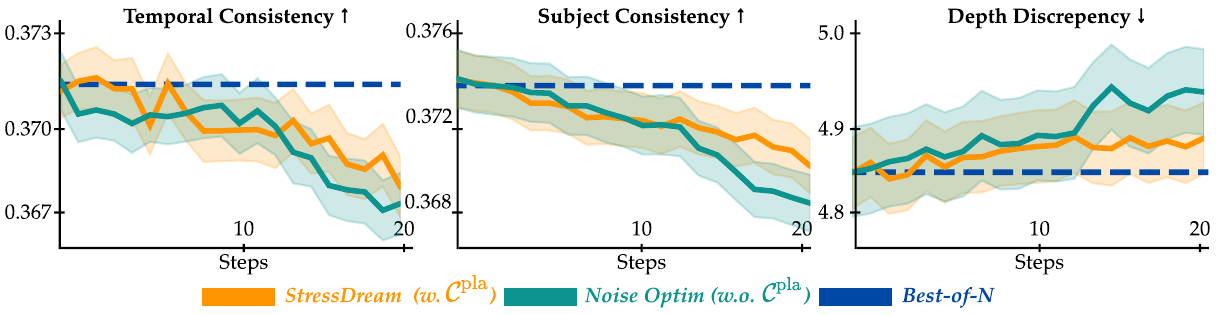} 
    \vspace{-0.2in}
    \caption{\small{\textit{Video generation quality of driving world-model generations evaluated with WorldLens~\cite{liang2025worldlens}.}}}
    \vspace{-0.2in}
    \label{fig:app-vista-worldlens}
\end{wrapfigure}Specifically, we evaluate \textit{Subject Consistency}, \textit{Depth Discrepancy}, and \textit{Temporal Consistency}. Subject Consistency uses DINO features~\cite{oquab2023dinov2} to measure whether dynamic subjects maintain consistent texture, shape, and structure over time. Depth Discrepancy measures the temporal stability of depth representations inferred from generated videos using DepthAnything~\cite{chen2025video}. Temporal Consistency measures frame-to-frame stability in a learned appearance space using CLIP encodings~\cite{radford2021learning}. We refer readers to the appendix of WorldLens~\cite{liang2025worldlens} for detailed metric definitions. Fig.~\ref{fig:app-vista-worldlens} shows the WorldLens generation-quality scores for driving world-model generations. Consistent with the WMB video-quality metrics, generation quality tends to degrade at later noise-optimization steps. However, using the proposed typical set constraints substantially mitigates this degradation, preserving subject consistency, depth stability, and temporal consistency compared to steering without regularization.

\para{Detailed Results: Failure Detection in Robotic Manipulation}
Fig.~\ref{fig:app-ctrlworld-failures} reports task-wise failure-detection results for robotic manipulation video world models. Steering is performed using Qwen3-VL~\cite{Qwen3-VL} scores computed from task-failure prompts, while the detection results are evaluated against trajectory-level success and failure labels from the dataset. Across tasks, \ours detects substantially more failure-prone trajectories than nominal generation or Best-of-\(N\) sampling.

\begin{figure}[ht]
    \centering
    \includegraphics[width=1.0\linewidth]{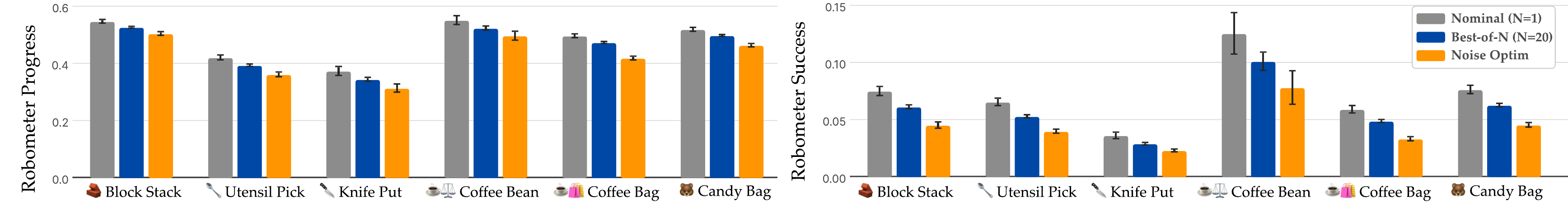}
    \vspace{-0.25in}
    \caption{\small{\textit{Robometer scores for imagined evaluation trajectories.}
    \textit{Left}: Robometer task-progress.
    \textit{Right}: Robometer task-success.
    \ours effectively generates more pessimistic imaginations than the baselines, yielding lower task-progress and task-success scores for the same action inputs.}}
    \label{fig:app-ctrlworld-quantitative}
    \vspace{-0.2in}
\end{figure}

\begin{figure}[ht]
    \centering
    \includegraphics[width=1.0\linewidth]{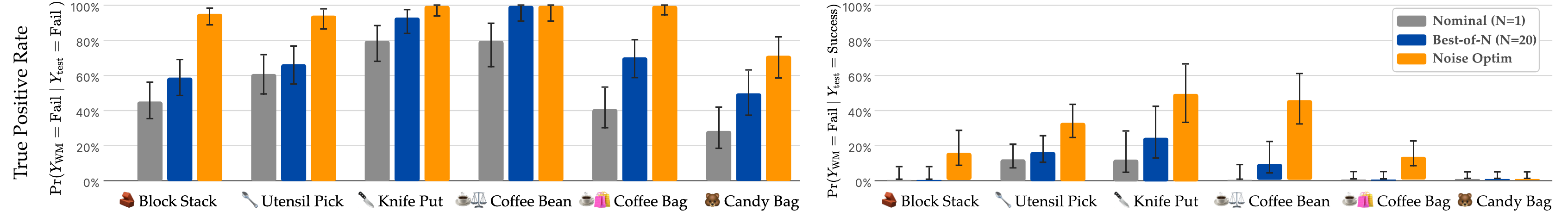}
    \vspace{-0.25in}
    \caption{\small{\textit{Failure detection in robotic manipulation video world models.}
    \textit{Left}: recall rate for failures, measuring the percentage of ground-truth failure trajectories for which the steered imagination predicts failure.
    \textit{Right}: percentage of imagined failure trajectories among ground-truth non-failure trajectories. This is not necessarily a false-positive rate, since some trajectories may be risky but happen to succeed during data collection.}}
    \label{fig:app-ctrlworld-failures}
    \vspace{-0.25in}
\end{figure}

Fig.~\ref{fig:app-ctrlworld-failures} also shows that \ours sometimes predicts failures for trajectories labeled as non-failure in the dataset. This should not be interpreted strictly as a false-positive rate: a trajectory that succeeded during data collection may still be risky under unobserved factors in a partially observable setting. Qualitatively, many of these flagged trajectories correspond to risky behaviors, suggesting that pessimistic imagination provides a more robust way of evaluating policies. At the same time, \ours does not flag all non-failure trajectories as failures, indicating that the method can still identify robust trajectories that remain safe even under pessimistic world-model imagination. These robust trajectories can then be prioritized for policy improvement, rather than relying on demonstrations that may have succeeded only by chance.

\subsection{\ours Can Effectively Discover Rare, Long-Tailed Events}\label{sec:appendix-longtail} \vspace{-0.05in}

Our main experiments evaluate whether \ours can generate videos that satisfy an inference-time specification while preserving plausibility more effectively than random sampling. This is particularly important for video world models, which operate in high-dimensional spaces where targeted generation is difficult: rare events often occupy complex, multimodal, and low-probability regions of the generative distribution, making random sampling highly sample-inefficient. In this section, we ask a complementary question: can \ours effectively \textit{discover} rare events in the target distribution by optimizing the initial noise with VLM guidance?

\begin{wrapfigure}{r}{0.3\textwidth}
    \centering
    \vspace{-0.2in}
    \includegraphics[width=\linewidth]{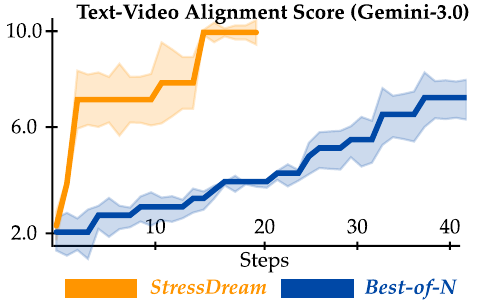} 
    \vspace{-0.22in}
    \caption{\small{\textit{Target-alignment score on long-tailed events in driving video world models.}}}
    \vspace{-0.2in}
    \label{fig:app-vista-longtail}
\end{wrapfigure}To evaluate this, we sample \(5\) evaluation examples from long-tailed driving events in the evaluation set, including pedestrian crossing, traffic-light change, and collision. As a baseline, we generate \(40\) random samples and report the best result among them, while \ours performs \(20\) optimization steps. We evaluate target alignment scores of generated videos using Gemini-3.0~\cite{team2023gemini}. Fig.~\ref{fig:app-vista-longtail} shows that \ours is more effective at discovering rare events than random sampling. Even with \(40\) samples, Best-of-\(N\) sampling does not match the target-alignment score achieved by \ours. This suggests that first-order optimization in the high-dimensional noise space can discover rare, long-tailed events more efficiently than zeroth-order search through random sampling.

\subsection{\ours Steers into Plausible Outcomes Only}\label{sec:appendix-plausibility}

We further elaborate on the experiments in Sec.~\ref{sec:experiments-inner} by examining whether steered generations remain grounded in plausible outcomes of the world model. In particular, we compare a driving world model fine-tuned with collision data against the base model without such fine-tuning. As shown in Fig.~\ref{fig:vista-ablation}, the base model yields low text--video alignment scores for collision events, even compared to random sampling from the collision-fine-tuned model, because collisions are not represented in its predictive distribution and therefore cannot be plausibly imagined.

\begin{figure}[ht]
    \centering
    \vspace{-0.15in}
    \resizebox{0.8\linewidth}{!}{
        \includegraphics[width=\linewidth]{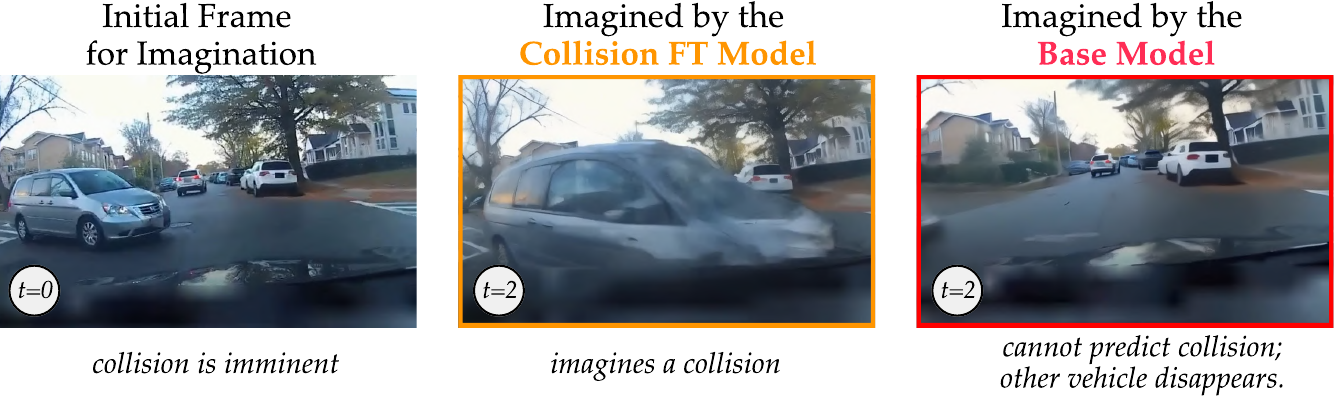} 
        }
    \vspace{-0.1in}
    \caption{\small{\textit{Imaginations near collision from the base model and the collision-fine-tuned world model.}}}
    \vspace{-0.25in}
    \label{fig:app-vista-plausibility-collision}
\end{figure}

Fig.~\ref{fig:app-vista-plausibility-collision} shows a qualitative example of imagination near a collision event using the base model and the collision-fine-tuned model. The base model fails to imagine the collision and instead hallucinates the other vehicle disappearing from the scene. In contrast, the fine-tuned model, whose predictive distribution includes collision events, correctly imagines the collision. This supports our claim that when the optimized noise remains close to the Gaussian prior, steering produces samples that are still supported by the world model distribution, rather than arbitrary hallucinations. In this sense, noise optimization differs from simply forcing inference-time alignment through a classifier or a fine-tuning objective, which can more easily push the generation off-manifold.

\begin{wrapfigure}{r}{0.5\textwidth}
    \centering
    \vspace{-0.25in}
    \includegraphics[width=\linewidth]{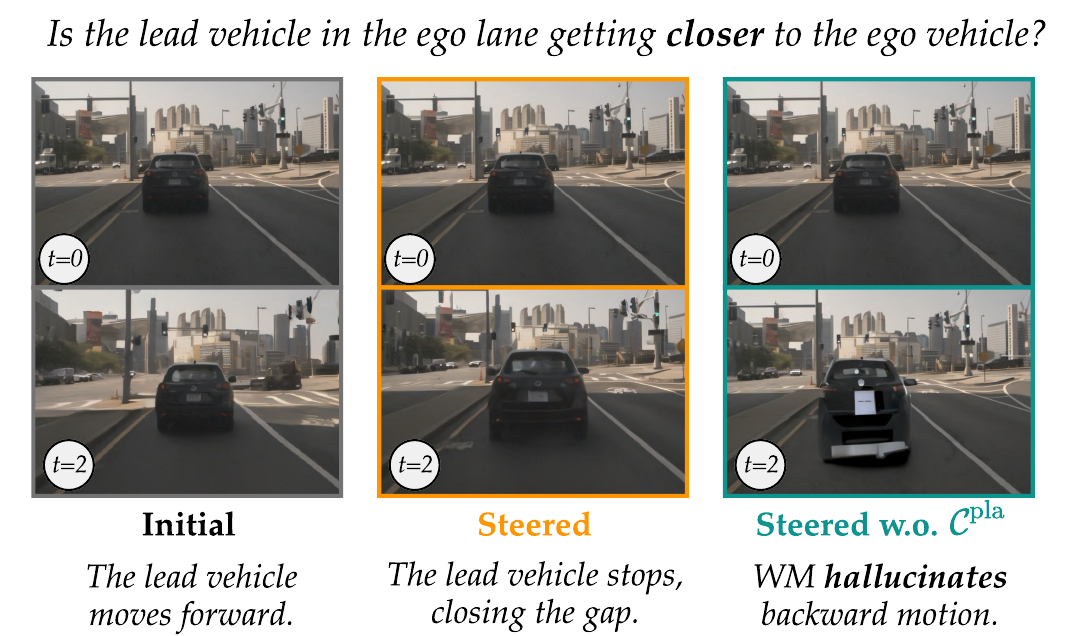} 
    \vspace{-0.22in}
    \caption{\small{\textit{Qualitative example: steering toward reduced distance to the lead vehicle. \ours produces a plausible outcome, whereas without a $\regularizer$, it leads to implausible hallucinations.}}}
    \vspace{-0.2in}
    \label{fig:app-vista-plausibility}
\end{wrapfigure}
Similarly, Fig.~\ref{fig:app-vista-plausibility} shows an example of steering a driving WM generation so that the distance between the ego vehicle and the lead vehicle decreases. In the initial generation, the lead vehicle moves forward, roughly maintaining its distance from the ego vehicle. In the steered generation, the lead vehicle instead remains stopped while the ego vehicle continues moving, causing the distance to decrease. Although the distance could be reduced even further if the lead vehicle moved backward, such behavior is highly implausible in a normal driving scenario. Accordingly, the steered generation chooses the plausible outcome of the lead vehicle stopping, rather than the implausible outcome of it reversing. By contrast, when steering is performed without the plausibility objective, the generation hallucinates unrealistic artifacts, such as the distorted backward motion of vehicle shapes.

\subsection{Robust Policy Improvement}\vspace{-0.05in}

\begin{figure}[ht]
    \centering
    \vspace{-0.1in}
     \resizebox{0.8\linewidth}{!}{
    \includegraphics[width=\linewidth]{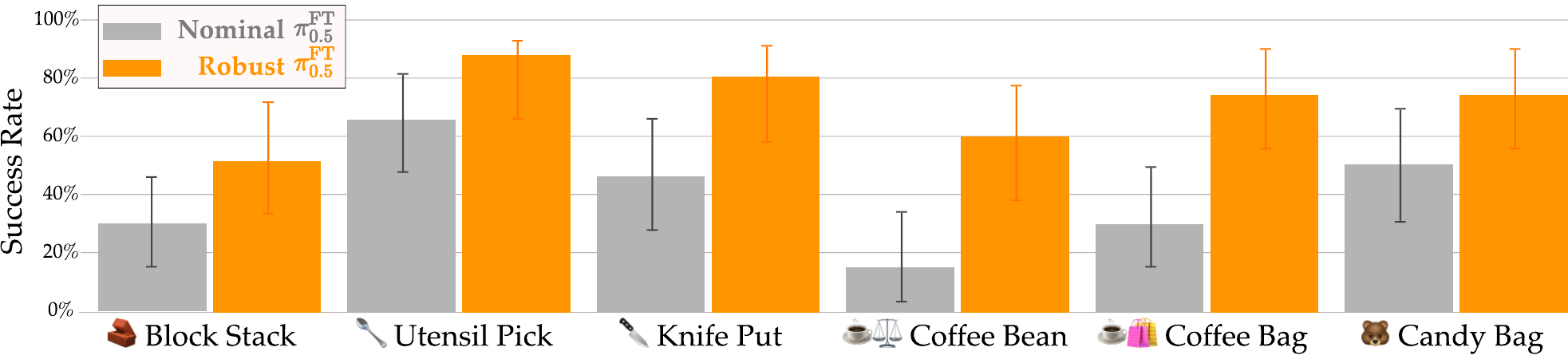} 
    }
    \vspace{-0.1in}
    \caption{\small{\textit{Task-wise success rates of fine-tuned \(\pi_{0.5}\)-\textnormal{\textsc{droid}} policies.}
    Robust fine-tuning with pessimistic video world-model imaginations improves policy success rates across tasks compared to nominal fine-tuning.}}
    \vspace{-0.15in}
    \label{fig:app-pi}
\end{figure}

Fig.~\ref{fig:app-pi} reports the task-wise success rates of the \(\pi_{0.5}\)-\textsc{droid} policy fine-tuned with video world-model imaginations (Sec.~\ref{sec:exp-pi}). Across all tasks, the \textit{Robust} \(\pi_{0.5}^{\mathrm{FT}}\) policy achieves higher success rates than the \textit{Nominal} \(\pi_{0.5}^{\mathrm{FT}}\) policy. This suggests that pessimistic imaginations help identify and downweight failure-prone demonstrations, thereby promoting actions that remain successful.

\section{Limitations \& Discussions}\label{sec:appendix-limitations}\vspace{-0.05in}

\para{Suboptimality of Noise Optimization}
While \ours provides a practical instantiation of steering video WM imaginations via gradient-based noise optimization, the optimized noise may still be suboptimal. The optimization in \ours should be viewed as a \textit{local refinement} of the initial noise rather than a global search over all possible futures. Exhaustively searching the high-dimensional noise space is infeasible: if \(\bm{\noise},\bm{\noise}'\overset{\mathrm{i.i.d.}}{\sim}\mathcal{N}(\mathbf{0},\mathbf{I}_\dimension)\), then \(\bm{\noise}-\bm{\noise}'\sim\mathcal{N}(\mathbf{0},2\mathbf{I}_\dimension)\), so \(\|\bm{\noise}-\bm{\noise}'\|_2^2/2\sim\chi_\dimension^2\), and therefore \(\|\bm{\noise}-\bm{\noise}'\|_2\) concentrates around \(\sqrt{2\dimension}\). In contrast, the average \(L_2\) distance between the initial and optimized noises in our experiments is far smaller, indicating that \ours performs local refinement rather than global exploration. The typical-set constraints make gradient updates reliable mainly within a moderate local region; overly large updates can push the noise OOD and make generations implausible. A natural extension is to combine Best-of-\(N\)-style search with gradient-based refinement~\cite{ma2025scaling,yuan2026inference}, or to learn a value function~\cite{holderrieth2026diamond} to guide denoising. Finally, for video WMs that require autoregressive rollouts, later generations condition on previously generated futures. The optimal solution would therefore require jointly optimizing a sequence of noises across rollout segments, but the current implementation solves it sequentially.

\para{Preserving Plausibility} \ours preserves plausibility by regularizing the optimized noise, encouraging it to remain in the typical set of the high-dimensional Gaussian. However, our notion of plausibility is defined with respect to the learned WM distribution: it assumes that generations from Gaussian noise are plausible under the WM. Thus, \ours can keep steering within the outcomes the WM supports, but it cannot correct a base WM that already produces physically implausible generations. One could incorporate video-quality or physical-consistency metrics as part of the steering objective~\cite{prabhudesai2024video, eyring2024reno, eyring2025noise}, to explicitly bias generations toward higher-quality imaginations. Conversely, \ours also cannot imagine outcomes that may be plausible in the real world but are absent from the WM distribution. This highlights the importance of training WMs on diverse interaction data, including failures and rare outcomes, rather than success-only data~\cite{khazatsky2024droid}, so that their predictive distributions can support policy evaluation and improvement.

\para{Computational Inefficiency} Although \ours is sample-efficient for generating plausible target outcomes compared to random sampling by leveraging dense VLM gradients, its runtime is still primarily bottlenecked by video WM generation, which can take several minutes per imagination with current models~\cite{gao2024vista,hassan2025gem,guo2025ctrl,agarwal2025cosmos}. Inference-time steering methods for text-to-image generation often reduce this cost with step-distilled or shortcut models that generate samples in one or a few denoising steps, making both noise optimization and gradient computation more tractable~\cite{eyring2024reno,eyring2025noise}. Our method is orthogonal to these advances: as video WMs become faster, \ours directly benefits from improved sampling efficiency, and shortcut models may enable more accurate gradient computations without relying on the current approximation. Another promising direction is to amortize the iterative optimization by learning a noise generator for fixed target outcomes, similar to~\citet{eyring2025noise,wagenmaker2025steering}, but enforcing the typical-set constraints in high-dimensional initial noise space remains challenging. Lastly, \ours incurs additional overhead from backpropagating through the criterion function, making each optimization step more expensive (Appendix~\ref{sec:appendix-time-complexity} for more details). However, random sampling still requires evaluating generated samples with expensive verifiers, so systematic steering within a finite sampling budget remains useful for discovering high-impact outcomes from WMs.

\para{Reward Hacking} Although our gradient-based optimization with the Vision-Language-Model's gradient is empirically effective, it can sometimes introduce reward hacking, where the criterion increases without inducing meaningful changes in the generated video~\cite{eyring2024reno,tang2024inference, hong2026understanding}. Prior work mitigates this issue by combining multiple reward functions or adding image-/video-quality metrics that directly penalize degenerate generations~\cite{eyring2024reno, prabhudesai2024video}. \ours can also incorporate these strategies. In practice, however, we observe that VLM-based criteria trained on internet-scale data are substantially more robust than reward functions learned from scratch: they tend to assign high scores only when the target event is semantically evident in the generated video, whereas task-specific learned rewards are more susceptible to adversarial artifacts. This limitation may be further mitigated by more efficient video WM architectures with robust latent representations~\cite{psenka2026parallel}, shortcut/distilled diffusion models with fewer denoising steps~\cite{hafner2025training}, and stronger reward models with better video understanding and robustness to adversarial artifacts~\cite{liang2026robometer}. Moreover, \ours requires specifying target high-impact outcomes with text prompts, which may need to be adapted for each scene, and the criterion's effectiveness depends on prompt quality (see Appendix~\ref{app:vlm-reward}).

\end{document}